\documentclass[11pt,letterpaper]{article}
\usepackage[margin=1in]{geometry}
\usepackage{lmodern}
\usepackage[utf8]{inputenc}
\usepackage[T1]{fontenc}
\usepackage[numbers,sort&compress]{natbib}
\usepackage{amsmath,amssymb,amsfonts,amsthm,mathtools}
\usepackage{booktabs,array,tabularx,longtable,graphicx,subcaption,enumitem,nicefrac}
\usepackage{microtype,xcolor,float,tikz}
\usetikzlibrary{arrows.meta,calc,positioning,patterns}
\usepackage{xurl,etoolbox,newunicodechar}
\newunicodechar{–}{--}
\usepackage{hyperref}
\hypersetup{colorlinks=true,allcolors=blue,pdftitle={Generalization in Nonlinear Least Squares via Learned Feature Geometry},pdfauthor={Ayub Kharel, Ilja Kuzborskij, Patrick Rebeschini, Yasin Abbasi-Yadkori}}
\usepackage[capitalize]{cleveref}
\newtheorem{theorem}{Theorem}
\newtheorem{lemma}[theorem]{Lemma}
\newtheorem{proposition}[theorem]{Proposition}
\newtheorem{definition}[theorem]{Definition}
\newtheorem{corollary}[theorem]{Corollary}
\newtheorem{remark}[theorem]{Remark}
\theoremstyle{definition}

\newcommand{\R}{\mathbb{R}}
\DeclareMathOperator{\E}{\mathbb{E}}
\newcommand{\deff}{d_{\mathrm{eff}}}
\newcommand{\dclass}{d_{\mathrm{lin}}}

\newcommand{\tr}{\operatorname{tr}}

\newcommand{\wh}{\widehat}

\newcommand{\cN}{\mathcal{N}}

\newcommand{\sumin}{\sum_{i=1}^n}
\newcommand{\gen}{\mathrm{gen}}

\newcommand{\pinv}{^{\dagger}}
\newcommand{\lmin}{\lambda_{\min}}
\newcommand{\dpred}{d_{\mathrm{eff}}^{\dagger}}
\newcommand{\gpred}{\gamma_{\mathrm{pred}}}
\newcommand{\Range}{\operatorname{Range}}

\title{Generalization in Nonlinear Least Squares via Learned Feature Geometry}
\author{%
\normalsize Ayub Kharel\textsuperscript{1}\quad
Ilja Kuzborskij\textsuperscript{2}\quad
Patrick Rebeschini\textsuperscript{1}\quad
Yasin Abbasi-Yadkori\\[0.6em]
\small\textsuperscript{1}University of Oxford\qquad
\textsuperscript{2}Google DeepMind}
\date{}
\begin{document}
\maketitle
\begin{abstract}
We study the generalization of ridge-regularized nonlinear least-squares models via on-average algorithmic stability, deriving error bounds for local minimizers in terms of a data-dependent effective dimension that reflects the geometry of the gradient model at the trained parameters, through the empirical Jacobian Gram matrix and a residual--curvature term. In the linear case, where the curvature term vanishes, this recovers the classical effective dimension of the Jacobian kernel covariance, but evaluated at the trained model rather than at initialization as is typical in neural tangent kernel analyses. We further bound this effective dimension via covering complexity of the gradient features, leading to guarantees that depend on learned geometry rather than parameter count. In particular, for manifold-supported data and piecewise Lipschitz Jacobians, the bounds scale with intrinsic dimension, while for one-hidden-layer ReLU networks, the mechanism can be made explicit through counts of activation-stable regions. Experiments on synthetic manifolds, clustered distributions, and benchmark datasets illustrate trained-Jacobian compression, the tightness of the residual-curvature linearization, and agreement between the stability bound and observed generalization gaps. A key feature of our bounds is the simplicity of their derivation, which follows from first principles using the Brascamp--Lieb inequality under strongly log-concave noise.
\end{abstract}

\section{Introduction} 
Modern machine learning models are often trained in regimes where classical notions of complexity appear inadequate. Highly overparameterized predictors can interpolate the data and yet generalize well \cite{zhang2021understanding,belkin2019reconciling,bartlett2020benign,koehler2021uniform}. This phenomenon suggests that generalization is governed not by the size of the hypothesis class alone, but by properties of the specific solution found by training. A natural question is therefore: which aspects of a fitted model determine its ability to generalize? In particular, can one identify quantities that depend on the learned representation and reflect the geometry induced by the data, rather than worst case complexity over all parameters?

To address this question in a clean and interpretable setting, we study ridge regularized nonlinear least squares regression under a fixed design, where the inputs are treated as deterministic and the randomness arises only from the response noise. This perspective isolates the role of the learned predictor and its geometry, allowing us to analyze how sensitivity to the data, and ultimately generalization, is controlled by properties of the fitted solution itself.

\paragraph{Setup.}
Let \(x_1,\ldots,x_n\in\mathbb R^d\) be a fixed design and suppose that the responses are generated as
\[
    Y_i = f^\star(x_i)+\xi_i,\qquad i=1,\ldots,n,
\]
where the noise variables are independent and, for the main results, Gaussian or more generally strongly log concave.  Given a differentiable predictor \(f(\cdot;\theta)\), with parameter \(\theta\in\mathbb R^p\), we study the ridge regularized nonlinear least squares objective
\begin{equation}
\label{eq:nonlinear_ls_intro}
\begin{aligned}
 L_\lambda(\theta)&=\widehat L(\theta)+\frac\lambda2\|\theta\|_2^2,\\
 \widehat L(\theta)&=\frac1{2n}\sum_{i=1}^n(f(x_i;\theta)-Y_i)^2.
\end{aligned}
\end{equation}

We study regression where \(f(x;\theta)\) is nonlinear in the parameters and the number of parameters may be much larger than the sample size, and we focus on neural networks as our central example.  Since practical optimization algorithms for such objectives are typically only expected to find local solutions, we focus on a nondegenerate local minimizer \(\widehat\theta\) of \eqref{eq:nonlinear_ls_intro}.  This abstracts away the detailed dynamics of a specific optimizer while retaining the central question. Which geometric properties of a fitted nonlinear least squares solution control its generalization?

Let \(Y_i'\) be an independent copy of \(Y_i\), and define, for the dataset \(D=\{(x_i,Y_i)\}_{i=1}^n\),
\[
\begin{aligned}
 L(\theta)&:=\frac1{2n}\sum_{i=1}^n\mathbb E[(f(x_i;\theta)-Y_i')^2],\\
 \gen_D(\theta)&:=L(\theta)-\widehat L(\theta).
\end{aligned}
\]
We choose to use fixed design with square loss instead of random design, as it is a well established framework in classical statistics. It does not require an assumption on the distribution of the covariates and depends directly on the observed inputs. Fixed and random design address different notions of prediction risk. Relating them requires further assumptions connecting the empirical and population distributions \cite{Wainwright19}. This setting is natural when the inputs of interest are already available, as in fixed benchmarks, while their labels or measurements are noisy. Appendix~\ref{app:fixed_design} gives further discussion and a potential extension to Gaussian design.

Classical generalization theory offers several existing methods.  Uniform convergence bounds control \(\sup_{\theta\in\Theta}|L(\theta)-\widehat L(\theta)|\) through quantities such as VC dimension \cite{vapnik95}, covering number~\cite{anthony1999neural,bartlett2019nearly}, or Rademacher complexity~\cite{koltchinskii2002rademacher,bartlett2019nearly} of a hypothesis class.  PAC Bayesian and information theoretic bounds provide nonuniform alternatives, replacing a global class complexity by a posterior dependent or information dependent quantity \cite{mcallester1998pac,shalev2014understanding,xu2017information}.  These frameworks are powerful, but several works have emphasized that conventional capacity measures and uniform convergence arguments can be vacuous or oblivious to algorithm dependent inductive biases for trained deep networks, especially in interpolating regimes \cite{zhang2021understanding,nagarajan2019uniform}.

We take the viewpoint of algorithmic stability \cite{bousquet2002stability}.  Stability asks whether the output of a learning procedure changes little when one training example is replaced by an independent copy.  If \(\widehat\theta^{(i)}\) denotes the local minimizer obtained after replacing only \(Y_i\) by \(Y_i'\), then a natural on average prediction stability quantity is
\[
    \frac1n\sum_{i=1}^n
    \mathbb E\Bigl[
    \bigl(
    f(x_i;\widehat\theta)
    -
    f(x_i;\widehat\theta^{(i)})
    \bigr)^2
    \Bigr] 
  \]  
  Small replace one sensitivity leads to a small expected generalization gap \cite{bousquet2002stability, pmlr-v48-hardt16}. In the context of optimization, algorithmic stability theory is well developed for convex and strongly convex problems, and has also been studied for nonconvex objectives under additional landscape assumptions such as the Polyak--\L{}ojasiewicz condition or weak convexity \cite{CharlesP18,RichardsR21}.  Other works explored stability in the nonconvex scenarios by introducing stabilization operations  into the algorithm itself, such as through decaying step sizes~\citep{pmlr-v48-hardt16} or adding noise to the gradient~\citep{pensia2018generalization}.  Our question is different. Can a local minimizer of the original nonlinear least squares problem result in stable predictor because the fitted model
has a low effective complexity, such as a low \emph{effective dimension}?

Indeed, in closely related kernel methods one does encounter such scenarios.
Since our guiding example is a neural network learning, of a particular interest here is
the Neural Tangent Kernel (NTK) theory which aim to explain generalization ability in wide neural networks~\cite{jacot2018neural,du2019gradientdescentfindsglobal,arora2019exactcomputationinfinitelywide,cao2019generalizationboundsstochasticgradient}.  The NTK approach linearizes the network around random initialization and studies the kernel matrix generated by the initialization Jacobian features.  For a fixed feature map with empirical covariance \(K_0\), the variance quantity is the classical ridge effective dimension
\[
    \dclass(K_0,\lambda)
    :=
    \tr\!\left(
    \bigl((K_0+\lambda I)^{-1}K_0\bigr)^2
    \right),
\]
which is central in kernel ridge regression and random feature analyses~\cite{CaponnettoV07,bach2013sharpanalysislowrankkernel,rudi2021generalizationpropertieslearningrandom}.  This fixed feature description captures the initialization geometry.  Feature learning requires understanding the geometry of the feature space which can change substantially during training.

In contrast,
this paper studies the nonlinear regime without linearizing at initialization.  At the fitted point, let \(\widehat G\) be the empirical covariance of the trained Jacobian features and let $\wh H_{\lambda}$ be the Hessian of the regularized empirical objective.  The effective dimension that appears in our bounds is
\begin{equation}
\label{eq:main_deff_def}
    \deff(\widehat\theta;\lambda)
    :=
    \tr\!\left(\bigl(\wh H_{\lambda}^{-1}\widehat G\bigr)^2\right).
\end{equation}
Thus the numerator is the Jacobian feature geometry after training, and the inverse metric is the local curvature of the nonlinear training objective.  

For the feature geometry of neural networks, direct connections between the complexity of activation regions and generalization bounds remain largely unexplored beyond highly restricted settings, such as shallow, one-dimensional binary classification under gradient flow \cite{safran2022effective}. Recent work \cite{patel2025localcomplexitylinearregions} studies the complexity of local linear regions in a ReLU network and highlights using it to mathematically obtain a generalization bound as an outstanding open problem. Our work makes this connection explicit, giving a theoretical grounding for empirical findings \cite{o'brien2025using}, and is general enough (due to the usage of covering arguments) to work even when activation functions are not piecewise linear, but remain piecewise Lipschitz.

We believe this constitutes a significant result for nonlinear least squares. Our bound holds with no global optimality, convexity or global landscape condition such as PL, algorithm agnostically without the stabilizing devices used in the nonconvex stability literature. Under the selection assumption of Theorem~\ref{thm:main_prediction_stability}, bounded expected training loss and effective dimension give a comparable $O(n^{-1/2})$ rate for fixed noise curvature (Table~\ref{tab:literature_comparison}). Our results also allow strongly log concave noise, a much broader class than the Gaussian noise classically studied in this setting. Further related work is discussed in Appendix~\ref{app:related_work}.

\subsection{Summary of contributions}
\label{sec:contributions}

\begin{enumerate}[leftmargin=*]
\item
\textbf{A trained model effective dimension through algorithmic stability.}
For any locally twice differentiable predictor, let the learning rule select nondegenerate local minimizers of ridge regularized nonlinear least squares. Under the selection assumption, we prove the prediction stability bound (\Cref{thm:main_prediction_stability})
\[
    \frac1n\sum_{i=1}^n
    \mathbb E
    \left[
    \bigl(
    f(x_i;\widehat\theta)
    -
    f(x_i;\widehat\theta^{(i)})
    \bigr)^2
    \right]
    \le
    \frac{4\,\mathbb E[
    \deff(\widehat\theta;\lambda)]}{\alpha\,n},
\]
where $\alpha$ is the strong log concavity parameter of the response noise (so $\alpha=1/\sigma^2$ for Gaussian noise of variance $\sigma^2$). For square loss this gives a bound of order (\Cref{cor:main_gen_gap}),
\[
    \mathbb E[\gen_D(\widehat\theta)]
    =
    O\!\left(      
    \sqrt{
    \frac{
        \mathbb E[\deff(\widehat\theta;\lambda)]
    }{n}}
\right)\quad(\mathrm{as}\ n\to\infty).
  \]  
  where $\lambda = \lambda_n$ is tuned based on $n$, the expected training loss is bounded, $\mathbb E[\deff]=O(n)$, and $\alpha$ is fixed.  
  The result applies to arbitrary locally twice differentiable predictors at nondegenerate local minimizers under the selection assumption, including multilayer neural networks. ReLU models require the additional regularity discussed in Appendix~\ref{app:assumptions}. The proof is a simple but novel application of a replace one centering argument with the Brascamp--Lieb inequality and an implicit function sensitivity identity for the local minimizer.  In the linear or zero residual curvature case, \(\deff\) reduces to \(\dclass\), so our bound is a strict extension of standard generalization bounds for the linear setting. We control \(\deff\) and show that it can be significantly smaller than $n$ depending on data or feature geometry. 

\item \textbf{Covering complexity of trained Jacobian features.}
We show that \(\deff\) is small whenever the trained Jacobian features can be compressed and prediction visible curvature is controlled, in the following sense. For $\lambda>0$, let $\gamma_{\rm pred}(\lambda)>0$ measure the profiled Hessian of the fitted objective relative to $\widehat G+\lambda I$ (Section~\ref{sec:geometry_main}). If \(\mathcal C_J(\varepsilon)\) is the empirical covering number of the trained Jacobian features at radius $\varepsilon$, then
\[
 \deff(\widehat\theta;\lambda)\le
 \gamma_{\rm pred}(\lambda)^{-2}\inf_{\varepsilon>0}
 \left[\mathcal C_J(\varepsilon)+\frac{\varepsilon^2}{\lambda}\right].
\]
The curvature factor measures how the fitted objective transmits feature geometry to prediction sensitivity.
  For one hidden layer ReLU networks, this covering bound can be expressed more explicitly by exploiting the piecewise linear structure of the model. It is controlled by the number of \emph{activation stable regions} (see Def.\ \ref{def:activation-stable-region}) occupied by the data and the local feature contraction within those regions (formal statement in Section~\ref{sec:geometry_main}). Compared to standard metric entropy bounds on Lipschitz function classes \cite{vonluxburg2004distance}, our bound is tighter because it covers only the empirical Jacobian features actually realized by the trained model, rather than an entire ambient class. We also explicitly connect this feature geometry to data geometry under stable feature maps, allowing for prediction of generalization gaps depending on geometric properties of the input data.  

\item
  \textbf{Intrinsic dimensional and ReLU consequences.}
If the data lie on an \(m\)-dimensional manifold and the trained Jacobian map is (piecewise) Lipschitz on the $M$ occupied parts of the manifold with constants $(L_r)_{r \leq M}$, then the input space covers transfer to Jacobian feature covers. In particular, writing $A=C_{\mathcal M}\sum_{r=1}^M L_r^m$, this gives, for $\lambda>0$, a bound of the form (Proposition~\ref{prop:manifold_main})
\[
\begin{aligned}
 \deff(\widehat\theta;\lambda)&\le\gamma_{\rm pred}(\lambda)^{-2}\\
 &\quad\times\left[M+C_m A^{2/(m+2)}\lambda^{-m/(m+2)}\right],
\end{aligned}
\]
where $C_{\mathcal M}$ is the manifold covering constant, defined formally in Section~\ref{sec:geometry_main}, and $C_m$ depends only on $m$.
Thus, when the intrinsic dimension \(m\), the number of occupied pieces \(M\), and the local Jacobian Lipschitz constants are controlled, the relevant complexity is governed by learned geometry on the data, even in large ambient and parameter spaces.  For twice differentiable ReLU networks, the pieces are activation stable regions and the constants \(L_r\) measure feature contraction.

Under particular data and training assumptions, simplicity bias results yield few occupied patterns \cite{phuong2020inductive,boursier2025simplicitybiasoptimizationthreshold}.
To this end, controlling the number of activation stable regions in a problem dependent way beyond strong assumptions remains a major open problem, and here we resort to experimental measurements in practice.

\item
\textbf{Numerical evidence for trained Jacobian compression.}
We complement the theory with simulations on synthetic manifolds, clustered distributions, and benchmark regression datasets.  The experiments detailed in Section \ref{sec:experiments} and  Appendix \ref{app:experimental_details} compare initialization and trained Jacobian Grams on the same samples, test the residual curvature comparison, probe cover and activation region geometry, and compare plug in scores to observed gaps. Additional replace one path diagnostics explore the selection assumption (Appendix~\ref{app:exp_selection}).  

\end{enumerate}

\begin{table}[t]
\centering
\footnotesize
\setlength{\tabcolsep}{4pt}
\renewcommand{\arraystretch}{1.08}
\caption{Selected generalization results. Rates suppress regularity constants. $T$ is the number of steps, $\sigma_t$ is the injected noise scale, $\Delta_S^{\rm oracle}$ is the oracle objective of \citet{RichardsK21}, $D=\mathbb E[d_{\rm eff}]$, and $\bar L=\mathbb E[\widehat L]$. Except where specified, entries concern expected gaps. Our risk uses fresh responses at the same inputs. Other entries may concern random design risk or estimation error. The assumptions and rates therefore do not give a uniform ordering of the results.}
\label{tab:literature_comparison}
\begin{tabular}{@{}p{.225\textwidth}p{.465\textwidth}p{.265\textwidth}@{}}
\toprule
Reference & Setting and principal requirements & Representative bound \\
\midrule
\citet{pmlr-v48-hardt16} & Nonconvex SGD, bounded, Lipschitz, $\beta$-smooth loss, $\eta_t\le c/t$ & $O(T^{\beta c/(1+\beta c)}/n)$ \\
\citet{pensia2018generalization} & Noisy iterates, bounded updates, subgaussian loss & $O(\sqrt{\sum_{t=1}^T\eta_t^2/\sigma_t^2\,/n})$ \\
\citet{CharlesP18} & Lipschitz loss, global PL condition, convergence to global minima & $O(1/n)$ plus optimization error \\
\citet{RichardsR21} & GD, smooth loss with Lipschitz Hessian, controlled weak convexity along the path and step sizes & $O(\sum_{t<T}\eta_t/n)$ when the curvature factor is bounded \\
\citet{RichardsK21} & GD, sufficiently wide shallow network, smooth bounded activation, fixed output weights & $O(\frac{\eta T}{n}(1+\frac{\eta T}{n})\mathbb E\Delta_S^{\rm oracle})$ \\
\citet{QiaoZSSW24} & Univariate ReLU, GD stable, optimized fits, regular target and bounded predictions & $\widetilde O(n^{-4/5})$ interior MSE \\
\citet{LiangQWP25} & Multivariate ReLU, flat solutions, uniform ball inputs, bounded labels and predictions & $\widetilde O_d(n^{-1/(2d+4)})$ gap, high probability \\
\citet{WangM23} & Homogeneous deep networks, bounded data and outputs, controlled GD/SGD trajectory & $\widetilde O(C_T/\sqrt n)$, high probability \\
\midrule
This paper & Strongly log concave response noise (Theorem~\ref{thm:main_prediction_stability}) & $O(\sqrt{\bar L D/n}+D/n)$ \\
\bottomrule
\end{tabular}
\end{table}
\section{Effective Dimension and Generalization Bound}
\label{sec:stability_setup}

This section establishes the main result that motivates the paper.  The inputs are fixed, and the randomness is in the response noise.  We compare the fitted local minimizer \(\widehat\theta\) with the local minimizer \(\widehat\theta^{(i)}\) obtained after replacing a response \(Y_i\) by an independent copy \(Y_i'\).  The main claim is that the average prediction change is controlled by an effective dimension evaluated at the trained model.

\subsection{Effective dimension at the trained local minimizer}

We work with local minimizers of the regularized empirical risk
\[
    L_\lambda(\theta) := \widehat L(\theta) + \tfrac{\lambda}{2}\|\theta\|_2^2,
\]
where \(\widehat L\) is as defined in \eqref{eq:nonlinear_ls_intro}. We take \(\lambda\ge0\) and assume that \(f(x_i;\cdot)\) is \(C^2\) near each selected point, for almost every response vector. To capture the geometric quantities used in our bounds, we will need not only stationarity but also the existence of a well defined inverse Hessian along the local solution branch as the responses vary.

\begin{definition}[Nondegenerate local minimizer]
\label{def:nondegenerate_local_min}
For \(y\in\mathbb R^n\), let \(D_y=\{(x_i,y_i)\}_{i=1}^n\) and \(L_{\lambda,y}(\theta)=\widehat L_y(\theta)+\frac{\lambda}{2}\|\theta\|_2^2\). A point \(\theta^\star\in\mathbb R^p\) is a \emph{nondegenerate local minimizer} of \(L_{\lambda,y}\) if \(\nabla_\theta L_{\lambda,y}(\theta^\star)=0\) and
\[
H_\lambda(y,\theta^\star):=\nabla_\theta^2 L_{\lambda,y}(\theta^\star)\succ 0 .
\]
We denote the set of such minimizers by \(\Theta^\star_\lambda(D_y)\).
\end{definition}

\paragraph{\textit{Selection assumption.}}
{\itshape Let \(\mu\) denote the law of \(Y\), fix a measurable learning rule \(A_\lambda\), and write \(\theta_y=A_\lambda(y)\), \(h_j(y)=f(x_j;\theta_y)\), and \(g_j(y)=\nabla_\theta f(x_j;\theta_y)\). We assume that for each \(j\in[n]\) the prediction map \(h_j\) belongs to \(W^{1,2}(\mu)\), meaning that its values and weak response derivatives are square integrable, and that for \(\mu\)-almost every \(y\) the selected point \(\theta_y\) is a nondegenerate local minimizer, \(\theta_y\in\Theta^\star_\lambda(D_y)\), with the weak derivatives of \(h_j\) given by
\begin{equation}
\label{eq:selection-influence}
\partial_{y_k}h_j(y)
=\frac1n g_j(y)^\top H_\lambda(y,\theta_y)^{-1}g_k(y).
\end{equation}
}

The assumption constrains the fitted predictions rather than a particular parameterization. It permits prediction equivalent parameter changes and piecewise smooth branches when the predictions remain in \(W^{1,2}(\mu)\) and satisfy \eqref{eq:selection-influence}.  Appendix~\ref{app:assumptions} gives a necessity counterexample, an explicit nonconvex example, and extensions to flat symmetry directions and approximate influence identities. Appendix~\ref{app:exp_selection} provides numerical diagnostics.

\begin{lemma}[Implicit response derivative]
\label{lem:implicit_response_derivative}
Let \(y_0\in\R^n\) and let \(\theta_0\) be a nondegenerate local minimizer of \(L_{\lambda,y_0}\), with \(f(x_i;\cdot)\) locally \(C^2\). Then there are neighbourhoods of \(y_0\) and \(\theta_0\) and a unique \(C^1\) map \(y\mapsto\theta_y\) of stationary points with \(\theta_{y_0}=\theta_0\), and on that neighbourhood \eqref{eq:selection-influence} holds for all \(j,k\in[n]\).
\end{lemma}

The proof is a direct application of the implicit function theorem to the first order condition and is given in Appendix~\ref{app:implicit_response_derivative}.

At the trained point we set
\begin{align*}
g_i&:=\nabla_\theta f(x_i;\widehat\theta),\qquad r_i:=f(x_i;\widehat\theta)-Y_i,\\
\widehat G&:=\frac1n\sum_i g_i g_i^\top,\\
\widehat\Delta&:=\frac1n\sum_i r_i\nabla_\theta^2 f(x_i;\widehat\theta).
\end{align*}The Hessian of the regularized empirical objective at $\widehat\theta$ is $\widehat H_{\lambda}:=\nabla_\theta^2\widehat L(\widehat\theta)+\lambda I=\widehat G+\widehat\Delta+\lambda I$.

By Definition~\ref{def:nondegenerate_local_min}, $\widehat H_{\lambda}$ is invertible and positive definite. The matrix \(\widehat\Delta\) is the residual curvature correction. It vanishes for predictors that are affine in the parameters, and it is small when the training residuals are small or if the local parameter curvature is mild.

Recall that $\deff(\widehat\theta;\lambda):=\tr((\widehat H_{\lambda}^{-1}\widehat G)^2)$. The covariance $\widehat G$ identifies directions visible to the predictions through the trained Jacobian features.  The inverse Hessian \(\widehat H_{\lambda}^{-1}\) measures the local sensitivity of the fitted solution in those directions.  Their product therefore gives a local degrees of freedom count for the trained predictor.

In the linear feature case, \(f(x;\theta)=\theta^\top\Phi(x)\), the second derivative \(\nabla_\theta^2 f\) vanishes.  Hence \(\widehat\Delta=0\), \(\widehat H_{\lambda}=\widehat G+\lambda I\), and
\[
\deff(\widehat\theta;\lambda)
=\sum_j\left(\frac{\mu_j(\widehat G)}{\mu_j(\widehat G)+\lambda}\right)^2
=\dclass(\widehat G,\lambda).
\]

\subsection{Prediction stability}

We state our main stability result under strongly log concave response noise. Recall that a probability density $\nu$ on $\mathbb R$ is \emph{$\alpha$-strongly log concave} if $\nu(z)\propto \exp(-V(z))$ for some twice differentiable potential $V$ with $V''(z)\ge \alpha>0$ in the interior of its convex support, with the corresponding Brascamp--Lieb variance inequality.

\begin{theorem}[Prediction stability]
\label{thm:main_prediction_stability}
Suppose the response noise variables \(\xi_1,\ldots,\xi_n\) are independent
and each \(\xi_i\) has an \(\alpha\)-strongly log concave density. Let
\(A_\lambda\) satisfy the selection assumption above, set
\(\widehat\theta=A_\lambda(Y)\), and set
\(\widehat\theta^{(i)}=A_\lambda(Y^{(i)})\), where \(Y^{(i)}\) replaces only
\(Y_i\) by an independent copy \(Y_i'\). Then
\begin{align*}
&\frac1n\sum_i\mathbb E\bigl[
 (f(x_i;\widehat\theta)-f(x_i;\widehat\theta^{(i)}))^2\bigr]\\
&\hspace{18mm}\le\frac{4\,\mathbb E[\deff(\widehat\theta;\lambda)]}{\alpha n}.
\end{align*}
\end{theorem}

A canonical example is Gaussian noise, $\xi_i\sim \cN(0,\sigma^2)$, where $V''\equiv 1/\sigma^2$, so $\alpha=1/\sigma^2$.

\begin{proof}[Proof sketch.]
The proof technique is inspired by results in
point variance estimation for nonlinear least squares~\citep[Sec. D.1 and Lemma 1]{kuzborskij2025pointwiseconfidenceestimationnonlinear}.
First, \(\widehat\theta\) and \(\widehat\theta^{(i)}\) have the same marginal law, because replacing \(Y_i\) by an iid copy preserves the distribution of the sample.  Thus, for any fixed input \(x_j\), the two random predictions have the same mean.  Centering them at this common mean gives
\begin{equation}
\label{eq:main_centering_step}
\mathbb E[(h_j(Y)-h_j(Y^{(i)}))^2]
\le4\operatorname{Var}(h_j(Y)).
\end{equation}
Next, the variance is controlled by the response sensitivity of the fitted local minimizer and selected learning rule.  The selection assumption supplies the weak identity
\begin{equation}
\label{eq:main_if_identity}
\partial_{Y_k}h_j(Y)=\frac1n g_j^\top\widehat H_\lambda^{-1}g_k.
\end{equation}
The multivariate Brascamp--Lieb inequality \cite{brascamp1976, liebmodern} applied to the response noise vector then yields
\begin{align}
\operatorname{Var}(h_j(Y))
&\le\alpha^{-1}\mathbb E\|\nabla_Y h_j(Y)\|_2^2\notag\\
&=\frac{1}{\alpha n}\mathbb E
\|g_j\|_{\widehat H_\lambda^{-1}\widehat G\widehat H_\lambda^{-1}}^2.
\label{eq:main_variance_bound}
\end{align}
Finally, average \eqref{eq:main_variance_bound} over the training inputs.  Since \(g_i=\nabla_\theta f(x_i;\widehat\theta)\),
\[
\frac1n\sum_j\|g_j\|_{\widehat H_\lambda^{-1}\widehat G\widehat H_\lambda^{-1}}^2
=\deff(\widehat\theta;\lambda).
\]
Combining this with \eqref{eq:main_centering_step} proves Theorem~\ref{thm:main_prediction_stability}.  Appendix~\ref{app:section2_proofs}  (Corollary~\ref{cor:log_concave_brascamp_lieb}) contains a detailed proof.
\end{proof}

\subsection{From stability to generalization}
\label{sec:stab-to-gen}

Recall the population risk \(L(\theta)\) and generalization gap \(\gen_D(\theta)\) defined above. We now combine prediction stability with the standard replace one identity for square loss to control \(\E[\gen_D(\widehat\theta)]\).

\begin{corollary}[Square loss generalization]
\label{cor:main_gen_gap}
Under the assumptions of Theorem~\ref{thm:main_prediction_stability}, write
\(D_n=\mathbb E[\deff(\widehat\theta;\lambda)]\) and
\(R_n=\mathbb E[\widehat L(\widehat\theta)]\). Then
\[
\mathbb E[\gen_D(\widehat\theta)]
\le\sqrt{\frac{8R_nD_n}{\alpha n}}+\frac{2D_n}{\alpha n}.
\]
Consequently, by Young's inequality, for every \(\eta>0\),
\[
\mathbb E[\gen_D(\widehat\theta)]
\le\eta R_n+\left(2+\frac2\eta\right)\frac{D_n}{\alpha n}.
\]
\end{corollary}

\begin{proof}[Proof sketch.]
Combine the standard replace one identity for square loss \citep{shalev2014understanding} with the quadratic decomposition \(\ell(a,y)-\ell(b,y)=(b-y)(a-b)+\tfrac12(a-b)^2\), then apply Cauchy--Schwarz and Theorem~\ref{thm:main_prediction_stability}. The $\eta$ form follows by Young's inequality. The full argument is in Appendix~\ref{app:section2_proofs} (Corollary~\ref{cor:gen_gap}).
\end{proof}

The first term is the smoothness price paid by square loss when the training residual is nonzero. The second term is the stability price.  Since \(\alpha\) is bounded away from zero, and if \(R_n=O(1)\), the bound is \(O(\sqrt{D_n/n})\). In the near interpolation regime \(R_n=O(D_n/(\alpha n))\), the bound may even attain the faster rate \(O(D_n/(\alpha n))\).

\section{How Data Geometry Affects Effective Dimension}
\label{sec:geometry_main}

\begin{figure}[t]
\centering
\resizebox{\linewidth}{!}{
\begingroup
\definecolor{geomblue}{HTML}{34699A}
\definecolor{geomteal}{HTML}{24867A}
\definecolor{geomorange}{HTML}{D18A38}
\definecolor{geomink}{HTML}{263442}
\definecolor{geomgray}{HTML}{DCE3E9}
\begin{tikzpicture}[x=1cm,y=1cm,font=\small,text=geomink,
  flow/.style={-{Latex[length=2.2mm]},line width=.8pt,geomink},
  head/.style={anchor=west,font=\small\bfseries},
  note/.style={font=\footnotesize,align=center}]
\path[use as bounding box] (0,0) rectangle (17,6.4);
\foreach \a/\b in {0/5.1,5.95/11.05,11.9/17}{
  \fill[geomgray!18,rounded corners=3pt] (\a,1.5) rectangle (\b,6.35);
}
\node[head] at (.25,5.99) {(a) Input geometry};
\node[head] at (6.2,5.99) {(b) Trained Jacobian features};
\node[head] at (12.15,5.99) {(c) Spectral compression};
\node[note] at (2.55,5.4) {$m$-dimensional support in $\mathbb R^d$};
\node[note] at (8.5,5.4) {$g_i=\nabla_\theta f(x_i;\widehat\theta)\in\mathbb R^p$};
\node[note] at (14.45,5.4) {A $K$-dimensional approximation};
\path[fill=geomblue!25,draw=geomblue,line width=.6pt]
 (.45,3.65) .. controls (.85,4.7) and (1.32,4.95) .. (1.86,4.35)
 -- (1.86,3.28) .. controls (1.32,3.85) and (.85,3.65) .. (.45,2.7) -- cycle;
\path[fill=geomteal!24,draw=geomteal,line width=.6pt]
 (1.86,4.35) .. controls (2.38,3.83) and (2.83,3.26) .. (3.3,3.57)
 -- (3.3,2.5) .. controls (2.83,2.19) and (2.38,2.76) .. (1.86,3.28) -- cycle;
\path[fill=geomorange!25,draw=geomorange,line width=.6pt]
 (3.3,3.57) .. controls (3.8,3.9) and (4.25,4.5) .. (4.68,4.08)
 -- (4.68,3.02) .. controls (4.25,3.44) and (3.8,2.84) .. (3.3,2.5) -- cycle;
\foreach \x/\y in {.72/3.43,.99/3.87,1.19/4.12,1.37/3.71,1.66/3.9,.98/3.51,1.62/4.17}
 \fill[geomblue] (\x,\y) circle (1.5pt);
\foreach \x/\y in {2.04/3.87,2.25/3.28,2.44/3.58,2.66/2.99,2.96/3.1,3.12/3.39,2.7/3.42}
 \fill[geomteal] (\x,\y) circle (1.5pt);
\foreach \x/\y in {3.52/3.2,3.76/3.54,3.94/3.76,4.17/3.5,4.39/3.92,4.48/3.37,3.64/3.15}
 \fill[geomorange] (\x,\y) circle (1.5pt);
\draw[geomblue,dashed] (1.28,3.95) ellipse (.43 and .48);
\node[note,geomblue] at (2.55,2.25) {Piece $U_r$, local constant $L_r$};
\node[note] at (2.55,1.79) {Covers on each occupied piece};
\foreach \x/\y/\c in {7.1/4.0/geomblue,9.73/4.05/geomteal,8.45/2.65/geomorange}{
 \fill[\c!8] (\x,\y) circle (.69);
 \draw[\c,dashed,line width=.7pt] (\x,\y) circle (.69);
 \fill[\c] (\x,\y) circle (1.8pt);
 \foreach \dx/\dy in {-.3/-.12,-.2/.24,.13/.34,.34/.01,.13/-.28,-.13/-.35,.32/.25}
  \fill[\c] (\x+\dx,\y+\dy) circle (1.5pt);
}
\draw[geomblue,line width=.65pt] (7.1,4) -- (7.73,4.28);
\node[font=\footnotesize,geomblue,fill=geomgray!18,inner sep=1pt] at (7.59,4.42) {$\varepsilon$};
\node[note] at (8.5,1.79) {$K=\mathcal C_J(\varepsilon)$ cover centers};
\draw[flow] (5.22,3.63) -- (5.82,3.63);
\node[note] at (5.53,4.07) {$g$};
\draw[flow] (11.17,3.63) -- (11.78,3.63);
\node[note] at (11.48,4.07) {$\widehat G$};
\draw[geomink!60,line width=.5pt] (12.45,2.3) -- (16.55,2.3);
\draw[geomink!60,line width=.5pt] (12.45,2.3) -- (12.45,4.82);
\node[font=\footnotesize,anchor=south] at (12.5,4.85) {$\mu_j(\widehat G)$};
\foreach \j/\h/\c in {0/2.03/geomblue,1/1.78/geomteal,2/1.48/geomorange,3/.54/geomgray,4/.45/geomgray,5/.35/geomgray,6/.29/geomgray,7/.24/geomgray,8/.20/geomgray,9/.16/geomgray,10/.12/geomgray,11/.10/geomgray}
 \fill[\c] (12.63+\j*.31,2.3) rectangle +(0.22,\h);
\draw[geomink!55,dashed] (13.51,2.23) -- (13.51,4.64);
\node[font=\footnotesize] at (13.03,2.04) {$j\le K$};
\node[font=\footnotesize] at (15.13,2.04) {$j>K$};
\node[note] at (15.15,4.21) {$\displaystyle\sum_{j>K}\mu_j\le\varepsilon^2$};
\node[note] at (14.45,1.67) {Low rank part + small residual};
\draw[geomgray] (0,.99) -- (17,.99);
\node[align=center,font=\normalsize] at (8.5,.53) {
 $\displaystyle d_{\mathrm{eff}}(\widehat\theta;\lambda)
 \;\le\;\gamma_{\mathrm{pred}}(\lambda)^{-2}\,d_{\mathrm{lin}}(\widehat G,\lambda)
 \;\le\;\gamma_{\mathrm{pred}}(\lambda)^{-2}\!\left(K+\frac{\varepsilon^2}{\lambda}\right),\qquad\lambda>0$};
\end{tikzpicture}
\endgroup}
\caption{\textbf{From input geometry to generalization.} Covers of the occupied
input pieces transfer through the trained Jacobian to feature covers.
These control the classical effective dimension at the training ridge level \(\lambda\).
The factor \(\gamma_{\rm pred}(\lambda)^{-2}\) transfers the bound to nonlinear
prediction sensitivity.}
\label{fig:section3overview}
\end{figure}

Section~\ref{sec:stability_setup} reduces generalization to controlling \(\deff(\widehat\theta;\lambda)\).  We now investigate when this quantity is small.  Our main result here is that the trained Jacobian features that can be well represented by a small number of centers have relatively small effective dimension, when the prediction visible curvature factor is controlled.

Let \(\mathcal R=\operatorname{Range}(\widehat G)\) and
\(\mathcal N=\ker(\widehat G)\). Write \(G_R\) for the restriction of
\(\widehat G\) to \(\mathcal R\). The Hessian profiled onto this subspace is defined to be
\[
H_{\rm pred}=H_{RR}-H_{RN}H_{NN}^{-1}H_{NR},
\]
where the blocks are those of \(\widehat H_\lambda\) in
\(\mathcal R\oplus\mathcal N\). If \(\mathcal N=\{0\}\), take
\(H_{\rm pred}=\widehat H_\lambda\). For \(\lambda>0\), define
\[
\gamma_{\rm pred}(\lambda)\!=\!\lambda_{\min}\!\left(
(G_R+\lambda I)^{-1/2}H_{\rm pred}(G_R+\lambda I)^{-1/2}\right).
\]
Positive definiteness of \(\widehat H_\lambda\) gives
\(\gamma_{\rm pred}(\lambda)>0\). When \(\widehat\Delta=0\), we have
\(H_{\rm pred}=G_R+\lambda I\) and \(\gamma_{\rm pred}(\lambda)=1\),
recovering \(\deff=\dclass(\widehat G,\lambda)\).
If \(\widehat G=0\), then \(\deff=0\), and we set
\(\gamma_{\rm pred}(\lambda)=1\) by convention.

\begin{proposition}[Prediction space curvature comparison]
\label{prop:nonlin_to_lin_main}
If \(\widehat H_\lambda\succ0\) and \(\lambda>0\), then
\begin{equation}
\label{eq:main_nonlin_to_lin}
\deff(\widehat\theta;\lambda)
\le\gamma_{\rm pred}(\lambda)^{-2}\dclass(\widehat G,\lambda).
\end{equation}
\end{proposition}

The proof, given in Appendix~\ref{app:prediction_curvature}, (Proposition~\ref{prop:gamma_pred}), follows from the Schur complement identity
\(\deff=\operatorname{tr}((H_{\rm pred}^{-1}G_R)^2)\) and
\(H_{\rm pred}\succeq\gamma_{\rm pred}(\lambda)(G_R+\lambda I)\), together with monotonicity of the squared ridge filter trace under the L\"owner order. See also Figure \ref{fig:section3overview} for a diagram visualizing the results of this section.

\subsection{Covers for transferring input geometry to Jacobian geometry}

For \(S\subseteq\mathcal X\) in a metric space \((\mathcal X,d)\), write
\(B_d(c,\varepsilon):=\{x\in\mathcal X:d(x,c)\le\varepsilon\}\). An
\(\varepsilon\)-cover of \(S\) is a finite set \(C\) such that
\(S\subseteq\bigcup_{c\in C}B_d(c,\varepsilon)\). Then, the smallest possible \(|C|\) is defined to be the covering number
\(\mathcal N(S,d,\varepsilon)\). Centers may be arbitrary points of \(\mathcal X\). Define
the empirical Jacobian feature covering number
\[
\mathcal C_J(\varepsilon)
:=\mathcal N(\{g_i\}_{i=1}^n,\|\cdot\|_2,\varepsilon),
\qquad g_i=\nabla_\theta f(x_i;\widehat\theta).
\]

\begin{proposition}[Cover bound]
\label{prop:cover_main}
Under Proposition~\ref{prop:nonlin_to_lin_main},
\begin{equation}
\label{eq:main_cover_bound}
\deff(\widehat\theta;\lambda)
\le \gamma_{\rm pred}(\lambda)^{-2}
\inf_{\varepsilon>0}\left[
\mathcal C_J(\varepsilon)+\frac{\varepsilon^2}{\lambda}
\right].
\end{equation}
\end{proposition}

\begin{proof}[Proof sketch.]
The proof projects each $g_i$ onto a subspace of dimension at most $K$ spanned by an $\varepsilon$-cover of size $K=\mathcal C_J(\varepsilon)$, bounds the sum of the tail eigenvalues of $\widehat G$ by $\varepsilon^2$, and applies the pointwise bound $(\mu/(\mu+\lambda))^2\le\min\{1,\mu/\lambda\}$. Proposition~\ref{prop:nonlin_to_lin_main} supplies the curvature factor. The full proof is given in Appendix~\ref{app:section3_proofs} (Theorem~\ref{thm:covering_J_section3}).
\end{proof}

The covering number \(\mathcal C_J(\varepsilon)\) can be controlled a priori by covering the inputs and transferring the cover through a stable Jacobian map. This transfer is most powerful when the data lies on a low dimensional manifold and the trained Jacobian is regular on appropriate pieces.

\begin{proposition}[Manifold and piecewise Lipschitz feature covers]
\label{prop:manifold_main}
Suppose the training inputs lie on a compact \(C^1\) embedded submanifold
\(\mathcal M\subset\mathbb R^d\) of dimension \(m\ge1\), with manifold covering constant \(C_{\mathcal M}\). Suppose
the data occupied part of $\mathcal M$ is covered by pieces $U_1,\dots,U_M\subset\mathcal M$ on which the trained Jacobian map $g(x)=\nabla_\theta f(x;\widehat\theta)$ is $L_r$-Lipschitz, where \(L_r\ge0\), and set
\(A=C_{\mathcal M}\sum_{r=1}^M L_r^m\). Then
\begin{equation}
\label{eq:main_feature_cover_from_manifold}
\mathcal C_J(\varepsilon)\le M+A\varepsilon^{-m},
\qquad \varepsilon>0,
\end{equation}
and consequently, under Proposition~\ref{prop:nonlin_to_lin_main}, optimizing over \(\varepsilon\) yields
\begin{equation}
\label{eq:main_optimized_manifold}
\begin{split}
\deff(\widehat\theta;\lambda)
\le{}&\gamma_{\rm pred}(\lambda)^{-2}\\[-2pt]
&{}\times\left[
M+C_m A^{2/(m+2)}\lambda^{-m/(m+2)}
\right],
\end{split}
\end{equation}
where \(C_m=(1+m/2)(2/m)^{m/(m+2)}\).
\end{proposition}

The full proof of Proposition~\ref{prop:manifold_main} is given in Appendix~\ref{app:section3_proofs} (Theorem~\ref{thm:piecewise_regular_section3} and Corollary~\ref{cor:optimized_nonlin_section3}). The cover transfer works as follows. For $L_r>0$, covering each piece $U_r$ at input radius $\varepsilon/L_r$ with centers in $U_r$, the Lipschitz property turns each input ball intersected with $U_r$ into an $\varepsilon$-ball in Jacobian feature space, so that summing over pieces gives the bound. Note that constant pieces require only one center. This provides a direct route for the activation region complexity $M$ to control the generalization error.

\begin{figure}[t]
  \centering
  \includegraphics[width=\textwidth]{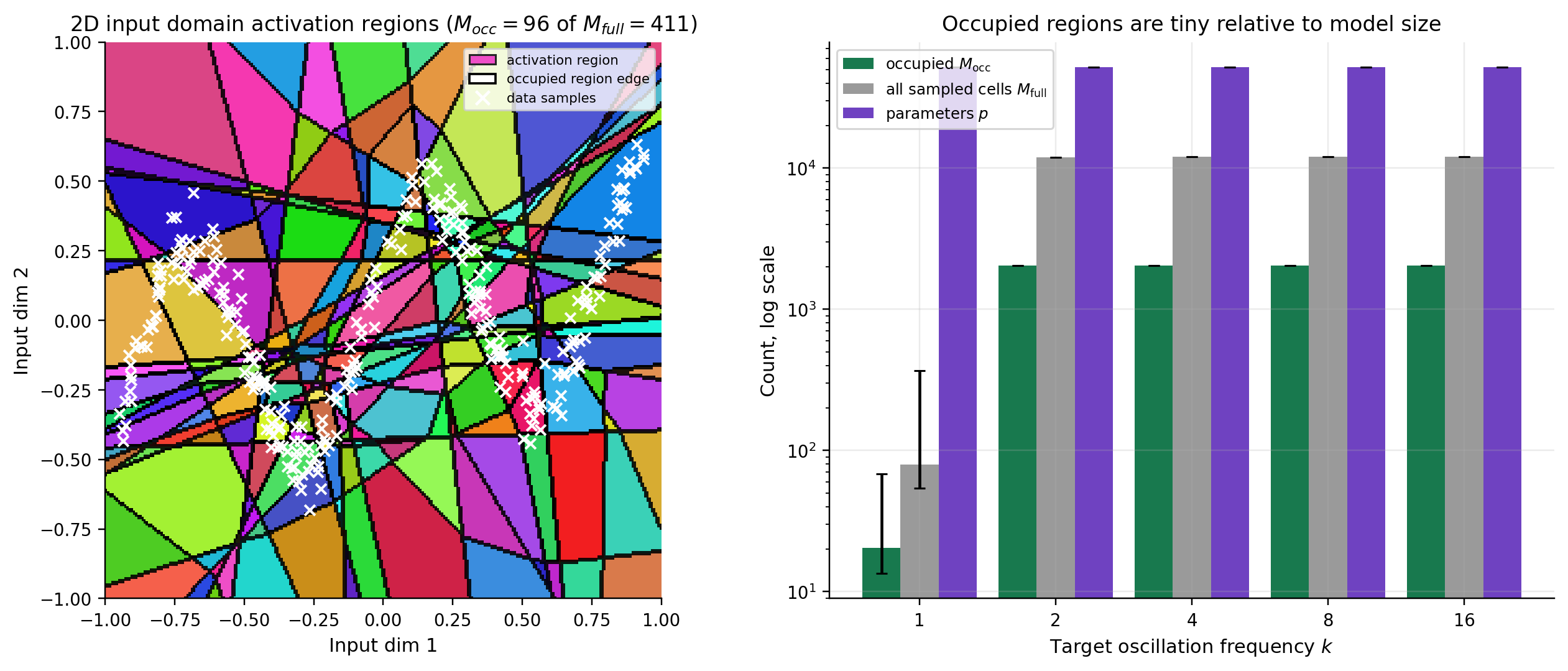}
  \caption{\textbf{The number of occupied activation regions is small.}  The left panel visualizes a two dimensional ReLU input domain partition. Only cells containing data matter for the bound.  The right panel measures the same count based on the frequency. This separate experiment uses a product torus with $m=4$, $d=100$, and $n=2048$.  The number of occupied regions is orders of magnitude smaller compared with the parameter count \(p=51{,}712\).}
  \label{fig:main_activation_regions}
\end{figure}

\subsection{One hidden layer ReLU networks}

For a one hidden layer ReLU network $f(x;\theta)=q^{-1/2}a^\top\sigma(Wx)$ with $q$ hidden units, weights $W\in\mathbb{R}^{q\times d}$, output weights $a\in\mathbb{R}^q$, and parameter $\theta=(\mathrm{vec}(W),a)$, activation stable regions provide an explicit realization of the pieces $U_r$ from Proposition~\ref{prop:manifold_main}.

\begin{definition}[Activation stable sets]
\label{def:activation-stable-region}
We define an \emph{activation stable set} as a set $U\subset\mathbb R^d$ on which every preactivation
\(\widehat w_j^\top x\) is nonzero and the activation pattern
\((\mathbf1\{\widehat w_j^\top x>0\})_{j=1}^q\) is constant for all $x\in U$.
\end{definition}

The maximal ambient sets are open polyhedral cones. Their intersections
with \(\mathcal M\) may be disconnected, but the same linear Jacobian
formula holds throughout each intersection. Hence \(M\) may count
distinct gate patterns occupied by the sample.

\begin{proposition}[ReLU instantiation]
\label{prop:relu_main}
On an activation stable piece with diagonal gate matrix \(D_r\), the trained
Jacobian is linear in \(x\), with a Lipschitz constant satisfying
\[
L_r^2=\|S_r^{\rm norm}\|_{\rm op},\qquad
S_r^{\rm norm}=\frac{\widehat W^\top D_r\widehat W
+\|D_r\widehat a\|_2^2I_d}{q}.
\]
\(S_r^{\rm norm}\) measures local feature contraction.
At a differentiable stationary point with \(\lambda>0\),
\(S_r^{\rm norm}\preceq2\widehat L(0)(q\lambda)^{-1}I_d\).
\end{proposition}

The proof is given in Appendix~\ref{app:section3_proofs}, Propositions~\ref{prop:relu_local_metric_section3} and~\ref{prop:l2_regularization_contraction_section3}.

The number $M$ of activation stable regions can scale exponentially with depth and neuron count in the worst case \cite{montufar2014number}. Under the random initialization assumptions
of \citet{pmlr-v97-hanin19a}, the expected count along a one dimensional
subspace is only linear in the number of neurons. However, we experimentally demonstrate that the number of occupied patterns after training can be even much lower than this.

Combining these geometric bounds with Corollary~\ref{cor:main_gen_gap} gives a generalization bound in terms of the occupied region count, feature contraction, and prediction visible curvature. Appendix~\ref{app:combined_generalization} states the combined bound and its rates.

\section{Experiments}
\label{sec:experiments}
We empirically test the main claims suggested by or required for the theory.  Here, we present results showing ridge regularized neural networks induce a small number of activation regions in Figure \ref{fig:main_activation_regions}, and that plug in geometry scores track closely to observed gaps in synthetic and real data in Figure \ref{fig:main_supportive_bounds}.
Figure \ref{fig:main_deff_validation_cluster} demonstrates the tightness of the post training spectral comparison and that $d_{\rm lin}$ after training is much lower than at initialization. A separate response perturbation experiment examines branch selection and the influence identity.
Full details and additional experiments are deferred to Appendix~\ref{app:experimental_details}.

\section{Discussion}
\label{sec:discussion}
We showed that stability in ridge regularized nonlinear least squares is controlled by a learned geometry effective dimension, yielding data dependent bounds. Plug in scores closely track observed generalization gaps. We explored the shallow ReLU setting standard in NTK analysis, but replaced the initialization Jacobian with our more general effective dimension, which let us study the effect of feature learning and activation region complexity that fixed kernel analyses cannot. Prediction visible curvature connects this learned feature geometry to generalization at the training regularization $\lambda>0$.

We also note some limitations.  Our analysis focuses on fixed design square loss regression with strongly log concave noise, the selection assumption, and a local nondegeneracy condition. Extending our framework to more general settings raises some technical challenges. 

In this paper, we chose to focus on the fixed design setting to isolate the core mechanism as cleanly as possible but moving to random design would require accounting for how the learned Jacobian geometry interacts with the data distribution, and the sensitivity analysis based on implicit function theory does not carry over directly to this setting. Extending the analysis to other loss functions is also nontrivial. The square loss plays a central role in enabling both the stability argument and the associated variance control. Appendix~\ref{app:other_losses} gives a smooth loss extension through a weighted Jacobian Gram matrix. Another obstacle lies in capturing implicit regularization from gradient based training. Our current approach relies on the sensitivity of a well defined selected local minimizer via an inverse or pseudoinverse Hessian characterization, whereas gradient based methods typically produce solutions shaped by the entire optimization trajectory and may not always admit such a representation. Developing tools to capture this algorithm dependent geometry, and to relate it to stability and generalization, remains an important direction for future work.

\label{main-text-end}
\section*{AI Use Statement}
Generative AI tools assisted with writing code for experiments, editing LaTeX, and producing Figure~\ref{fig:section3overview}. The authors are responsible for the final content.
\clearpage
\bibliographystyle{plainnat}
\bibliography{learning_capitalized}
\clearpage
\appendix
\section*{Appendix Contents}
\setcounter{tocdepth}{2}
\begingroup
\setlength{\parskip}{0pt}
\makeatletter
\@starttoc{apx}
\makeatother
\endgroup
\clearpage
\let\standardaddcontentsline\addcontentsline
\renewcommand{\addcontentsline}[3]{%
  \ifstrequal{#1}{toc}{\standardaddcontentsline{apx}{#2}{#3}}{\standardaddcontentsline{#1}{#2}{#3}}}
\section{Assumptions and Extensions}
\label{app:assumptions}

This appendix develops the assumptions and geometric quantities used in Section~\ref{sec:stability_setup}. Sections~\ref{app:selection_necessity}--\ref{app:influence_residual} concern prediction regularity. Section~\ref{app:prediction_curvature} connects prediction visible curvature with feature geometry. Sections~\ref{app:other_losses}--\ref{app:variance_inequalities} discuss smooth losses, fixed design, and variance inequalities. Section~\ref{app:margin_equiv} gives the differentiability condition used for the ReLU instantiation.

\subsection{Why some selection regularity is necessary}
\label{app:selection_necessity}

The effective dimension describes the infinitesimal sensitivity of the fitted predictor along a selected branch. Without control of branch switching, two solutions can each have negligible within branch sensitivity while the learning rule jumps between predictors with very different outputs. The obstruction is the same as that for applying a Poincar\'e inequality to a step function. Its classical derivative vanishes almost everywhere, yet its variance is positive because its distributional derivative contains a singular jump.

The following example shows that nondegeneracy of every selected point does not by itself suffice.

\begin{proposition}[Necessity of branch regularity]
\label{prop:selection_necessity}
Let $n=p=1$, $\lambda=0$, and
\[
f(\theta)=\theta+2\sin\theta,\qquad \ell_y(\theta)=\tfrac12(f(\theta)-y)^2 .
\]
Put $\theta_-=\tfrac{2\pi}{3}$, $\theta_+=\tfrac{10\pi}{3}$, and $c_\pm=f(\theta_\pm)$, so that $c_-=\tfrac{2\pi}{3}+\sqrt3$ and $c_+=\tfrac{10\pi}{3}-\sqrt3$. Then for every $y\in(c_-,c_+)$ both $\theta_-$ and $\theta_+$ are strict nondegenerate local minimizers of $\ell_y$. Let $Y$ have a symmetric truncated Gaussian law on $(c_-,c_+)$ centred at $2\pi$, which is strongly log concave on its convex support, and let
\[
A(Y)=\theta_-\ \text{ if } Y<2\pi,\qquad A(Y)=\theta_+\ \text{ if } Y\ge 2\pi .
\]
Then $\widehat G=(f'(\theta_\pm))^2=0$ and hence $\deff=0$ at every selected point, while for an independent copy $Y'$,
\[
\E\bigl[(f(A(Y))-f(A(Y')))^2\bigr]=\tfrac12(c_+-c_-)^2>0 .
\]
\end{proposition}

\begin{proof}
Since $f'(\theta)=1+2\cos\theta$, we have $f'(\theta_\pm)=0$. Moreover
$\ell_y''(\theta)=(f'(\theta))^2+(f(\theta)-y)f''(\theta)$ with $f''(\theta)=-2\sin\theta$, so $f''(\theta_-)=-\sqrt3$ and $f''(\theta_+)=\sqrt3$. For $y\in(c_-,c_+)$ we have $c_--y<0$ and $c_+-y>0$, whence $\ell_y''(\theta_\pm)>0$ and both points are strict nondegenerate local minimizers. The Jacobian feature at either point is $f'(\theta_\pm)=0$, so $\widehat G=0$ and $\deff=0$. Finally $A(Y)$ takes the two values with probability $1/2$ each by symmetry, and $f(A(Y))\in\{c_-,c_+\}$, so $\E[(f(A(Y))-f(A(Y')))^2]=2\cdot\tfrac14(c_+-c_-)^2$.
\end{proof}

The obstruction persists for every fixed $\lambda>0$. Restrict the response law to a symmetric interval compactly contained in $(c_-,c_+)$, replace $f$ by $f_M=Mf$ and the response by $MY$, and divide the objective by $M^2$. The ridge then contributes only $\lambda\theta^2/(2M^2)$, so the two nondegenerate branches persist for large $M$ by the implicit function theorem. Their prediction gap is of order $M$, whereas along either branch $\widehat G_M=O(M^{-2})$, $\widehat H_{\lambda,M}=\Theta(M^2)$ and $\deff{}_{,M}=O(M^{-8})$, while the strong log concavity parameter of the scaled response law is $\alpha_M=\Theta(M^{-2})$. Hence $\alpha_M^{-1}\deff{}_{,M}=O(M^{-6})$ while the squared branch switching gap is $\Theta(M^2)$. A stability bound depending only on the within branch effective dimension can not hold without either prediction regularity or an explicit branch switching penalty.

\subsection{When the selection assumption holds}
\label{app:selection_holds}

Under local $C^2$ regularity of the predictors, the local smoothness and the influence identity hold along the branch through a nondegenerate local minimum. Writing $F(\theta,y)=\nabla_\theta L_{\lambda,y}(\theta)$, if $F(\theta_0,y_0)=0$ and $D_\theta F(\theta_0,y_0)=H_\lambda(y_0,\theta_0)\succ0$, the implicit function theorem supplies neighbourhoods of $y_0$ and $\theta_0$ and a unique $C^1$ map $y\mapsto\theta_y$ with $F(\theta_y,y)=0$. Positive definiteness persists after shrinking, and differentiating the identity gives \eqref{eq:selection-influence}. This is Lemma~\ref{lem:implicit_response_derivative}. The nonlocal part of the assumption concerns regularity of the selected predictions across the response distribution.

Consider a concrete nonconvex ReLU example to help make this explicit: Fix $\lambda>0$. Take $n=1$, $x=1$, one hidden unit, $f(x;a,w)=a\sigma(wx)$, and
\[
L_{\lambda,y}(a,w)=\tfrac12(aw-y)^2+\tfrac\lambda2(a^2+w^2),\qquad w>0 .
\]
On any response support with $y>\lambda+\delta$ for some $\delta>0$, define $A_\lambda(y)=(\sqrt{y-\lambda},\sqrt{y-\lambda})$. This is a $C^1$ branch with a positive activation margin. At the selected point the gradient vanishes and
\[
\widehat H_\lambda=\begin{pmatrix} y & y-2\lambda\\ y-2\lambda & y\end{pmatrix},
\]
whose eigenvalues are $2(y-\lambda)>0$ and $2\lambda>0$. The selected solutions therefore form a nondegenerate branch even though the objective is globally nonconvex, and a strongly log concave response law supported above $\lambda+\delta$ with finite second moment gives an explicit probabilistic example satisfying the selection assumption. On this branch the fitted prediction is $h(y)=y-\lambda$ and $\deff=1$.

More generally, under uniform smooth dependence and convergence conditions, a fixed initialization optimizer induces such a local branch when its limit is an exponentially stable local minimum and response perturbations do not cross a basin boundary. For a ReLU network one additionally needs a positive activation margin, or sufficiently regular transverse activation boundary crossings, so that the prediction map remains in $W^{1,2}(\mu)$. Appendix~\ref{app:exp_selection} tests this experimentally on some nonconvex examples.

\subsection{Flat symmetry directions and the pseudoinverse}
\label{app:flat_symmetry}

The prediction level formulation extends to a smooth manifold of minima, which is the relevant case for parameterizations with continuous symmetries. Suppose that for $\mu$-almost every $y$ the selected predictions agree locally with those of a $C^1$ stationary branch $y\mapsto\theta_y$ (possibly after fixing a smooth local gauge), that $h_j\in W^{1,2}(\mu)$, and that $H_\lambda(y)\succeq0$ has locally constant rank with
\[
g_i(y)\in\Range\bigl(H_\lambda(y)\bigr)\qquad\text{for all }i .
\]
Differentiating stationarity gives $H_\lambda\,\partial_{y_k}\theta_y=\tfrac1n g_k$, and every solution has the form
\[
\partial_{y_k}\theta_y=\tfrac1n H_\lambda\pinv g_k+v_k,\qquad v_k\in\ker(H_\lambda).
\]
Since $g_j\in\Range(H_\lambda)=\ker(H_\lambda)^\perp$, the null component is invisible to the predictions, and therefore
\[
\partial_{y_k}h_j(y)=\tfrac1n g_j(y)^\top H_\lambda(y)\pinv g_k(y).
\]
Defining $\dpred:=\tr\bigl((H_\lambda\pinv \widehat G)^2\bigr)$, the Brascamp--Lieb argument of Theorem~\ref{thm:main_prediction_stability} yields the same conclusion with $\deff$ replaced by $\dpred$,
\[
\frac1n\sumin\E\bigl[(h_i(Y)-h_i(Y^{(i)}))^2\bigr]\le \frac{4}{\alpha n}\E[\dpred].
\]
These hypotheses hold locally for a family of minima whose predictions are constant along the minimum manifold, with a smooth transverse gauge. The null directions are then function preserving parameter symmetries. The experiment in Appendix~\ref{app:exp_selection} measures the prediction energy of numerical null directions.

\subsection{Extension to approximate influence identities}
\label{app:influence_residual}

The stability argument also accommodates deviations from the exact influence identity. Under the same noise assumptions and Sobolev regularity, the following extension separates the effective-dimension contribution from an explicit penalty for the discrepancy between the weak prediction derivative and the Hessian-based influence formula. This quantifies how approximation errors affect the stability guarantee.

An approximate influence identity contributes an additional term to the stability bound. For $h_j\in W^{1,2}(\mu)$, define the residual between the weak prediction derivative and the Hessian based influence formula by
\[
e_{jk}:=\partial_{y_k}h_j-\tfrac1n g_j^\top\widehat H_\lambda^{-1}g_k,
\qquad e_j=(e_{j1},\ldots,e_{jn}).
\]
Assume $e_j\in L^2(\mu;\mathbb R^n)$. For any $\eta>0$, the argument of Theorem~\ref{thm:main_prediction_stability} together with $(a+b)^2\le(1+\eta)a^2+(1+\eta^{-1})b^2$ gives
\[
\frac1n\sumin\E\bigl[(h_i(Y)-h_i(Y^{(i)}))^2\bigr]
\le
\frac{4(1+\eta)}{\alpha n}\E[\deff]
+\frac{4(1+\eta^{-1})}{\alpha n}\sumin\E\bigl[\|e_i\|_2^2\bigr].
\]
When every residual vanishes, letting \(\eta\downarrow0\) recovers Theorem 3. More generally, the bound quantifies the additional cost of an approximate influence identity. In particular, if \(\sum_i\mathbb E\|e_i\|_2^2\) is bounded by a constant multiple of \(\mathbb E[d_{\mathrm{eff}}]\), choosing \(\eta=1\) preserves the order of the effective-dimension stability bound.
\subsection{Prediction visible curvature and feature geometry}
\label{app:prediction_curvature}

The stability and generalization results require the selection assumption and positive definite fitted Hessians. For $\lambda>0$, we compare curvature on the directions visible to the predictions with the ridge metric at the training regularization $\lambda$, which connects the effective dimension to feature geometry.

\paragraph{Prediction visible reduction.}
Let $\mathcal R=\Range(\widehat G)$, $\mathcal N=\ker(\widehat G)$, and let $G_R$ be the restriction of $\widehat G$ to $\mathcal R$. If $\mathcal R=\{0\}$, then $\deff=0$ and all bounds below hold trivially, with $\gpred(\lambda)=1$. Otherwise $G_R\succ0$. In the orthogonal decomposition $\mathcal R\oplus\mathcal N$, write the blocks of $\widehat H_\lambda$ as $H_{RR},H_{RN},H_{NR},H_{NN}$. Positive definiteness gives $H_{NN}\succ0$ and the positive definite Schur complement
\[
H_{\mathrm{pred}}:=H_{RR}-H_{RN}H_{NN}^{-1}H_{NR},\qquad
(\widehat H_\lambda^{-1})_{RR}=H_{\mathrm{pred}}^{-1}.
\]
When $\mathcal N=\{0\}$, take $H_{\mathrm{pred}}=\widehat H_\lambda$.
Because $\widehat G$ is zero on $\mathcal N$, block multiplication gives the exact identity
\begin{equation}
\label{eq:profiled_deff}
\deff(\widehat\theta;\lambda)=\tr\bigl((H_{\mathrm{pred}}^{-1}G_R)^2\bigr).
\end{equation}
Thus curvature in Jacobian null directions enters through the coupling term $H_{RN}H_{NN}^{-1}H_{NR}$.

\begin{proposition}[Prediction space curvature comparison]
\label{prop:gamma_pred}
Suppose $\widehat H_\lambda\succ0$ and $\lambda>0$. Define
\[
\gpred(\lambda):=\lmin\!\left((G_R+\lambda I)^{-1/2}H_{\mathrm{pred}}(G_R+\lambda I)^{-1/2}\right)>0.
\]
Then
\begin{equation}
\label{eq:gamma_comparison}
\deff(\widehat\theta;\lambda)
\le \gpred(\lambda)^{-2}\dclass(\widehat G,\lambda).
\end{equation}
This implies
\begin{equation}
\label{eq:gamma_cover}
\deff(\widehat\theta;\lambda)
\le \gpred(\lambda)^{-2}
\inf_{\varepsilon>0}
\left(\mathcal C_J(\varepsilon)+\frac{\varepsilon^2}{\lambda}\right).
\end{equation}
\end{proposition}

\begin{proof}
Set $Q=G_R+\lambda I$. The definition gives $H_{\mathrm{pred}}\succeq\gpred(\lambda)Q$, so inversion and congruence by $G_R^{1/2}$ give
\[
0\preceq X:=G_R^{1/2}H_{\mathrm{pred}}^{-1}G_R^{1/2}
\preceq Y:=\gpred(\lambda)^{-1}G_R^{1/2}Q^{-1}G_R^{1/2}.
\]
For positive semidefinite $X\preceq Y$,
$\tr(Y^2)-\tr(X^2)=\tr((Y-X)(Y+X))\ge0$.
This trace inequality holds without commutativity. Taking traces and using \eqref{eq:profiled_deff} proves \eqref{eq:gamma_comparison}. The standard covering bound (Theorem~\ref{thm:covering_J_section3}) proves \eqref{eq:gamma_cover}.

\end{proof}

\paragraph{Singular Hessians.}
Under the prediction regularity and range conditions of Appendix~\ref{app:flat_symmetry}, the same comparison holds for $\dpred$. At a fixed fitted point, restrict all matrices to $\mathcal S=\Range(\widehat H_\lambda)$, on which $H_S\succ0$. Since $\Range(\widehat G)\subseteq\mathcal S$, decompose
$\mathcal S=\mathcal R\oplus(\mathcal S\cap\ker\widehat G)$ and form the Schur complement of $H_S$ there. Its inverse equals $(\widehat H_\lambda\pinv)_{RR}$, so the preceding proof applies without issue.

\subsection{General smooth losses}
\label{app:other_losses}

For a smooth loss, response sensitivity is governed by a weighted feature geometry. Consider
$\widehat L_\ell(\theta)=n^{-1}\sum_i\ell(f(x_i;\theta),Y_i)$ and let $H_\ell\succ0$ be the Hessian of its ridge regularized objective. Impose the selection assumption with the corresponding weak influence identity. Along a smooth stationary branch, differentiation gives
\[
\partial_{Y_k}\widehat\theta=-\frac1n H_\ell^{-1}g_k\,
\ell_{uy}(f(x_k;\widehat\theta),Y_k).
\]
Put $c_k=-\ell_{uy}(f(x_k;\widehat\theta),Y_k)$,
$G_{c^2}=n^{-1}\sum_k c_k^2g_kg_k^\top$, and
\[
D_\ell=\tr(H_\ell^{-1}G_{c^2}H_\ell^{-1}\widehat G).
\]
The same variance argument gives prediction stability
$S_\ell\le4\mathbb E[D_\ell]/(\alpha n)$, where $S_\ell$ is the left hand side of Theorem~\ref{thm:main_prediction_stability} for this learning rule. If $0<\mu_\ell\le\ell_{uu}\le M_\ell$ and $|\ell_{uy}|\le C_\ell$, set $G_w=n^{-1}\sum_i\ell_{uu}(f(x_i;\widehat\theta),Y_i)g_ig_i^\top$. Then $\widehat G\preceq G_w/\mu_\ell$, $G_{c^2}\preceq C_\ell^2G_w/\mu_\ell$, and
\[
D_\ell\le\frac{C_\ell^2}{\mu_\ell^2}\tr((H_\ell^{-1}G_w)^2).
\]
This bounds stability by the analogous weighted geometry.

For the generalization step, scalar $M_\ell$-smoothness yields, when expectations are finite,
\[
\mathbb E[L_\ell(\widehat\theta)-\widehat L_\ell(\widehat\theta)]
\le\sqrt{Q_\ell S_\ell}+\frac{M_\ell}{2}S_\ell,
\qquad
Q_\ell:=\frac1n\sum_i\mathbb E[\ell_u(f(x_i;\widehat\theta),Y_i)^2].
\]
Indeed, apply the replace one identity, the smoothness inequality, and Cauchy--Schwarz. For square loss $c_i=1$, $G_w=\widehat G$, $Q_\ell=2\mathbb E[\widehat L]$, and $M_\ell=1$, recovering the stated result. For other losses, a bound on $Q_\ell$ in terms of the training risk gives the corresponding training risk form.

\subsection{Fixed versus random design}
\label{app:fixed_design}

Fixed design with square loss is a classical framework in mathematical statistics, and addresses a different notion of prediction risk from random design. Under fixed design the covariates are deterministic and performance is evaluated at the observed design points. Under random design the covariates are themselves random and the target is the expected error under the underlying distribution. Guarantees in one setting do not in general imply guarantees in the other, and relating them requires further assumptions connecting the empirical and population distributions, so separate analyses are the norm \cite{Wainwright19}.

Our fixed design result does not require a distributional assumption on the covariates and measures prediction error at the observed inputs, with fresh independent responses. It is relevant when the input set is fixed while its labels or measurements are noisy. It also connects to classical degrees of freedom. If $\widehat y=h(Y)$, then
\[
\tr(D_Yh)=\tr(\widehat H_\lambda^{-1}\widehat G),\qquad
\|D_Yh\|_F^2=\deff.
\]
Under homoscedastic Gaussian noise and the requisite Stein integrability, the expectation of the first quantity equals
$\sigma^{-2}\sum_i\operatorname{Cov}(Y_i,h_i(Y))$, the covariance based degrees of freedom \cite{Ye98,Efron04}. Our effective dimension is the squared sensitivity quantity, coinciding with the variance effective dimension for a linear ridge smoother.

A potential extension to Gaussian design $Y_i=f^\star(X_i)+\xi_i$ with independent Gaussian covariates and noise, would require an application of the product Poincar\'e inequality to the primitive variables $(X,\xi)$. For a sufficiently smooth target and a stationary branch regular in inputs and responses, the prediction derivatives involve $\nabla_x f$, $\nabla_x f^\star$, and mixed parameter--input derivatives. Controlling their squared integrability, including evaluation at an independent test input, would allow a replace one analysis for the full observation $(X_i,Y_i)$. We leave this extension to future work.

\subsection{Noise laws with variance inequalities}
\label{app:variance_inequalities}

Strong log concavity is an assumption on the response \emph{noise} and is weaker than Gaussianity. Recall that it enters in just one step of our derivation, and we suggest some ways in which it could be generalized. For a response density $\nu\propto e^{-V}$ with $V''\ge\alpha>0$, the Brascamp--Lieb inequality gives, for $h\in W^{1,2}(\nu)$,
\[
\operatorname{Var}(h(Y))\le \E\bigl[\nabla h(Y)^\top(\nabla^2V(Y))^{-1}\nabla h(Y)\bigr]\le\frac1\alpha\E\|\nabla h(Y)\|_2^2 ,
\]
and substituting the influence identity \eqref{eq:selection-influence} and averaging over $i$ produces Theorem~\ref{thm:main_prediction_stability}. For Gaussian noise this is the Gaussian Poincar\'e inequality with $\alpha=1/\sigma^2$.

More generally, suppose the product response law satisfies an unweighted Poincar\'e inequality
$\operatorname{Var}(h(Y))\le C_P\mathbb E\|\nabla h(Y)\|_2^2$ for the prediction maps, and assume $\mathbb E\|Y\|_2^2<\infty$. Then all the stated bounds hold with $\alpha^{-1}$ replaced by $C_P$, and with exactly the same $\deff$. For independent coordinates, a common upper bound on their Poincar\'e constants supplies $C_P$ by tensorization. Thus strong log concavity is a sufficient way to obtain the needed variance inequality, not a necessary one.

A weighted inequality $\operatorname{Var}(h)\le\mathbb E[\nabla h^\top W(Y)\nabla h]$ instead produces a weighted sensitivity trace and Brascamp--Lieb takes $W=(\nabla^2V)^{-1}$, as detailed in Corollary~\ref{cor:log_concave_brascamp_lieb}. Such inequalities can cover some heavy tailed laws, but require their own weights and integrability conditions, and weak Poincar\'e inequality with an additional remainder would also add a remainder to the stability bound, and so a full extension remains outside of the scope of this current paper. 

\subsection{Activation margin and twice differentiability}
\label{app:margin_equiv}

The ReLU calculations in Appendix~\ref{app:section3_proofs} use strict sample gates. In particular, the residual curvature calculation in Proposition~\ref{prop:rho_relu_section3} needs $\nabla_\theta^2f(x_i;\widehat\theta)$ to exist at every training input, so that $\widehat\Delta$ is well defined with fixed gate matrices. The positive activation margin condition $\widehat w_j^\top x_i\neq0$ for every nonzero training point $x_i$ and hidden unit $j$ is equivalent for the jointly parameterized bias free shallow network. Under the margin the gates are locally constant, so the fitted network is smooth in its parameters at every training input. Conversely, a unit sitting exactly on its activation boundary at a nonzero training input destroys the mixed $a$--$W$ second derivative, so twice differentiability forces the margin at such inputs. The zero input is an exception. Since $f(0;\theta)=0$ for all parameters, it does not need an activation margin condition.

The condition is not implied by local optimality. With $\lambda>0$, $n=1$, $x=1$, $q=1$ and $y=0$, the point $(\widehat a,\widehat w)=(0,0)$ is a strict nondegenerate minimizer of $\tfrac12(a\sigma(w))^2+\tfrac\lambda2(a^2+w^2)$ with objective Hessian $\lambda I$, yet $f$ is not twice differentiable there. It must therefore be assumed, and it is checkable directly at the trained point. Appendix~\ref{app:exp_selection} provides some demonstrations.

\clearpage
\section{Stability and Generalization Proofs}
\label{app:section2_proofs}

We establish here how to relate the notions of effective dimension we develop back to generalization error.

\subsection{Proof of Lemma~\ref{lem:implicit_response_derivative}}
\label{app:implicit_response_derivative}

\begin{proof}
Let \(F(\theta,y)=\nabla_\theta L_{\lambda,y}(\theta)\). Local \(C^2\) regularity makes \(F\) a \(C^1\) map. At \((\theta_0,y_0)\), \(F=0\) and \(D_\theta F=H_\lambda(y_0,\theta_0)\succ0\), so the implicit function theorem supplies a unique local \(C^1\) stationary branch \(y\mapsto\theta_y\). Its Hessian remains positive definite after shrinking the neighbourhood. Differentiating \(F(\theta_y,y)=0\) along this branch gives
\[
D_\theta F(\theta_y,y)\partial_{y_k}\theta_y+\partial_{y_k}F(\theta_y,y)=0 .
\]
Here \(D_\theta F(\theta_y,y)=H_\lambda(y,\theta_y)\), while
\[
\partial_{y_k}F(\theta_y,y)=-\frac1n\nabla_\theta f(x_k;\theta_y).
\]
Therefore \(\partial_{y_k}\theta_y=\frac1n H_\lambda(y,\theta_y)^{-1}\nabla_\theta f(x_k;\theta_y)\). Applying the chain rule to \(h_j(y)=f(x_j;\theta_y)\) gives
\[
\partial_{y_k}h_j(y)=\frac1n\nabla_\theta f(x_j;\theta_y)^\top H_\lambda(y,\theta_y)^{-1}\nabla_\theta f(x_k;\theta_y),
\]
as claimed.
\end{proof}

\subsection{Classical Effective Dimension $\dclass(\widehat G,\lambda)$}

Assume first that $Y_i=f^\star(x_i)+Z_i$, where $Z=(Z_1,\ldots,Z_n)\sim\cN(0,I_n)$, and that $\widehat\Delta=0$. This includes the linear in parameters case. Throughout this appendix we impose the selection assumption of Section~\ref{sec:stability_setup}. 

\begin{theorem}[Prediction stability via classical effective dimension]
\label{thm:stability_lin}
Under Gaussian noise and the selection assumption, with $\widehat\Delta=0$,
\[
\frac1n\sum_{i=1}^n
\mathbb E\Bigl[
\bigl(f(x_i;\widehat\theta)-f(x_i;\widehat\theta^{(i)})\bigr)^2
\Bigr]
\le
\frac{4\,\mathbb E[\dclass(\widehat G,\lambda)]}{n},
\]
where $\widehat\theta^{(i)}$ denotes the fitted parameter obtained from the sample in which only $Y_i$ is replaced by an independent copy $Y_i'$.
\end{theorem}

\begin{lemma}
\label{lem:var_wn_lin}
For any design input $x\in\{x_1,\ldots,x_n\}$,
\[
\operatorname{Var}(f(x;\widehat\theta))
\le
\frac1n\,
\mathbb E\bigl\|\nabla_\theta f(x;\widehat\theta)\bigr\|_{\widehat M_{\mathrm{lin}}}^2,
\qquad
\widehat M_{\mathrm{lin}}
:=
(\widehat G+\lambda I)^{-1}\widehat G(\widehat G+\lambda I)^{-1}.
\]
\end{lemma}

\begin{proof}
For a fixed design input \(x\), let
\(g_x(Z):=f(x;A_\lambda(Y(Z)))\), where \(Y_i(Z)=f^\star(x_i)+Z_i\).
By the selection assumption, $g_x\in W^{1,2}$ under the Gaussian noise law. Hence Gaussian Poincar\'e gives
\[
\operatorname{Var}(g_x(Z))
\le
\mathbb E\|\nabla_Z g_x(Z)\|_2^2.
\]
The weak derivative is supplied by the selection assumption. On a smooth local branch it is the square loss pointwise influence identity
\cite[Lemma~1 and Section~4.1]{kuzborskij2025pointwiseconfidenceestimationnonlinear},
\[
\partial_{Z_k}g_x(Z)
=
\frac1n
\nabla_\theta f(x;\widehat\theta)^\top
(\widehat G+\lambda I)^{-1}g_k .
\]
Therefore
\[
\|\nabla_Z g_x(Z)\|_2^2
=
\frac1n
\bigl\|\nabla_\theta f(x;\widehat\theta)\bigr\|_{\widehat M_{\mathrm{lin}}}^2,
\]
which proves the claim.
\end{proof}

\begin{proof}[Proof of \Cref{thm:stability_lin}]
For every training index $i$ and every design input $x$, the random variables
\[
X:=f(x;\widehat\theta),
\qquad
X':=f(x;\widehat\theta^{(i)})
\]
have the same distribution because replacing $Y_i$ by an independent copy preserves the distribution of the sample. Hence, by $(a-b)^2\le2a^2+2b^2$,
\[
\mathbb E\bigl[(X-X')^2\bigr]
=
\mathbb E\bigl[(X-\mathbb E X+\mathbb E X'-X')^2\bigr]
\le
4\,\operatorname{Var}(f(x;\widehat\theta)).
\]
Applying \Cref{lem:var_wn_lin} at $x=x_i$ and averaging over $i$ yields
\[
\frac1n\sum_{i=1}^n
\mathbb E\bigl[(f(x_i;\widehat\theta)-f(x_i;\widehat\theta^{(i)}))^2\bigr]
\le
\frac{4}{n}
\mathbb E\left[
\frac1n\sum_{i=1}^n \bigl\|\nabla_\theta f(x_i;\widehat\theta)\bigr\|_{\widehat M_{\mathrm{lin}}}^2
\right].
\]
Now
\[
\frac1n\sum_{i=1}^n \bigl\|\nabla_\theta f(x_i;\widehat\theta)\bigr\|_{\widehat M_{\mathrm{lin}}}^2
=
\operatorname{tr}(\widehat M_{\mathrm{lin}}\widehat G)
=
\operatorname{tr}\!\left(((\widehat G+\lambda I)^{-1}\widehat G)^2\right)
=
\dclass(\widehat G,\lambda).
\]
Substituting this identity into the previous estimate proves the theorem.
\end{proof}

Introduce the fixed design comparison risk
\[
L(\theta)
:=
\frac1n\sum_{i=1}^n
\mathbb E\!\left[\frac12\bigl(f(x_i;\theta)-Y_i'\bigr)^2\right],
\]
where $Y_i'$ is an independent copy of $Y_i$ and $\ell(a,y)=\frac12(a-y)^2$. If $S=(x_i,Y_i)_{i=1}^n$ and $S^{(i)}$ denotes the sample in which only $Y_i$ is replaced by $Y_i'$, then the standard replace one identity for expected generalization gap \cite[Chapter~13]{shalev2014understanding} gives
\[
\mathbb E\!\left[L(\widehat\theta)-\widehat L(\widehat\theta)\right]
=
\frac1n\sum_{i=1}^n
\mathbb E\!\left[
\ell(f(x_i;\widehat\theta^{(i)}),Y_i)-\ell(f(x_i;\widehat\theta),Y_i)
\right].
\]

\begin{corollary}[Square loss generalization bound for $\dclass$]
\label{cor:gen_gap_lin}
For square loss, under the assumptions of \Cref{thm:stability_lin},
\[
\mathbb E\!\left[L(\widehat\theta)-\widehat L(\widehat\theta)\right]
\le
\sqrt{\frac{8\,\mathbb E[\widehat L(\widehat\theta)]\,\mathbb E[\dclass(\widehat G,\lambda)]}{n}}
+
\frac{2\,\mathbb E[\dclass(\widehat G,\lambda)]}{n}.
\]
Consequently, for every $\eta>0$,
\[
\mathbb E\!\left[L(\widehat\theta)-\widehat L(\widehat\theta)\right]
\le
\eta\,\mathbb E[\widehat L(\widehat\theta)]
+
\left(2+\frac{2}{\eta}\right)\frac{\mathbb E[\dclass(\widehat G,\lambda)]}{n}.
\]
\end{corollary}

\begin{proof}
For each $i$, write
\[
a_i:=f(x_i;\widehat\theta^{(i)}),
\qquad
b_i:=f(x_i;\widehat\theta).
\]
By the replace one identity,
\[
\mathbb E\!\left[L(\widehat\theta)-\widehat L(\widehat\theta)\right]
=
\frac1n\sum_{i=1}^n
\mathbb E\!\left[\ell(a_i,Y_i)-\ell(b_i,Y_i)\right].
\]
Using the exact quadratic identity for square loss,
\[
\ell(a_i,Y_i)-\ell(b_i,Y_i)
=
(b_i-Y_i)(a_i-b_i)+\frac12(a_i-b_i)^2.
\]
Therefore,
\[
\mathbb E\!\left[\ell(a_i,Y_i)-\ell(b_i,Y_i)\right]
\le
\sqrt{\mathbb E[(b_i-Y_i)^2]\,\mathbb E[(a_i-b_i)^2]}
+
\frac12\mathbb E[(a_i-b_i)^2].
\]
Since
\[
\frac1n\sum_{i=1}^n \frac12\mathbb E[(b_i-Y_i)^2]
=
\mathbb E[\widehat L(\widehat\theta)],
\]
Cauchy--Schwarz over the sample index gives
\[
\frac1n\sum_{i=1}^n
\sqrt{\mathbb E[(b_i-Y_i)^2]\,\mathbb E[(a_i-b_i)^2]}
\le
\sqrt{
2\,\mathbb E[\widehat L(\widehat\theta)]
\cdot
\frac1n\sum_{i=1}^n \mathbb E[(a_i-b_i)^2]
}.
\]
Now apply \Cref{thm:stability_lin} to obtain
\[
\frac1n\sum_{i=1}^n \mathbb E[(a_i-b_i)^2]
\le
\frac{4\,\mathbb E[\dclass(\widehat G,\lambda)]}{n}.
\]
Substituting this estimate proves the displayed $\sqrt{\cdot}+\cdot$ bound. The $\eta$ form follows from Young's inequality.
\end{proof}

\subsection{Effective Dimension $\deff(\widehat\theta;\lambda)$}

Return to the Gaussian noise assumption $Z\sim\cN(0,I_n)$, and now allow $\widehat\Delta$ to be nonzero. For a nonlinear predictor the residual curvature correction enters the local curvature, so the influence identity contains $\widehat H_\lambda^{-1}$.

\begin{theorem}[Prediction stability via effective dimension]
\label{thm:stability}
Under Gaussian noise and the selection assumption,
\[
\frac1n\sum_{i=1}^n
\mathbb E\Bigl[
\bigl(f(x_i;\widehat\theta)-f(x_i;\widehat\theta^{(i)})\bigr)^2
\Bigr]
\le
\frac{4\,\mathbb E[\deff(\widehat\theta;\lambda)]}{n}.
\]
\end{theorem}

\begin{proof}
The proof is identical to the proof of \Cref{thm:stability_lin} until the trace calculation. At a design input $x$, the variance lemma becomes
\[
\operatorname{Var}(f(x;\widehat\theta))
\le
\frac1n\,
\mathbb E\bigl\|\nabla_\theta f(x;\widehat\theta)\bigr\|_{\widehat M_{\mathrm{eff}}}^2,
\qquad
\widehat M_{\mathrm{eff}}
:=
\widehat H_{\lambda}^{-1}\widehat G\widehat H_{\lambda}^{-1}.
\]
This follows by the same Gaussian Poincar\'e argument as before, using the
response derivative supplied by the selection assumption, equivalently
the pointwise influence identity of
\cite[Lemma~1 and Section~4.1]{kuzborskij2025pointwiseconfidenceestimationnonlinear}
on a single nondegenerate branch,
\[
\|\nabla_Z f(x;\widehat\theta(Z))\|_2^2
=
\frac1n\bigl\|\nabla_\theta f(x;\widehat\theta)\bigr\|_{\widehat M_{\mathrm{eff}}}^2.
\]
Averaging over the training inputs gives
\[
\frac1n\sum_{i=1}^n \bigl\|g_i\bigr\|_{\widehat M_{\mathrm{eff}}}^2
=
\operatorname{tr}(\widehat M_{\mathrm{eff}}\widehat G)
=
\operatorname{tr}\!\left((\widehat H_{\lambda}^{-1}\widehat G)^2\right)
=
\deff(\widehat\theta;\lambda).
\]
Substituting this trace identity into the replace one variance argument proves the theorem.
\end{proof}

\begin{corollary}[Square loss generalization bound]
\label{cor:gen_gap}
For square loss, under the assumptions of Theorem~\ref{thm:stability},
\[
\mathbb E\!\left[L(\widehat\theta)-\widehat L(\widehat\theta)\right]
\le
\sqrt{\frac{8\,\mathbb E[\widehat L(\widehat\theta)]\,\mathbb E[\deff(\widehat\theta;\lambda)]}{n}}
+
\frac{2\,\mathbb E[\deff(\widehat\theta;\lambda)]}{n}.
\]
Consequently, for every $\eta>0$,
\[
\mathbb E\!\left[L(\widehat\theta)-\widehat L(\widehat\theta)\right]
\le
\eta\,\mathbb E[\widehat L(\widehat\theta)]
+
\left(2+\frac{2}{\eta}\right)\frac{\mathbb E[\deff(\widehat\theta;\lambda)]}{n}.
\]
\end{corollary}

\begin{proof}
Repeat the proof of \Cref{cor:gen_gap_lin}, replacing \Cref{thm:stability_lin} by \Cref{thm:stability}. Other steps remain the same.
\end{proof}

\subsection{Log concave noise}

For independent log concave noise coordinates, the Brascamp--Lieb variance inequality \cite{brascamp1976,liebmodern} controls response sensitivity. The following result combines this inequality with the replace one identity.

\begin{corollary}[Brascamp--Lieb bound for log concave noise]
\label{cor:log_concave_brascamp_lieb}
Impose the selection assumption and let $Y_i=f^\star(x_i)+Z_i$, where $Z_1,\ldots,Z_n$ are independent and each $Z_i$ has density proportional to $\exp(-V_i(z))$. Assume $V_i$ is twice differentiable with $V_i''(z)>0$ in the interior of its convex support, and the Brascamp--Lieb inequality applies with finite right hand side. Define
\[
w_i:=\frac{1}{V_i''(Z_i)},
\qquad
\widehat G_{\mathrm{BL}}
:=
\frac1n\sum_{i=1}^n w_i g_i g_i^\top.
\]
For any positive definite matrix $B_\lambda$ equal to either $\widehat G+\lambda I$ in the linear case $\widehat\Delta=0$, or $\widehat H_\lambda$ in the nonlinear case, define
\[
\mathfrak d_{\mathrm{BL}}(B_\lambda)
:=
\tr\!\left(B_\lambda^{-1}\widehat G_{\mathrm{BL}}B_\lambda^{-1}\widehat G\right).
\]
Then the replace one stability bound becomes
\[
\frac1n\sum_{i=1}^n
\mathbb E\Bigl[
\bigl(f(x_i;\widehat\theta)-f(x_i;\widehat\theta^{(i)})\bigr)^2
\Bigr]
\le
\frac{4\,\mathbb E[\mathfrak d_{\mathrm{BL}}(B_\lambda)]}{n}.
\]
Consequently, for square loss,
\[
\mathbb E\!\left[L(\widehat\theta)-\widehat L(\widehat\theta)\right]
\le
\sqrt{\frac{8\,\mathbb E[\widehat L(\widehat\theta)]\,\mathbb E[\mathfrak d_{\mathrm{BL}}(B_\lambda)]}{n}}
+
\frac{2\,\mathbb E[\mathfrak d_{\mathrm{BL}}(B_\lambda)]}{n}.
\]
In particular, if $V_i''(z)\ge\alpha>0$ for all $i$ and all $z$ in the support, then
\[
\begin{aligned}
 \widehat G_{\mathrm{BL}}&\preceq\alpha^{-1}\widehat G,\\
 \mathfrak d_{\mathrm{BL}}(\widehat G+\lambda I)
 &\le\alpha^{-1}\dclass(\widehat G,\lambda)\qquad(\widehat\Delta=0),\\
 \mathfrak d_{\mathrm{BL}}(\widehat H_\lambda)
 &\le\alpha^{-1}\deff(\widehat\theta;\lambda).
\end{aligned}
\]
Thus the same generalization bounds hold with an additional factor $\alpha^{-1}$ inside the effective dimension term.
\end{corollary}

\begin{proof}
For a fixed design input \(x\), let \(h(Z):=f(x;A_\lambda(Y(Z)))\). By the selection
assumption, \(h\in W^{1,2}\) under the product noise law, so the
product Brascamp--Lieb inequality \cite{brascamp1976} gives
\[
\operatorname{Var}(h(Z))
\le
\mathbb E\sum_{j=1}^n \frac{(\partial_{Z_j}h(Z))^2}{V_j''(Z_j)}.
\]
The derivative from the selection assumption gives, and on a
single nondegenerate branch agrees with the pointwise influence identity of
\cite[Lemma~1 and Section~4.1]{kuzborskij2025pointwiseconfidenceestimationnonlinear}, with $B_\lambda=\widehat G+\lambda I$ in the linear case and $B_\lambda=\widehat H_\lambda$ in the nonlinear case,
\[
\partial_{Z_j}h(Z)
=
\frac1n\,\nabla_\theta f(x;\widehat\theta)^\top B_\lambda^{-1}g_j.
\]
Therefore
\[
\operatorname{Var}(f(x;\widehat\theta))
\le
\frac1n\,\mathbb E\left[
\nabla_\theta f(x;\widehat\theta)^\top
B_\lambda^{-1}\widehat G_{\mathrm{BL}}B_\lambda^{-1}
\nabla_\theta f(x;\widehat\theta)
\right].
\]
The replace one comparison $\mathbb E[(X-X')^2]\le4\operatorname{Var}(X)$ remains the same. Averaging over $x_1,\ldots,x_n$ gives $4\mathbb E[\mathfrak d_{\mathrm{BL}}(B_\lambda)]/n$, which is the stability bound. The square loss generalization bound follows from the same proof as \Cref{cor:gen_gap_lin}.

If $V_i''\ge\alpha$, then $w_i\le\alpha^{-1}$ and $\widehat G_{\mathrm{BL}}\preceq\alpha^{-1}\widehat G$. Congruence by $B_\lambda^{-1}\widehat G^{1/2}$ and taking traces yield
$\mathfrak d_{\mathrm{BL}}(B_\lambda)\le\alpha^{-1}\operatorname{tr}((B_\lambda^{-1}\widehat G)^2)$, proving the stated simplifications. For standard Gaussian noise, $w_i=1$ and $\widehat G_{\mathrm{BL}}=\widehat G$, recovering the Gaussian results.
\end{proof}

\clearpage
\section{Geometric Proofs}
\label{app:section3_proofs}

\paragraph{Roadmap.} The prediction space comparison of
Proposition~\ref{prop:gamma_pred} separates curvature from the geometry
of the trained Jacobian features. We first study its order monotonicity, then prove the subspace and covering bounds. Manifold covers
transfer through a piecewise Lipschitz Jacobian, giving
Proposition~\ref{prop:manifold_main}. Finally, we instantiate the geometry for
one hidden layer ReLU networks and explain the role of homogeneity.
Throughout, the training regularization satisfies \(\lambda>0\). Every nonlinear comparison below assumes
\(\widehat H_\lambda\succ0\). The corresponding pseudoinverse version
requires the additional conditions of Appendix~\ref{app:flat_symmetry}.

\begin{figure}[H]
\centering
\resizebox{\linewidth}{!}{
\begingroup
\definecolor{geomblue}{HTML}{34699A}
\definecolor{geomteal}{HTML}{24867A}
\definecolor{geomorange}{HTML}{D18A38}
\definecolor{geomink}{HTML}{263442}
\definecolor{geomgray}{HTML}{DCE3E9}
\begin{tikzpicture}[x=1cm,y=1cm,font=\small,text=geomink,
  flow/.style={-{Latex[length=2.2mm]},line width=.8pt,geomink},
  head/.style={anchor=west,font=\small\bfseries},
  note/.style={font=\footnotesize,align=center}]
\path[use as bounding box] (0,0) rectangle (17,6.4);
\foreach \a/\b in {0/5.1,5.95/11.05,11.9/17}{
  \fill[geomgray!18,rounded corners=3pt] (\a,1.5) rectangle (\b,6.35);
}
\node[head] at (.25,5.99) {(a) Input geometry};
\node[head] at (6.2,5.99) {(b) Trained Jacobian features};
\node[head] at (12.15,5.99) {(c) Spectral compression};
\node[note] at (2.55,5.4) {$m$-dimensional support in $\mathbb R^d$};
\node[note] at (8.5,5.4) {$g_i=\nabla_\theta f(x_i;\widehat\theta)\in\mathbb R^p$};
\node[note] at (14.45,5.4) {A $K$-dimensional approximation};
\path[fill=geomblue!25,draw=geomblue,line width=.6pt]
 (.45,3.65) .. controls (.85,4.7) and (1.32,4.95) .. (1.86,4.35)
 -- (1.86,3.28) .. controls (1.32,3.85) and (.85,3.65) .. (.45,2.7) -- cycle;
\path[fill=geomteal!24,draw=geomteal,line width=.6pt]
 (1.86,4.35) .. controls (2.38,3.83) and (2.83,3.26) .. (3.3,3.57)
 -- (3.3,2.5) .. controls (2.83,2.19) and (2.38,2.76) .. (1.86,3.28) -- cycle;
\path[fill=geomorange!25,draw=geomorange,line width=.6pt]
 (3.3,3.57) .. controls (3.8,3.9) and (4.25,4.5) .. (4.68,4.08)
 -- (4.68,3.02) .. controls (4.25,3.44) and (3.8,2.84) .. (3.3,2.5) -- cycle;
\foreach \x/\y in {.72/3.43,.99/3.87,1.19/4.12,1.37/3.71,1.66/3.9,.98/3.51,1.62/4.17}
 \fill[geomblue] (\x,\y) circle (1.5pt);
\foreach \x/\y in {2.04/3.87,2.25/3.28,2.44/3.58,2.66/2.99,2.96/3.1,3.12/3.39,2.7/3.42}
 \fill[geomteal] (\x,\y) circle (1.5pt);
\foreach \x/\y in {3.52/3.2,3.76/3.54,3.94/3.76,4.17/3.5,4.39/3.92,4.48/3.37,3.64/3.15}
 \fill[geomorange] (\x,\y) circle (1.5pt);
\draw[geomblue,dashed] (1.28,3.95) ellipse (.43 and .48);
\node[note,geomblue] at (2.55,2.25) {Piece $U_r$, local constant $L_r$};
\node[note] at (2.55,1.79) {Covers on each occupied piece};
\foreach \x/\y/\c in {7.1/4.0/geomblue,9.73/4.05/geomteal,8.45/2.65/geomorange}{
 \fill[\c!8] (\x,\y) circle (.69);
 \draw[\c,dashed,line width=.7pt] (\x,\y) circle (.69);
 \fill[\c] (\x,\y) circle (1.8pt);
 \foreach \dx/\dy in {-.3/-.12,-.2/.24,.13/.34,.34/.01,.13/-.28,-.13/-.35,.32/.25}
  \fill[\c] (\x+\dx,\y+\dy) circle (1.5pt);
}
\draw[geomblue,line width=.65pt] (7.1,4) -- (7.73,4.28);
\node[font=\footnotesize,geomblue,fill=geomgray!18,inner sep=1pt] at (7.59,4.42) {$\varepsilon$};
\node[note] at (8.5,1.79) {$K=\mathcal C_J(\varepsilon)$ cover centers};
\draw[flow] (5.22,3.63) -- (5.82,3.63);
\node[note] at (5.53,4.07) {$g$};
\draw[flow] (11.17,3.63) -- (11.78,3.63);
\node[note] at (11.48,4.07) {$\widehat G$};
\draw[geomink!60,line width=.5pt] (12.45,2.3) -- (16.55,2.3);
\draw[geomink!60,line width=.5pt] (12.45,2.3) -- (12.45,4.82);
\node[font=\footnotesize,anchor=south] at (12.5,4.85) {$\mu_j(\widehat G)$};
\foreach \j/\h/\c in {0/2.03/geomblue,1/1.78/geomteal,2/1.48/geomorange,3/.54/geomgray,4/.45/geomgray,5/.35/geomgray,6/.29/geomgray,7/.24/geomgray,8/.20/geomgray,9/.16/geomgray,10/.12/geomgray,11/.10/geomgray}
 \fill[\c] (12.63+\j*.31,2.3) rectangle +(0.22,\h);
\draw[geomink!55,dashed] (13.51,2.23) -- (13.51,4.64);
\node[font=\footnotesize] at (13.03,2.04) {$j\le K$};
\node[font=\footnotesize] at (15.13,2.04) {$j>K$};
\node[note] at (15.15,4.21) {$\displaystyle\sum_{j>K}\mu_j\le\varepsilon^2$};
\node[note] at (14.45,1.67) {Low rank part + small residual};
\draw[geomgray] (0,.99) -- (17,.99);
\node[align=center,font=\normalsize] at (8.5,.53) {
 $\displaystyle d_{\mathrm{eff}}(\widehat\theta;\lambda)
 \;\le\;\gamma_{\mathrm{pred}}(\lambda)^{-2}\,d_{\mathrm{lin}}(\widehat G,\lambda)
 \;\le\;\gamma_{\mathrm{pred}}(\lambda)^{-2}\!\left(K+\frac{\varepsilon^2}{\lambda}\right),\qquad\lambda>0$};
\end{tikzpicture}
\endgroup}
\caption{\textbf{From input geometry to generalization.} Covers of the occupied
input pieces transfer through the trained Jacobian to feature covers.
These control the classical effective dimension at the training ridge level \(\lambda\).
The factor \(\gamma_{\rm pred}(\lambda)^{-2}\) transfers the bound to nonlinear
prediction sensitivity. Reproduced from Figure~\ref{fig:section3overview}.}
\label{fig:section3overview_appendix}
\end{figure}

\subsection{Comparing inverse metrics}

\begin{lemma}[Monotonicity in the inverse metric]
\label{lem:metric_mono_section3}
Let $A\succeq 0$ and let $0\prec B_1\preceq B_2$. Then
\[
\tr\!\bigl(A^{1/2}B_2^{-1}A B_2^{-1}A^{1/2}\bigr)
\le
\tr\!\bigl(A^{1/2}B_1^{-1}A B_1^{-1}A^{1/2}\bigr).
\]
Equivalently,
\[
\tr\!\bigl((B_2^{-1/2}AB_2^{-1/2})^2\bigr)
\le
\tr\!\bigl((B_1^{-1/2}AB_1^{-1/2})^2\bigr).
\]
\end{lemma}

\begin{proof}
Because $B_1\preceq B_2$, inversion reverses the order and gives $B_2^{-1}\preceq B_1^{-1}$. Conjugating by $A^{1/2}$ yields
\[
0\preceq A^{1/2}B_2^{-1}A^{1/2}\preceq A^{1/2}B_1^{-1}A^{1/2}.
\]
For positive semidefinite matrices, the map $X\mapsto \tr(X^2)$ is monotone under the Loewner order, so the claim follows.
\end{proof}

The proof of the prediction space comparison, including the Schur complement
identity, appears in Proposition~\ref{prop:gamma_pred}. Combining that
comparison with the covering argument below gives
\begin{equation}
\label{eq:raw_cover_bound_section3}
\deff(\widehat\theta;\lambda)
\le \gamma_{\rm pred}(\lambda)^{-2}
\inf_{\varepsilon>0}
\left[\mathcal C_J(\varepsilon)+\frac{\varepsilon^2}{\lambda}\right].
\end{equation}

\subsection{Covering bounds for Jacobian features}
An \(\varepsilon\)-cover provides us with a low dimensional subspace within
\(\varepsilon\) of every feature vector. A variational eigenvalue identity
then converts this approximation into a tail eigenvalue bound.

\begin{lemma}[Subspace approximation bound]
\label{lem:subspace_approx_section3}
Let $A\succeq 0$ be a positive semidefinite matrix on $\R^p$, and suppose there exists a subspace $W\subset\R^p$ with $\dim W\le K$ such that
\[
    E
    \;:=\;
    \tr\!\bigl((I-P_W)\,A\,(I-P_W)\bigr)
\]
is finite, where $P_W$ is the orthogonal projection onto $W$. Then for every $\lambda>0$,
\[
\dclass(A,\lambda)\;\le\;K+\frac{E}{\lambda}.
\]
\end{lemma}

\begin{proof}
If \(K\ge p\), the claim follows from \(\dclass(A,\lambda)\le p\).
Otherwise extend \(W\) to dimension \(K\), which can only decrease
the approximation error \(E\).
By the Ky Fan max principle \cite{kyfan}, for any $K$-dimensional subspace $W$,
\[
    \sum_{j=1}^K\mu_j(A)
    \;\ge\;
    \tr\!\bigl(P_W A P_W\bigr).
\]
Since $A\succeq 0$ and $\tr(A)=\tr(P_W A P_W)+\tr((I-P_W)A(I-P_W))$,
\[
    \sum_{j>K}\mu_j(A)
    \;=\;
    \tr(A)-\sum_{j=1}^K\mu_j(A)
    \;\le\;
    \tr(A)-\tr(P_W A P_W)
    \;=\;
    E.
\]
Using the pointwise bound $(\mu/(\mu+\lambda))^2\le \min\{1,\mu/\lambda\}$ for $\mu,\lambda>0$,
\begin{align*}
\dclass(A,\lambda)
&=
\sum_{j=1}^p\!\left(\frac{\mu_j(A)}{\mu_j(A)+\lambda}\right)^2
\le
\sum_{j=1}^K 1+\frac{1}{\lambda}\sum_{j>K}\mu_j(A) \\
&\le
K+\frac{E}{\lambda}. \qedhere
\end{align*}
\end{proof}

\begin{theorem}[Covering number bound for Jacobian features]
\label{thm:covering_J_section3}
For every \(\lambda>0\),
\[
\dclass(\widehat G,\lambda)
\le \inf_{\varepsilon>0}
\left[\mathcal C_J(\varepsilon)+\frac{\varepsilon^2}{\lambda}\right].
\]
Consequently, if \(\widehat H_\lambda\succ0\), the right hand side multiplied
by \(\gamma_{\rm pred}(\lambda)^{-2}\) bounds \(\deff(\widehat\theta;\lambda)\).
\end{theorem}
\begin{proof}
Fix \(\varepsilon>0\), let \(K=\mathcal C_J(\varepsilon)\), and choose
cover centers \(u_1,\ldots,u_K\). The subspace
\(W=\operatorname{span}\{u_1,\ldots,u_K\}\) has dimension at most \(K\).
For every \(i\), a center \(u_{c(i)}\in W\) satisfies
\(\|g_i-u_{c(i)}\|_2\le\varepsilon\), so
\[
\|(I-P_W)g_i\|_2\le\varepsilon.
\]
Thus
\[
\operatorname{tr}((I-P_W)\widehat G(I-P_W))
=\frac1n\sum_i\|(I-P_W)g_i\|_2^2\le\varepsilon^2.
\]
Apply Lemma~\ref{lem:subspace_approx_section3}, then take the infimum over
\(\varepsilon\). Proposition~\ref{prop:gamma_pred} gives the
nonlinear comparison.
\end{proof}

\begin{remark}[Finite sample saturation and the choice of radius]
\label{rem:finite_sample_saturation_section3}
If the empirical features are distinct, let
\(\Delta_J=\min_{i\ne j}\|g_i-g_j\|_2\), with the minimum of an empty set taken as \(+\infty\). With arbitrary cover centers,
\(\mathcal C_J(\varepsilon)=n\) whenever \(2\varepsilon<\Delta_J\).
If there are repeated features, the analogous statement uses the number
and minimum separation of the distinct features. Useful radii satisfy
\[
K_\varepsilon:=\mathcal C_J(\varepsilon)\ll n,
\qquad \varepsilon^2/\lambda\lesssim K_\varepsilon.
\]
Then \(\deff\lesssim\gamma_{\rm pred}(\lambda)^{-2}K_\varepsilon\).
Both feature compression and the curvature factor are important for the final
bound and the manifold estimates below give sufficient conditions for such
compression.
\end{remark}

\subsection{Manifold covers and piecewise regular feature maps}

\begin{proposition}[Compact manifold $\Rightarrow$ polynomial covering growth]
\label{prop:manifold_cover_section3}
Let \(\mathcal M\subset\mathbb R^d\) be a nonempty compact \(C^1\)
embedded manifold of dimension \(m\ge1\), and let
\(D_{\mathcal M}=\operatorname{diam}(\mathcal M)\). There exists
\(C_{\mathcal M}<\infty\) such that every nonempty subset
\(U\subseteq\mathcal M\) admits a \(\delta\)-cover with centers in \(U\)
and cardinality at most
\[
\max\{1,C_{\mathcal M}\delta^{-m}\},\qquad \delta>0.
\]
For \(0<\delta\le D_{\mathcal M}\), the cardinality is at most
\(C_{\mathcal M}\delta^{-m}\).
\end{proposition}
\begin{proof}
By compactness, there are finitely many coordinate maps \(\phi_a\) and
closed coordinate balls \(K_a\subset\mathbb R^m\), of radii \(R_a\),
whose images cover \(\mathcal M\). Choose each ball inside the domain of
its chart. Boundedness of the derivative on a slightly larger ball gives
a finite Lipschitz constant \(L_a\) for \(\phi_a\) on \(K_a\).
A maximal \(\eta\)-separated subset of \(K_a\) is an \(\eta\)-cover
with centers in \(K_a\). The disjoint balls of radius \(\eta/2\)
around these centers give the volumetric bound
\((1+2R_a/\eta)^m\). Taking \(\eta=\delta/L_a\) and mapping the
centers through \(\phi_a\) therefore yields
\[
\mathcal N(\mathcal M,\|\cdot\|_2,\delta)
\le \sum_a(1+2L_aR_a/\delta)^m
\le C_0\delta^{-m},\qquad 0<\delta\le D_{\mathcal M},
\]
where \(C_0=\sum_a(D_{\mathcal M}+2L_aR_a)^m\).

Now fix a nonempty \(U\subseteq\mathcal M\) and
\(0<\delta\le D_{\mathcal M}\). Cover \(\mathcal M\) at radius
\(\delta/2\), discard balls not meeting \(U\), and choose a point
of \(U\) in each remaining ball. Every point of \(U\) in such a
ball is within distance \(\delta\) of the chosen point.
Thus \(U\) has a cover with centers in \(U\), of cardinality at
most \(2^m C_0\delta^{-m}\). For \(\delta>D_{\mathcal M}\), any
single point of \(U\) is a cover center. Set
\(C_{\mathcal M}=2^m C_0\) then this also gives
\(C_{\mathcal M}\ge D_{\mathcal M}^m\), completing the claim.
\end{proof}

\begin{remark}[The constant depends on the manifold geometry]
\label{rem:manifold_constant_section3}
The constant \(C_{\mathcal M}\) can depend on the particular manifold, including
its diameter, the number of charts, their coordinate radii, and their
Lipschitz constants.
The subset cover formulation ensures that every center used in the next
step lies in the domain where the relevant feature map is regular.
\end{remark}

\begin{remark}[Zero dimensional data]
A compact zero dimensional manifold is a finite set. If \(N_X\) is
the number of distinct sampled inputs, then the Jacobian covariance
has rank at most \(\min\{p,N_X\}\), and hence
\(\deff\le\gamma_{\rm pred}(\lambda)^{-2}\min\{p,N_X\}\).
A Lipschitz assumption is not needed for this finite set bound. The
polynomial optimizations below concern \(m\ge1\).
\end{remark}

\begin{corollary}[Transferring manifold covers to feature space]
\label{cor:feature_cover_transfer_section3}
Let \(\Psi:\mathcal M\to\mathcal H\) be \(L_\Psi\)-Lipschitz into a
Hilbert space, with \(L_\Psi>0\). For
\(0<\varepsilon\le L_\Psi D_{\mathcal M}\),
\[
\mathcal N(\{\Psi(x_i)\}_{i=1}^n,\|\cdot\|_{\mathcal H},\varepsilon)
\le\mathcal N(\Psi(\mathcal M),\|\cdot\|_{\mathcal H},\varepsilon)
\le C_{\mathcal M}(L_\Psi/\varepsilon)^m.
\]
\end{corollary}

\begin{proof}
By Proposition~\ref{prop:manifold_cover_section3}, choose a cover of
\(\mathcal M\) at radius \(\varepsilon/L_\Psi\), with centers
\(z_j\in\mathcal M\), of cardinality at most
\(C_{\mathcal M}(L_\Psi/\varepsilon)^m\). For each \(x\in\mathcal M\),
choose a center satisfying \(\|x-z_j\|_2\le\varepsilon/L_\Psi\).
Then \(\|\Psi(x)-\Psi(z_j)\|_{\mathcal H}\le\varepsilon\).
The image centers therefore cover \(\Psi(\mathcal M)\), and hence
the empirical feature set. All centers lie in the domain of \(\Psi\).
\end{proof}

\begin{corollary}[Feature space and Jacobian bounds under a manifold hypothesis]
\label{cor:manifold_effdim_section3}
Under Corollary~\ref{cor:feature_cover_transfer_section3}, write
\(\widehat C_\Psi=n^{-1}\sum_i\Psi(x_i)\otimes\Psi(x_i)\). For
\(\lambda>0\) and \(0<\varepsilon\le L_\Psi D_{\mathcal M}\),
\[
\dclass(\widehat C_\Psi,\lambda)
\le C_{\mathcal M}(L_\Psi/\varepsilon)^m+\varepsilon^2/\lambda.
\]
For \(\Psi=g\), the corresponding bound for \(\deff\) is multiplied by
\(\gamma_{\rm pred}(\lambda)^{-2}\).
\end{corollary}
\begin{proof}
Apply the subspace argument of Theorem~\ref{thm:covering_J_section3}
to the finite rank covariance \(\widehat C_\Psi\), then use
Corollary~\ref{cor:feature_cover_transfer_section3}. The nonlinear conclusion
uses Proposition~\ref{prop:gamma_pred}.
\end{proof}

Global Lipschitz continuity can be too strong for ReLU Jacobians. Regularity
on pieces covering the sample is enough. The pieces need not be connected
or have regular boundaries.

\begin{theorem}[Piecewise regular feature maps on a manifold]
\label{thm:piecewise_regular_section3}
Let \(x_1,\ldots,x_n\in\mathcal M\), where \(\mathcal M\) satisfies
Proposition~\ref{prop:manifold_cover_section3}. Suppose
\(\{x_1,\ldots,x_n\}\subseteq\bigcup_{r=1}^M U_r\), with
nonempty \(U_r\subseteq\mathcal M\), and let \(\Psi\) be a feature
map into a Hilbert space that is \(L_r\)-Lipschitz on \(U_r\),
where \(L_r\ge0\). For every \(\varepsilon>0\),
\[
\mathcal N(\{\Psi(x_i)\}_{i=1}^n,\|\cdot\|_{\mathcal H},\varepsilon)
\le \sum_{r=1}^M
\max\!\left\{1,C_{\mathcal M}(L_r/\varepsilon)^m\right\}
\le M+A\varepsilon^{-m},
\]
where \(A=C_{\mathcal M}\sum_{r=1}^M L_r^m\). Writing
\(\widehat C_\Psi=n^{-1}\sum_i\Psi(x_i)\otimes\Psi(x_i)\),
we have, for every \(\lambda>0\),
\[
\dclass(\widehat C_\Psi,\lambda)
\le M+C_m A^{2/(m+2)}\lambda^{-m/(m+2)}.
\]
If all \(L_r>0\), the sharper cover bound
\[
\mathcal N(\{\Psi(x_i)\}_{i=1}^n,\|\cdot\|_{\mathcal H},\varepsilon)
\le A\varepsilon^{-m}
\]
holds for \(0<\varepsilon\le D_{\mathcal M}\min_rL_r\). Thus the
additive \(M\) can be omitted from the preceding effective dimension
bound when \((mA\lambda/2)^{1/(m+2)}\) lies in this radius range.

For \(\Psi=g\), the nonlinear bounds follow by multiplying the
corresponding classical bounds by \(\gamma_{\rm pred}(\lambda)^{-2}\).
\end{theorem}
\begin{proof}
If \(L_r>0\), cover \(U_r\) at radius \(\varepsilon/L_r\), using
centers in \(U_r\) supplied by
Proposition~\ref{prop:manifold_cover_section3}. Their images form an
\(\varepsilon\)-cover of \(\Psi(U_r)\) by the Lipschitz property.
If \(L_r=0\), the map is constant on \(U_r\), so one image center
suffices. The union of these covers contains every empirical feature.
Summing their cardinalities and using \(\max\{1,z\}\le1+z\)
proves the first expression.

The subspace argument of Theorem~\ref{thm:covering_J_section3} gives
\[
\dclass(\widehat C_\Psi,\lambda)
\le M+A\varepsilon^{-m}+\varepsilon^2/\lambda.
\]
For \(A>0\), minimizing at
\(\varepsilon_\star=(mA\lambda/2)^{1/(m+2)}\) gives the stated constant
\(C_m=(1+m/2)(2/m)^{m/(m+2)}\). If \(A=0\), let
\(\varepsilon\downarrow0\). In the bounded radius regime all
\(\varepsilon/L_r\le D_{\mathcal M}\), so the subset cover
proposition gives the sharper polynomial bound with no additive
\(M\). Finally apply the prediction space comparison to \(\Psi=g\).
\end{proof}

For reference, the all radius feature cover estimate is
\begin{equation}
\label{eq:robust_manifold_cover_bound_section3}
\mathcal C_J(\varepsilon)\le
\sum_{r=1}^M
\max\{1,C_{\mathcal M}(L_r/\varepsilon)^m\}.
\end{equation}
If \(L_{\min}=\min_rL_r>0\), it sharpens to
\begin{equation}
\label{eq:manifold_cover_bound_section3}
\mathcal C_J(\varepsilon)\le A\varepsilon^{-m},
\qquad0<\varepsilon\le D_{\mathcal M}L_{\min}.
\end{equation}
Together with the curvature comparison, this gives
\begin{equation}
\label{eq:manifold_nonlin_bound_section3}
\deff(\widehat\theta;\lambda)\le
\gamma_{\rm pred}(\lambda)^{-2}
\inf_{0<\varepsilon\le D_{\mathcal M}L_{\min}}
[A\varepsilon^{-m}+\varepsilon^2/\lambda].
\end{equation}
As the most applicable examples, bounded ambient Euclidean data give
the analogous exponent \(d\) and data on a sphere of positive dimension
give \(d-1\).
The manifold hypothesis replaces these by the intrinsic dimension.

\begin{corollary}[Optimized effective dimension manifold bound]
\label{cor:optimized_nonlin_section3}
Under Theorem~\ref{thm:piecewise_regular_section3} with \(\Psi=g\), let
\(A=C_{\mathcal M}\sum_r L_r^m\), assume \(L_{\min}>0\), and set
\(\varepsilon_\star=(mA\lambda/2)^{1/(m+2)}\).
If \(\varepsilon_\star\le D_{\mathcal M}L_{\min}\), then
\begin{equation}
\label{eq:optimized_manifold_nonlin_rate_section3}
\deff(\widehat\theta;\lambda)
\le\gamma_{\rm pred}(\lambda)^{-2}
C_mA^{2/(m+2)}\lambda^{-m/(m+2)}.
\end{equation}
For arbitrary \(L_r\ge0\) and \(\lambda>0\), the same bound with
\(\gamma_{\rm pred}(\lambda)^{-2}M\) added to the right hand side
always holds.
\end{corollary}
\begin{proof}
The stationary point of \(A\varepsilon^{-m}+\varepsilon^2/\lambda\) is
\((mA\lambda/2)^{1/(m+2)}\). Substitution gives
\(C_m=(1+m/2)(2/m)^{m/(m+2)}\). The sharper cover is applicable
only when this radius is in the stated range. The all radius alternative
follows from \(\max\{1,z\}\le1+z\) as in
Theorem~\ref{thm:piecewise_regular_section3}.
\end{proof}

\subsection{One hidden layer ReLU. When is the Jacobian map regular?}
We use the same normalized model as the main text,
\(f(x;\theta)=q^{-1/2}a^\top\sigma(Wx)\), with both layers trained.
Write
\[
\mathcal H_J=\mathbb R^q\times\mathbb R^{q\times d},
\qquad \|(u,U)\|_{\mathcal H_J}^2=\|u\|_2^2+\|U\|_F^2,
\]
and
\(\widehat D(x)=\operatorname{diag}(\mathbf1\{\widehat w_j^\top x>0\})_{j=1}^q\).
Classical derivatives below are used only at strict, nonzero
preactivations.

\begin{proposition}[Exact Jacobian metric on an activation stable set]
\label{prop:relu_local_metric_section3}
Consider \(f(x;\theta)=q^{-1/2}a^\top\sigma(Wx)\). On an
activation stable set \(U\) with gate matrix \(D_U\),
\[
g(x)=q^{-1/2}\bigl(D_U\widehat W x,
(D_U\widehat a)x^\top\bigr),
\]
and, for every \(x,z\in U\),
\[
\|g(x)-g(z)\|_{\mathcal H_J}^2
=(x-z)^\top S_U^{\rm norm}(x-z),\qquad
S_U^{\rm norm}=q^{-1}\widehat W^\top D_U\widehat W
+q^{-1}\|D_U\widehat a\|_2^2 I_d.
\]
Consequently \(L_U=\|S_U^{\rm norm}\|_{\rm op}^{1/2}\) is a
valid Lipschitz constant on \(U\). In particular,
\(L_U^2\le q^{-1}(\|\widehat W\|_{\rm op}^2+\|\widehat a\|_2^2)\).
\end{proposition}
\begin{proof}
Strict gates give the classical derivatives in the first expression.
For \(h=x-z\), the squared product space norm of their difference is
\[
q^{-1}\|D_U\widehat W h\|_2^2
+q^{-1}\|(D_U\widehat a)h^\top\|_F^2
=h^\top S_U^{\rm norm}h.
\]
The norm bound follows from \(0\preceq D_U\preceq I_q\).
\end{proof}

\begin{proposition}[Gate margins guarantee activation stable cells]
\label{prop:relu_margin_section3}
Suppose all rows \(\widehat w_j\) are nonzero. For \(c\in\mathbb R^d\),
define
\[
\gamma(c)=\min_{1\le j\le q}
\frac{|\widehat w_j^\top c|}{\|\widehat w_j\|_2}.
\]
If \(0<\delta<\gamma(c)\), then every preactivation is nonzero
with fixed sign throughout \(B(c,\delta)\). In particular the
normalized Jacobian is Lipschitz there, with
\(L_c^2\le q^{-1}(\|\widehat W\|_{\rm op}^2+\|\widehat a\|_2^2)\).
\end{proposition}
\begin{proof}
For every \(x\in B(c,\delta)\) and every \(j\),
\[
|\widehat w_j^\top(x-c)|
\le\|\widehat w_j\|_2\delta
<|\widehat w_j^\top c|.
\]
The signs are therefore constant and strict. Apply
Proposition~\ref{prop:relu_local_metric_section3}.
\end{proof}

\begin{corollary}[ReLU under a manifold hypothesis and an activation stable cover]
\label{cor:relu_piecewise_section3}
Suppose the training inputs lie on \(\mathcal M\) and have strict
preactivations. Let \(U_1,\ldots,U_M\) be the intersections of
\(\mathcal M\) with the occupied gate cells, and set
\(L_r=\|S_r^{\rm norm}\|_{\rm op}^{1/2}\). Then all feature cover and
effective dimension conclusions of
Theorem~\ref{thm:piecewise_regular_section3} apply. If
\(S_r^{\rm norm}\preceq\eta I_d\) on every occupied piece, then for
every \(\lambda>0\),
\[
\deff(\widehat\theta;\lambda)
\le\gamma_{\rm pred}(\lambda)^{-2}
\left[
M+C_m(C_{\mathcal M}M\eta^{m/2})^{2/(m+2)}
\lambda^{-m/(m+2)}
\right].
\]
\end{corollary}
\begin{proof}
Proposition~\ref{prop:relu_local_metric_section3} gives the required
Lipschitz constants on each entire gate piece, whether connected or not.
Apply Theorem~\ref{thm:piecewise_regular_section3} and use
\(\sum_rL_r^m\le M\eta^{m/2}\).
\end{proof}

The manifold gives the input covering
law, and the piecewise regular Jacobian transfers it to feature space.
The spectral quantity \(\|S_r^{\rm norm}\|_{\rm op}^{1/2}\) is a
valid ambient Lipschitz constant. It need not be the smallest possible
constant restricted to the manifold. Contraction in these local metrics
can reduce feature covers substantially, while occupied pattern counts
control the number of pieces.

\subsection{Homogeneity, residual curvature, and regularization}

\begin{proposition}[Residual curvature of homogeneous predictors]
\label{prop:homogeneous_residual_identity}
Suppose each predictor \(f(x_i;\theta)\) is twice differentiable
near \(\widehat\theta\) and positively homogeneous of degree
\(k>1\) in the jointly trained parameter. Write
\(\rho=\|\widehat\Delta\|_{\rm op}\). If
\(\widehat\theta\ne0\) is a stationary point of
\(\widehat L(\theta)+\lambda\|\theta\|_2^2/2\), then
\[
\widehat\Delta\widehat\theta
=-(k-1)\lambda\widehat\theta,
\qquad \rho\ge(k-1)\lambda.
\]
For the bias free one hidden layer ReLU model with both layers
trained, \(k=2\), so the residual curvature eigenvalue in the
radial parameter direction is \(-\lambda\).
\end{proposition}
\begin{proof}
Euler's identity gives
\(\langle\nabla_\theta f(x_i;\theta),\theta\rangle=kf(x_i;\theta)\).
Differentiating with respect to \(\theta\), then evaluating at
\(\widehat\theta\), yields
\[
\nabla_\theta^2f(x_i;\widehat\theta)\widehat\theta
=(k-1)g_i.
\]
Writing \(r_i=f(x_i;\widehat\theta)-y_i\), stationarity is
\(n^{-1}\sum_i r_i g_i=-\lambda\widehat\theta\). Hence
\[
\widehat\Delta\widehat\theta
=\frac1n\sum_i r_i\nabla_\theta^2f(x_i;\widehat\theta)\widehat\theta
=-(k-1)\lambda\widehat\theta.
\]
The operator norm bound follows because \(\widehat\theta\ne0\).
For shallow ReLU, scaling both \(a\) and \(W\) by a positive
factor scales \(f\) by its square, proving the final claim.
\end{proof}

\begin{proposition}[Residual control for one hidden layer ReLU]
\label{prop:rho_relu_section3}
For \(f(x;\theta)=q^{-1/2}a^\top\sigma(Wx)\), assume
\(\widehat w_j^\top x_i\ne0\) for every \(i,j\). Then
\[
\rho\le
\sqrt{\frac{2\widehat L(\widehat\theta)}q}
\left(\frac1n\sum_{i=1}^n\|x_i\|_2^2\right)^{1/2}.
\]
\end{proposition}
\begin{proof}
On a fixed gate pattern, the only nonzero second derivatives are
the mixed \(a\)--\(W\) blocks. With \(D_i=\widehat D(x_i)\),
the block mapping a weight perturbation \(U\) to an output weight
perturbation is \(B_i(U)=q^{-1/2}D_i Ux_i\), so
\(\|B_i\|_{\rm op}\le q^{-1/2}\|x_i\|_2\). The symmetric Hessian
\(\bigl[\begin{smallmatrix}0&B_i\\B_i^\top&0\end{smallmatrix}\bigr]\)
has the same operator norm. Therefore
\[
\rho\le\frac1{n\sqrt q}\sum_i|r_i|\|x_i\|_2
\le\sqrt{\frac{2\widehat L(\widehat\theta)}q}
\left(\frac1n\sum_i\|x_i\|_2^2\right)^{1/2}.
\]
\end{proof}
Thus the residual curvature norm is controlled by the training residuals
and the empirical input norm.

\begin{proposition}[$L_2$ regularization gives a sufficient local metric bound]
\label{prop:l2_regularization_contraction_section3}
Consider the normalized shallow ReLU model, with \(\lambda>0\).
Let \(\widehat\theta\) be any stationary point of
\(\widehat J_\lambda(\theta)=\widehat L(\theta)+\lambda\|\theta\|_2^2/2\)
at which all sample preactivations are nonzero. Then
\[
\widehat J_\lambda(\widehat\theta)
=\widehat L(0)-\frac1{2n}\sum_i f(x_i;\widehat\theta)^2
\le\widehat L(0).
\]
Consequently, on every activation stable piece,
\[
S_r^{\rm norm}\preceq
\frac{\|\widehat W\|_F^2+\|\widehat a\|_2^2}{q}I_d
\preceq\frac{2\widehat L(0)}{q\lambda}I_d.
\]
If the inputs and occupied pieces satisfy the manifold assumptions of
Corollary~\ref{cor:relu_piecewise_section3}, and
\(\widehat H_\lambda\succ0\), set \(\eta=2\widehat L(0)/(q\lambda)\).
Then the manifold bound is
\[
\deff(\widehat\theta;\lambda)
\le\gamma_{\rm pred}(\lambda)^{-2}
\left[
M+C_m(C_{\mathcal M}M\eta^{m/2})^{2/(m+2)}
\lambda^{-m/(m+2)}
\right].
\]
Under the width scaled penalty
\(\alpha_{\rm reg}\|\theta\|_2^2/(2q)\), this corresponds to
\(\lambda=\alpha_{\rm reg}/q\) and
\(\eta=2\widehat L(0)/\alpha_{\rm reg}\).
\end{proposition}
\begin{proof}
The predictor is homogeneous of degree two in \(\theta\).
Dotting stationarity with \(\widehat\theta\) and applying Euler's
identity gives
\[
\frac1n\sum_i\bigl(f(x_i;\widehat\theta)-y_i\bigr)
f(x_i;\widehat\theta)
+\frac\lambda2\|\widehat\theta\|_2^2=0.
\]
Expand the square loss and substitute this identity to obtain
the first term. Since \(\widehat L(\widehat\theta)\ge0\),
\(\lambda\|\widehat\theta\|_2^2/2\le\widehat L(0)\).
Moreover \(0\preceq D_r\preceq I_q\) implies
\[
S_r^{\rm norm}
\preceq q^{-1}(\|\widehat W\|_F^2+\|\widehat a\|_2^2)I_d.
\]
This proves the metric bound. Apply
Theorem~\ref{thm:piecewise_regular_section3}, noting
\(\sum_r L_r^m\le M\eta^{m/2}\), and then the prediction space
comparison.
\end{proof}

\paragraph{Networks with biases.}
The feature covering argument also applies to ReLU networks
with biases and multiple hidden layers. On each complete
activation region, the parameter Jacobian is affine in the
input. This allows a fairly straightforward extension, although we may lose some interpretability of the constants involved.

\begin{proposition}[ReLU networks with biases and arbitrary depth]
\label{prop:deep_biased_relu}
Consider a network with $K\ge2$ affine layers and scalar output,
\[
h_0(x)=x,\qquad
h_\ell(x)=\sigma(W_\ell h_{\ell-1}(x)+b_\ell),
\quad 1\le\ell<K,
\]
\[
f(x;\theta)=\kappa\bigl(W_Kh_{K-1}(x)+b_K\bigr),
\]
where $\kappa>0$ is a fixed normalization and $\theta$ contains
all weights and biases. Fix $\widehat\theta$ and an activation
cell $U$ on which all hidden preactivations are nonzero and
the gate matrices $D_1,\ldots,D_{K-1}$ are constant. Define
\[
T_0=I_d,\qquad
T_\ell=D_\ell\widehat W_\ell T_{\ell-1},
\quad 1\le\ell<K,
\]
and
\[
\delta_K=1,\qquad
\delta_\ell=D_\ell\widehat W_{\ell+1}^{\top}\delta_{\ell+1},
\quad 1\le\ell<K.
\]
Then, for every $x,z\in U$,
\[
\|g(x)-g(z)\|_2^2=(x-z)^\top S_U(x-z),
\qquad
S_U=\kappa^2\sum_{\ell=1}^K
\|\delta_\ell\|_2^2 T_{\ell-1}^{\top}T_{\ell-1}.
\]
Consequently, $L_U=\|S_U\|_{\rm op}^{1/2}$ is a Lipschitz
constant for the Jacobian on $U$.
\end{proposition}

\begin{proof}
On a fixed activation cell, each hidden representation is
affine in the input, with linear part $T_\ell$. Differentiation
gives
\[
\nabla_{W_\ell}f(x;\widehat\theta)
=\kappa\delta_\ell h_{\ell-1}(x)^\top,
\qquad
\nabla_{b_\ell}f(x;\widehat\theta)=\kappa\delta_\ell.
\]
The bias derivative blocks are constant on the cell. The
difference of the weight derivative blocks at $x$ and $z$ is
\[
\kappa\delta_\ell\bigl(T_{\ell-1}(x-z)\bigr)^\top.
\]
Its squared Frobenius norm is
$\kappa^2\|\delta_\ell\|_2^2
\|T_{\ell-1}(x-z)\|_2^2$.
Summing over the layers proves the identity.
\end{proof}

Under the manifold assumptions of
Theorem~\ref{thm:piecewise_regular_section3}, intersect the
occupied activation cells with $\mathcal M$ and take
$L_r=\|S_{U_r}\|_{\rm op}^{1/2}$.
The same feature covering and effective dimension bounds
then apply, with
\[
A=C_{\mathcal M}\sum_{r=1}^M
\|S_{U_r}\|_{\rm op}^{m/2}.
\]
This includes the activation partitions of networks with
multiple hidden layers and biases, as illustrated in
Figure~\ref{fig:main_activation_regions}.

\begin{corollary}[Regularization controls the local metrics]
\label{cor:deep_biased_regularization}
For the network in Proposition~\ref{prop:deep_biased_relu},
let $\widehat\theta$ be a stationary point of
$\widehat L(\theta)+\lambda\|\theta\|_2^2/2$, with $\lambda>0$,
at which all sample preactivations are nonzero. Then
\[
\begin{aligned}
\widehat L(0)
={}&\widehat L(\widehat\theta)
+\frac1{2n}\sum_i f(x_i;\widehat\theta)^2\\
&+\frac{\lambda}{K}
\left(
\sum_{\ell=1}^K\|\widehat W_\ell\|_F^2
+\sum_{\ell=1}^K\ell\|\widehat b_\ell\|_2^2
\right).
\end{aligned}
\]
Consequently, on every strict activation cell,
\[
S_U\preceq\eta I_d,
\qquad
\eta=\kappa^2K
\left(\frac{\widehat L(0)}{\lambda}\right)^{K-1}.
\]
\end{corollary}

\begin{proof}
For $s>0$, replacing each $W_\ell$ by $sW_\ell$ and each
$b_\ell$ by $s^\ell b_\ell$ multiplies $h_\ell$ by $s^\ell$
and the network output by $s^K$. Differentiating at $s=1$
and using stationarity gives
\[
\begin{aligned}
&\frac1n\sum_i
\bigl(f(x_i;\widehat\theta)-y_i\bigr)
f(x_i;\widehat\theta)\\
&\qquad=
-\frac{\lambda}{K}
\left(
\sum_{\ell=1}^K\|\widehat W_\ell\|_F^2
+\sum_{\ell=1}^K\ell\|\widehat b_\ell\|_2^2
\right).
\end{aligned}
\]
Expanding the square loss proves the first identity.
Since the other terms on its right hand side are
nonnegative,
\[
\sum_{\ell=1}^K\|\widehat W_\ell\|_F^2
\le \frac{K\widehat L(0)}{\lambda}.
\]
The forward and backward recursions, together with
$\|D_\ell\|_{\rm op}\le1$, give
\[
S_U\preceq
\kappa^2\left(
\sum_{\ell=1}^K
\prod_{j\ne\ell}\|\widehat W_j\|_{\rm op}^2
\right)I_d.
\]
For nonnegative $a_1,\ldots,a_K$, the elementary symmetric
mean inequality gives
\[
\sum_{\ell=1}^K\prod_{j\ne\ell}a_j
\le K\left(\frac{\sum_j a_j}{K}\right)^{K-1}.
\]
Apply this with $a_j=\|\widehat W_j\|_{\rm op}^2$ and use
$\sum_j a_j\le\sum_j\|\widehat W_j\|_F^2$ to obtain the
claimed metric bound.
\end{proof}

For $K=2$ and $\kappa=q^{-1/2}$, this gives
$\eta=2\widehat L(0)/(q\lambda)$, recovering the shallow
network bound with biases included.

If the inputs satisfy the manifold assumptions of
Theorem~\ref{thm:piecewise_regular_section3}, let $M$ denote
the number of occupied complete activation patterns.
When $\widehat H_\lambda\succ0$, the resulting bound is
\[
\deff(\widehat\theta;\lambda)
\le
\gamma_{\rm pred}(\lambda)^{-2}
\left[
M+C_m
\bigl(C_{\mathcal M}M\eta^{m/2}\bigr)^{2/(m+2)}
\lambda^{-m/(m+2)}
\right].
\]
Indeed, $A\le C_{\mathcal M}M\eta^{m/2}$, so this follows
from Theorem~\ref{thm:piecewise_regular_section3} and
Proposition~\ref{prop:gamma_pred}.

\subsection{Scale and interpretation}
In the unnormalized parameterization
\(f(x;\theta)=a^\top\sigma(Wx)\), the same Jacobian metric is
\(\widehat W^\top D_r\widehat W+\|D_r\widehat a\|_2^2I_d\).
The width normalized parameterization divides it by \(q\).
The magnitude of this metric also depends on the fitted weight norms.

If the occupied pieces have comparable constants \(L_r\approx L\), their
geometric contribution is
\[
A^{2/(m+2)}\approx(C_{\mathcal M}ML^m)^{2/(m+2)}.
\]
The factor \(L^m\) makes contraction along the intrinsic data dimension
visible before optimization over the covering radius. The curvature
factor is \(\gamma_{\rm pred}(\lambda)^{-2}\).

\paragraph{Combined generalization bound.}\label{app:combined_generalization}
Under the selection and noise assumptions of Theorem~\ref{thm:main_prediction_stability},
the geometric bounds combine to yield a single generalization statement. With
$R_n=\mathbb E[\widehat L(\widehat\theta)]$ and the notation of
Proposition~\ref{prop:manifold_main}, set
\[
B_n(\lambda):=\mathbb E\Big[
\gamma_{\rm pred}(\lambda)^{-2}
\big(M+C_mA^{2/(m+2)}\lambda^{-m/(m+2)}\big)
\Big].
\]
Then Corollary~\ref{cor:main_gen_gap} gives
\[
\mathbb E[\gen_D(\widehat\theta)]
\le\sqrt{\frac{8R_nB_n(\lambda)}{\alpha n}}+
\frac{2B_n(\lambda)}{\alpha n}.
\]
Thus, for fixed noise curvature and bounded expected training loss,
$B_n(\lambda)=o(n)$ gives a vanishing bound, and $B_n(\lambda)=O(1)$ gives
$O(n^{-1/2})$. The expectation keeps the learned geometry and curvature
together, and independence between these quantities is not needed.

\begin{remark}[Interpreting the rate]
If
\(\mathbb E[\deff(\widehat\theta;\lambda_n)]\le b_n\) and
\(\mathbb E[\widehat L(\widehat\theta)]=O(1)\), with a fixed
positive noise curvature parameter, then
\[
\mathbb E[\gen_D(\widehat\theta)]
=O\!\left(\sqrt{b_n/n}+b_n/n\right).
\]
In particular the bound vanishes if \(b_n=o(n)\). For example,
uniformly bounded \(M\), \(A\), and
\(\gamma_{\rm pred}(\lambda_n)^{-2}\), with \(\lambda_n=n^{-s}\), give
\(b_n=O(1+n^{sm/(m+2)})\).
\end{remark}

\clearpage
\section{Experimental Details}
\label{app:experimental_details}
\subsection{Computable spectral bounds and analysis scales}
\label{app:analysis_scale}
The experiments use less general spectral forms of the effective dimension bounds that are easier to compute. At a fixed fitted model, keep the weights and response residuals fixed and define
\[
H_t=\widehat G+\widehat\Delta+tI,\qquad
\rho=\|\widehat\Delta\|_{\rm op},\qquad
d_t=\operatorname{tr}\bigl((H_t^{-1}\widehat G)^2\bigr)
\]
whenever $H_t\succ0$. For any $t>0$, the classical covering bound gives
\begin{equation}
\label{eq:experiment_classical_cover}
\dclass(\widehat G,t)\le
\inf_{\varepsilon>0}
\left[\mathcal C_J(\varepsilon)+\frac{\varepsilon^2}{t}\right].
\end{equation}
When $t>\rho$, $H_t\succ0$ and the residual curvature gives the further comparison
\begin{equation}
\label{eq:experiment_spectral_comparison}
\begin{aligned}
\dclass(\widehat G,t+\rho)&\le d_t\le\dclass(\widehat G,t-\rho)\\
&\le\inf_{\varepsilon>0}
\left[\mathcal C_J(\varepsilon)+\frac{\varepsilon^2}{t-\rho}\right].
\end{aligned}
\end{equation}

\begin{proof}
The norm bound gives
$\widehat G+(t-\rho)I\preceq H_t\preceq\widehat G+(t+\rho)I$.
All three matrices are positive definite. Inverting and taking congruences by $\widehat G^{1/2}$ yields
\[
0\preceq\widehat G^{1/2}(\widehat G+(t+\rho)I)^{-1}\widehat G^{1/2}
\preceq\widehat G^{1/2}H_t^{-1}\widehat G^{1/2}
\preceq\widehat G^{1/2}(\widehat G+(t-\rho)I)^{-1}\widehat G^{1/2}.
\]
For $0\preceq A\preceq B$,
$\operatorname{tr}(B^2)-\operatorname{tr}(A^2)=\operatorname{tr}((B-A)(B+A))\ge0$.
Taking traces of the squares and using cyclicity proves the two spectral inequalities. Equation~\eqref{eq:experiment_classical_cover} is Theorem~\ref{thm:covering_J_section3} at scale $t$. Applying it at $t-\rho$ proves the last inequality.
\end{proof}

At $t=\lambda$, $H_t$ is the training objective Hessian and $d_t=\deff(\widehat\theta;\lambda)$. Thus the stronger condition $\lambda>\rho$ supplies a computable specialization of the general prediction space comparison. Under the assumptions of Corollary~\ref{cor:main_gen_gap}, if $\lambda>\rho$ almost surely, set
$D_n^\rho=\mathbb E[\dclass(\widehat G,\lambda-\rho)]$ and $R_n=\mathbb E[\widehat L(\widehat\theta)]$. Then
\[
\mathbb E[\gen_D(\widehat\theta)]
\le\sqrt{\frac{8R_nD_n^\rho}{\alpha n}}+\frac{2D_n^\rho}{\alpha n}.
\]
The expectation includes the randomness of both $\widehat G$ and $\rho$.

For the experimental training objective, $\lambda_{\rm train}=\alpha_{\rm reg}/q$, where $\alpha_{\rm reg}$ is the penalty coefficient denoted by $\alpha$ in the training settings below. We evaluate the fitted geometry at post training scales $t$, which can differ from $\lambda_{\rm train}$. Classical spectral plots evaluate $\dclass(\widehat G,t)$ directly. In the nonlinear matrix diagnostics, we estimate $d_t$ by Hutchinson probes and conjugate gradients when $t>\rho$, and use $\dclass(\widehat G,t)$ otherwise. $C_t$ denotes the resulting analysis scale complexity estimate. These post training quantities measure compression and residual curvature at the fitted model and enter the empirical plug in scores below.

\subsection{Compute details}
All experiments were run on a MacBook (14 inch, 2024) with an M4 Pro chip and 24GB memory. Running all experiments takes approximately 8 hours on this system.

\subsection{Synthetic data generation}

The geometric scaling experiments use product tori with intrinsic dimension \(m\ge 4\). The response perturbation diagnostic below uses $m=1$ so that dense Hessian calculations are feasible.  Latents are sampled independently as \(\theta_i\sim {\rm Unif}([0,2\pi]^m)\), and embedded by
\[
z(\theta)=\frac1{\sqrt m}
(\cos\theta_1,\sin\theta_1,\ldots,\cos\theta_m,\sin\theta_m),
\qquad x=Qz(\theta),
\]
where \(Q\in\mathbb R^{d\times 2m}\) is a random orthonormal embedding. Rows are normalized after embedding.  The rich target used for the initialization compression and cover experiments is
\[
f_\star(\theta)=\frac1{\sqrt m}\sum_{j=1}^m
\bigl(\sin\theta_j+0.3\cos(2\theta_j)\bigr).
\]
The residual curvature and plug in score synthetic experiments use a seed dependent mixed Fourier target with up to 128 integer frequencies in \(\{-3,\ldots,3\}^m\), with moderate label noise.  The activation region frequency sweep uses
\[
f_{\star,k}(\theta)=\frac12\sum_{j=1}^4\sin(k\theta_j),
\qquad k\in\{1,2,4,8,16\}.
\]
Figure~\ref{fig:app_targets} shows the actual target families used in the experiments.

The clustered sphere experiment in Figure~\ref{fig:main_deff_validation_cluster}B samples \(K\) equal size clusters in \(\mathbb R^{60}\), with points drawn near random unit sphere centers and then normalized.  The displayed \(K\) sweep fixes the within cluster spread at \(\sigma=0.1\) and uses a rank one per cluster target, so increasing \(K\) increases the number of target directions the trained model must represent.

\begin{figure}[t]
  \centering
  \includegraphics[width=\textwidth]{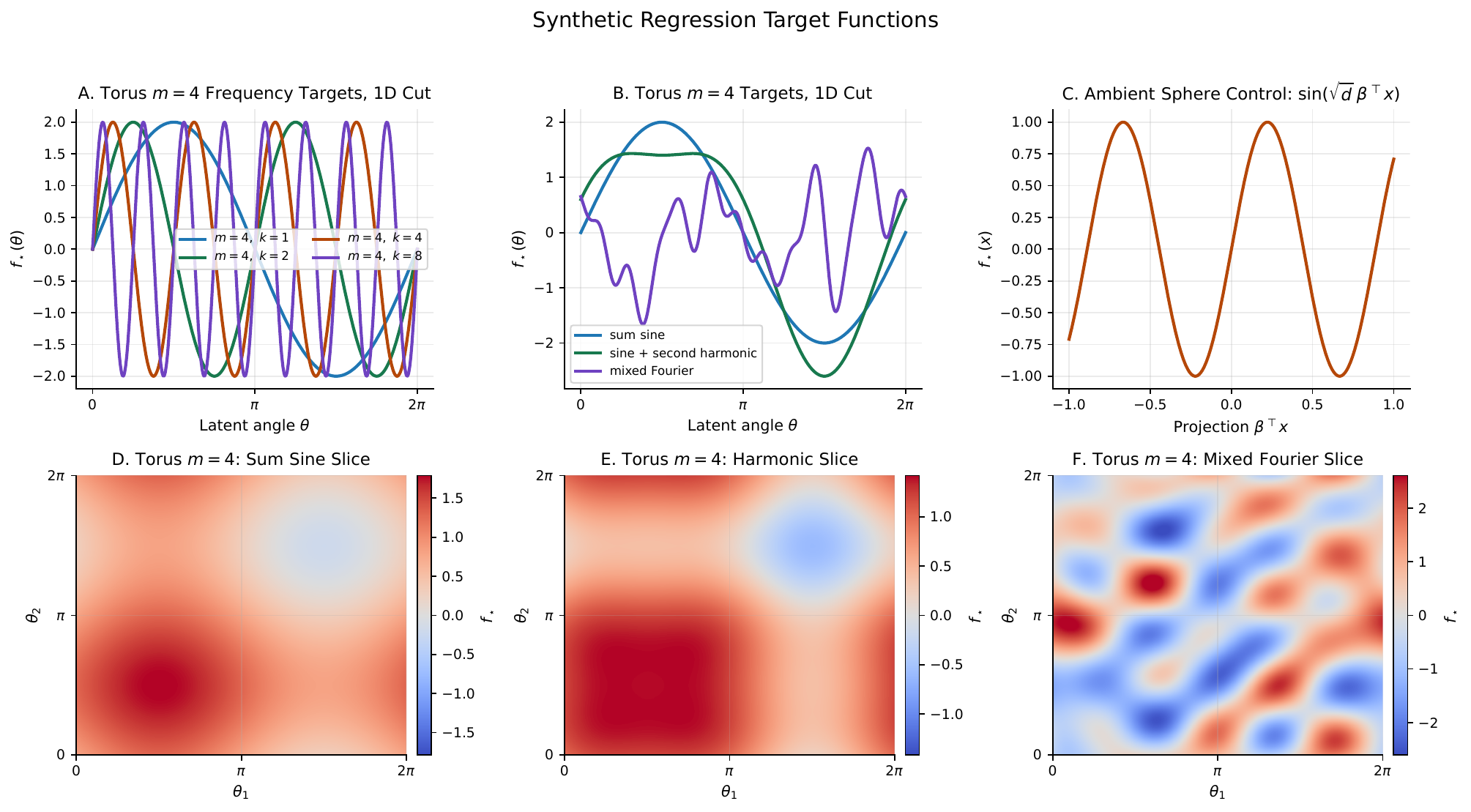}
  \caption{\textbf{Synthetic regression targets.}  The manifold experiments avoid intrinsic dimensions lower than 4 where results can quickly become trivial. The targets shown here use \(m=4\) one dimensional cuts or \(m=4\) two dimensional slices.}
  \label{fig:app_targets}
\end{figure}

\begin{figure}[t]
  \centering
  \includegraphics[width=\textwidth]{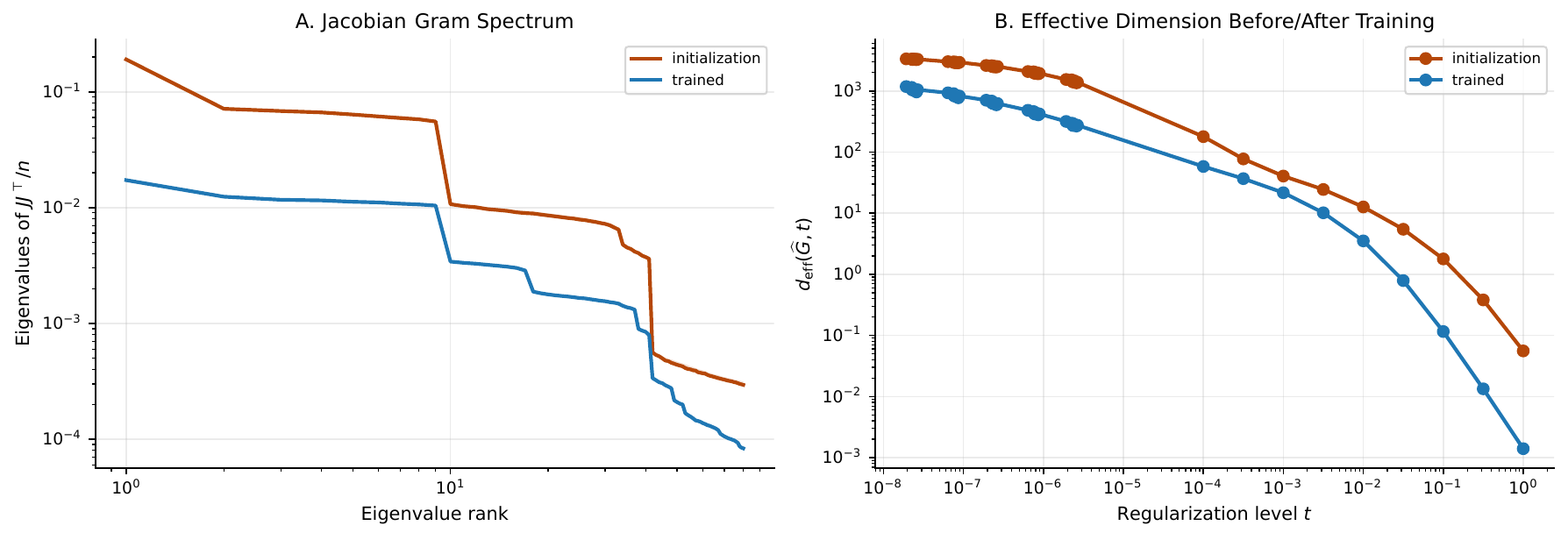}
  \caption{\textbf{Training compresses the relevant Jacobian geometry.} On the same \(m=4\) manifold regression task, the trained Jacobian Gram has a much faster spectral decay than the initialization Gram.  At \(t\approx 10^{-2}\), the median $d_{\rm lin}(\widehat G,t)$ drops from \(12.6\) at initialization to \(3.52\) after training, even though \(p=51{,}712\) and \(n=4096\).  This directly supports the theory's use of the trained Jacobian rather than a fixed NTK-style measure.  Solid curves show medians over ten seeds.}
  \label{fig:app_init_training}
\end{figure}

\subsection{Network architectures and training details}
All neural network experiments use a width normalized one hidden layer ReLU network without biases,
\[
f_\theta(x)=\frac1{\sqrt q}\sum_{r=1}^q a_r\,\sigma(w_r^\top x),
\qquad \sigma(u)=\max\{u,0\}.
\]
Thus \(p=q(d+1)\).  Training minimizes regularized full batch square loss,
\[
\frac{1}{2n}\sum_{i=1}^n(f_\theta(x_i)-y_i)^2+
\frac{\alpha}{2q}\bigl(\|W\|_F^2+\|a\|_2^2\bigr),
\]
using Adam.  We use the seed counts given in Table~\ref{tab:app_seed_counts}.  Unless stated otherwise, line and band summaries are medians with interquartile ranges over seed level runs.

\begin{table}[t]
\centering
\scriptsize
\caption{Dataset sizes, model widths, ambient dimensions, and optimization settings for the experiments.  Seed counts are given in Table~\ref{tab:app_seed_counts}. The real data diagnostics use \(\alpha=0.05\), 4000 steps, lr 0.006 for California Housing, \(\alpha=0.04\), 1500 steps, lr 0.005 for Wine Quality, and \(\alpha=0.04\), 5000 steps, lr 0.006 for UCI Superconductivity.}
\label{tab:app_hyperparameters}
\resizebox{\textwidth}{!}{%
\begin{tabular}{lllllll}
\toprule
Figure(s) & Dataset/process & \(n_{\rm train}\) & \(n_{\rm test}\) & \(d\) & \(q\) & Optimization \\
\midrule
Fig.~\ref{fig:main_deff_validation_cluster}A & noisy manifold, \(m\in\{4,8\}\), mixed Fourier & 256 & 4000 & 60, 80 & 256 & \(\alpha=0.01\), 2500 steps, lr 0.01 \\
Fig.~\ref{fig:main_deff_validation_cluster}B & clustered sphere, fixed \(\sigma=0.1\), \(K\) sweep & 2000 & 2000 & 60 & 2500 & \(\alpha=0.001\), 4000 steps, cosine lr 0.01 \\
Fig.~\ref{fig:app_init_training} & manifold, \(m=4\), rich target & 4096 & 20000 & 100 & 512 & \(\alpha=0.02\), 6000 steps, lr 0.01 \\
Fig.~\ref{fig:app_cover_bound} & manifold, \(m=4\), rich target & 4096 & 20000 & 100 & 512 & \(\alpha=0.02\), 6000 steps, lr 0.01 \\
Fig.~\ref{fig:main_activation_regions} & manifold, \(m=4\), frequency sweep & 2048 & 12000 & 100 & 512 & \(\alpha=0.004\), 5000 steps, lr 0.01 \\
Fig.~\ref{fig:main_supportive_bounds} synthetic & noisy manifold, \(m\in\{4,8\}\), mixed Fourier & 512, 1024, 2048 & 4000 & 64, 96 & 192 & \(\alpha=0.012\), 900 steps, lr 0.01 \\
Fig.~\ref{fig:main_supportive_bounds} real & California Housing, Wine Quality & 2048 & 4096, 2048 & 8, 12 & 512 & \(\alpha=0.04\), 1500 steps, lr 0.005 \\
Figs.~\ref{fig:app_california}, \ref{fig:app_wine}, \ref{fig:app_superconductivity} & real data geometry diagnostics & 4096, 2048 & 8192, 2048 & 8, 12, 81 & 512 & dataset specific settings in captions \\
\bottomrule
\end{tabular}}
\end{table}

\begin{table}[t]
\centering
\scriptsize
\caption{Number of seed level evaluations in each experiment, including sweeps over the analysis scale where applicable.}
\label{tab:app_seed_counts}
\begin{tabular}{lrrl}
\toprule
Experiment & Evaluations & Seeds & Intrinsic dimensions used \\
\midrule
Initialization geometry & 280 & 10 & \(m=4\) task \\
Jacobian feature covers & 140 & 10 & \(4\) \\
Activation regions & 50 & 10 & \(4\) \\
California Housing & 5 & 5 & PCA95 \(=6\) \\
Superconductivity & 5 & 5 & PCA95 \(=17\) \\
Wine Quality & 3 & 3 & PCA95 \(=9\) \\
Residual curvature comparison & 240 & 5 & \(4,8\) \\
Cluster difficulty & 84 & 4 & clustered sphere \\
Plug in score chart & 30 & 5 & \(4,8\) and real data \\
\bottomrule
\end{tabular}
\end{table}

\subsection{Effective dimension and cover computations}
For a trained two layer ReLU network, the Jacobian feature Gram matrix is computed exactly.  Writing \(\phi_i=(\sigma(w_1^\top x_i),\ldots,\sigma(w_q^\top x_i))\), we implement
\[
H_{ij}=\langle \nabla_\theta f_\theta(x_i),\nabla_\theta f_\theta(x_j)\rangle
=\frac1q\left[
\phi_i^\top\phi_j
 +(x_i^\top x_j)\sum_{r=1}^q a_r^2
 {\bf 1}\{w_r^\top x_i>0\}{\bf 1}\{w_r^\top x_j>0\}
\right].
\]

The empirical \(\varepsilon\)-cover is computed by farthest first greedy \(K\)-center clustering in either input space or Jacobian feature space.  In Jacobian feature space, squared distances are obtained from the Gram matrix as
\[
\|J_i-J_j\|^2=H_{ii}+H_{jj}-2H_{ij}.
\]
For each chosen \(K\), the cover radius \(\varepsilon_K\) gives
\[
B_{\rm cover}(K,t)=K+\varepsilon_K^2/t.
\]

\begin{figure}[t]
  \centering
  \includegraphics[width=0.74\linewidth]{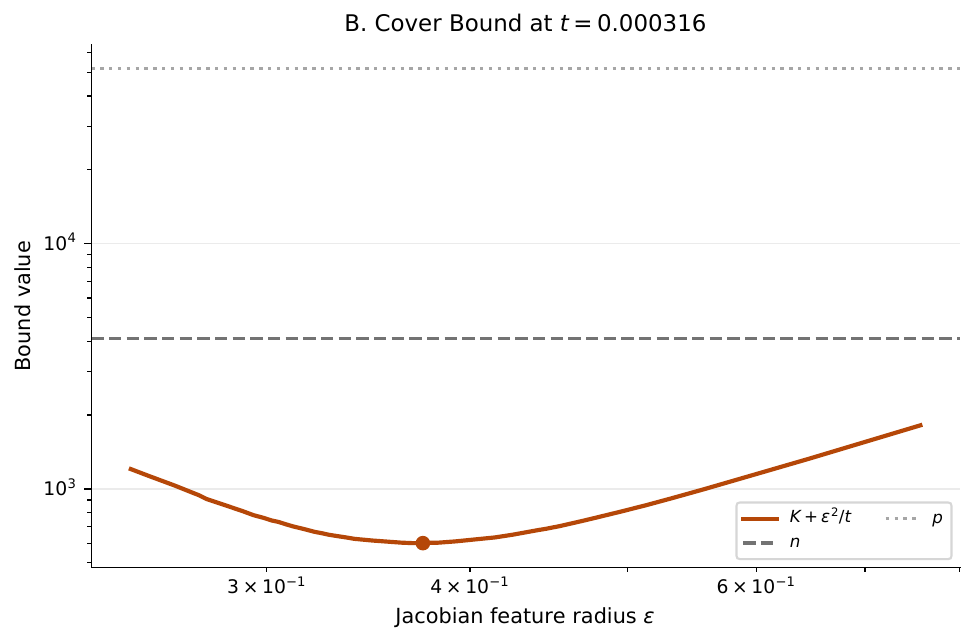}
  \caption{\textbf{Empirical Jacobian feature cover bound.}  At \(t=3.16\cdot 10^{-4}\), the best cover value is \(598\), well below \(n=4096\) and \(p=51{,}712\), while $d_{\rm lin}(\widehat G,t)$ is \(36.7\). Here we see the characteristic U shaped behaviour of the cover bound.  As the radius \(\varepsilon\) increases, the covering number \(K(\varepsilon)\) decreases, while the approximation term \(\varepsilon^2/t\) increases.  The turning point reflects the tradeoff between using fewer Jacobian feature cover centers and paying a larger within ball error penalty.}
  \label{fig:app_cover_bound}
\end{figure}

\subsection{Residual curvature and cluster difficulty diagnostics}

\begin{figure}[H]
  \centering
  \includegraphics[width=\textwidth]{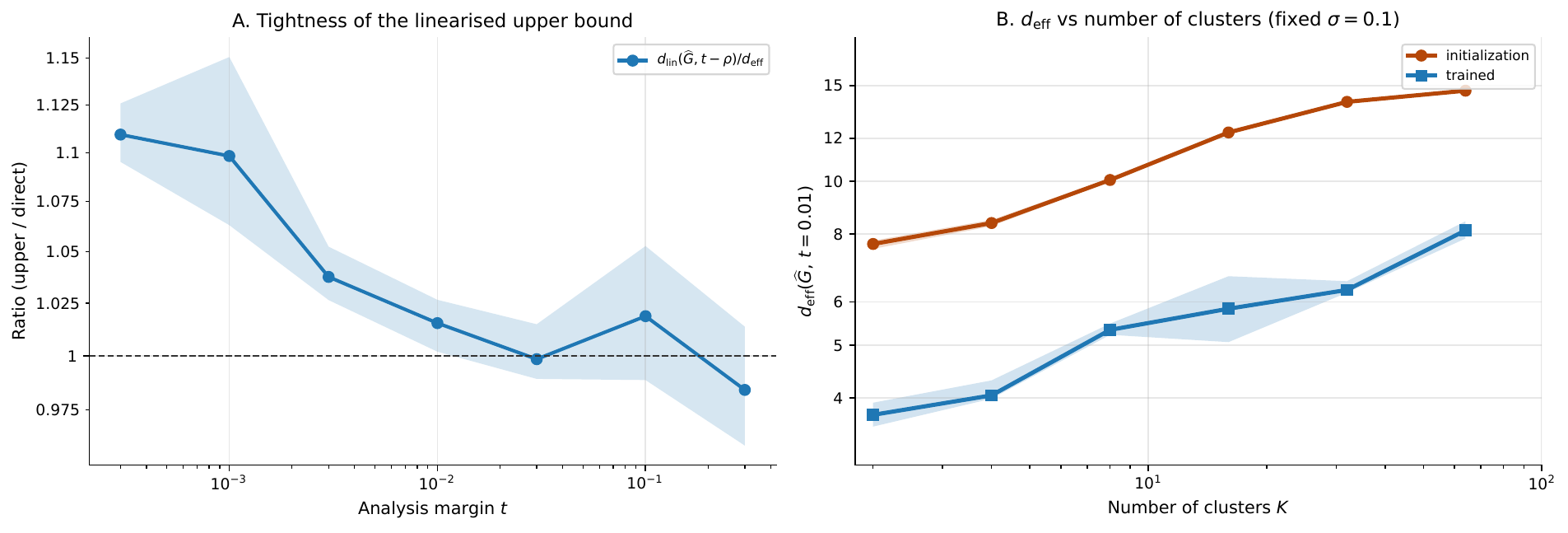}
  \caption{\textbf{The nonlinear effective dimension is controlled by trained Jacobian geometry.} (Left) Across noisy manifold regression tasks, \(d_{\rm lin}(\widehat G,t-\rho)\) closely tracks the directly estimated nonlinear effective dimension at post training scales $t>\rho$, where $\rho=\|\widehat\Delta\|_{\rm op}$. These analysis scales may differ from the training ridge $\lambda$. (Right) On clustered sphere data, the trained spectral complexity remains below the initialization geometry as the number of displayed clusters grows to \(K\le 100\), showing that feature learning can compress the relevant directions even when the data geometry becomes harder.}
  \label{fig:main_deff_validation_cluster}
\end{figure}

Figure~\ref{fig:main_deff_validation_cluster}A evaluates the post training nonlinear matrix diagnostic at scale $t$ using Hutchinson estimates for \(\operatorname{tr}((H_t^{-1}G)^2)\), with
\[
H_t=\widehat G+\Delta+tI,\qquad
\Delta=\frac1n\sum_{i=1}^n r_i\nabla^2_\theta f_{\widehat\theta}(x_i).
\]
The comparison curve uses the spectral proxy \(d_{\rm lin}(\widehat G,t-\rho)\), where \(\rho=\|\Delta\|_{\rm op}\).  We aggregate six noisy manifold configurations with \(m\in\{4,8\}\), noise levels \(\sigma\in\{0.05,0.10,0.15\}\), and five seeds.  Figure~\ref{fig:main_deff_validation_cluster}B uses the clustered sphere \(K\) sweep at fixed \(\sigma=0.1\). 

\subsection{Activation region diagnostics}
For the one hidden layer ReLU experiment in Figure~\ref{fig:main_activation_regions}, an occupied activation region is a unique binary activation pattern \(({\bf 1}\{w_r^\top x>0\})_{r=1}^q\) that contains at least one training point.  We also give the number of sampled cells \(M_{\rm full}\) appearing on a large held out set.  This is deliberately not a count of all combinatorially possible ReLU regions because the theory only needs the data occupied pieces.  The visualization panel uses a two dimensional biased ReLU partition solely for display, while the quantitative panel uses the \(m=4\) manifold frequency sweep described in Table~\ref{tab:app_hyperparameters}.

\subsection{Real data experiments}
We use three regression diagnostics.  California Housing \cite{kelleypacecali} is loaded from \nolinkurl{sklearn.datasets.fetch_california_housing}. The target is median house value and the ambient dimension is \(d=8\).  Wine Quality \cite{cortez2009winequality} combines the UCI red and white wine tables, using 11 physicochemical covariates plus a red/white indicator to predict integer quality scores.  UCI Superconductivity \cite{hamidieh2018superconductivty} uses 81 elemental composition covariates to predict critical temperature.  Inputs and targets are standardized using training set statistics.  The same Jacobian Gram, analysis scale complexity $C_t$, cover, partition, and plug in score computations are applied.

\begin{figure}[H]
  \centering
  \includegraphics[width=0.82\textwidth]{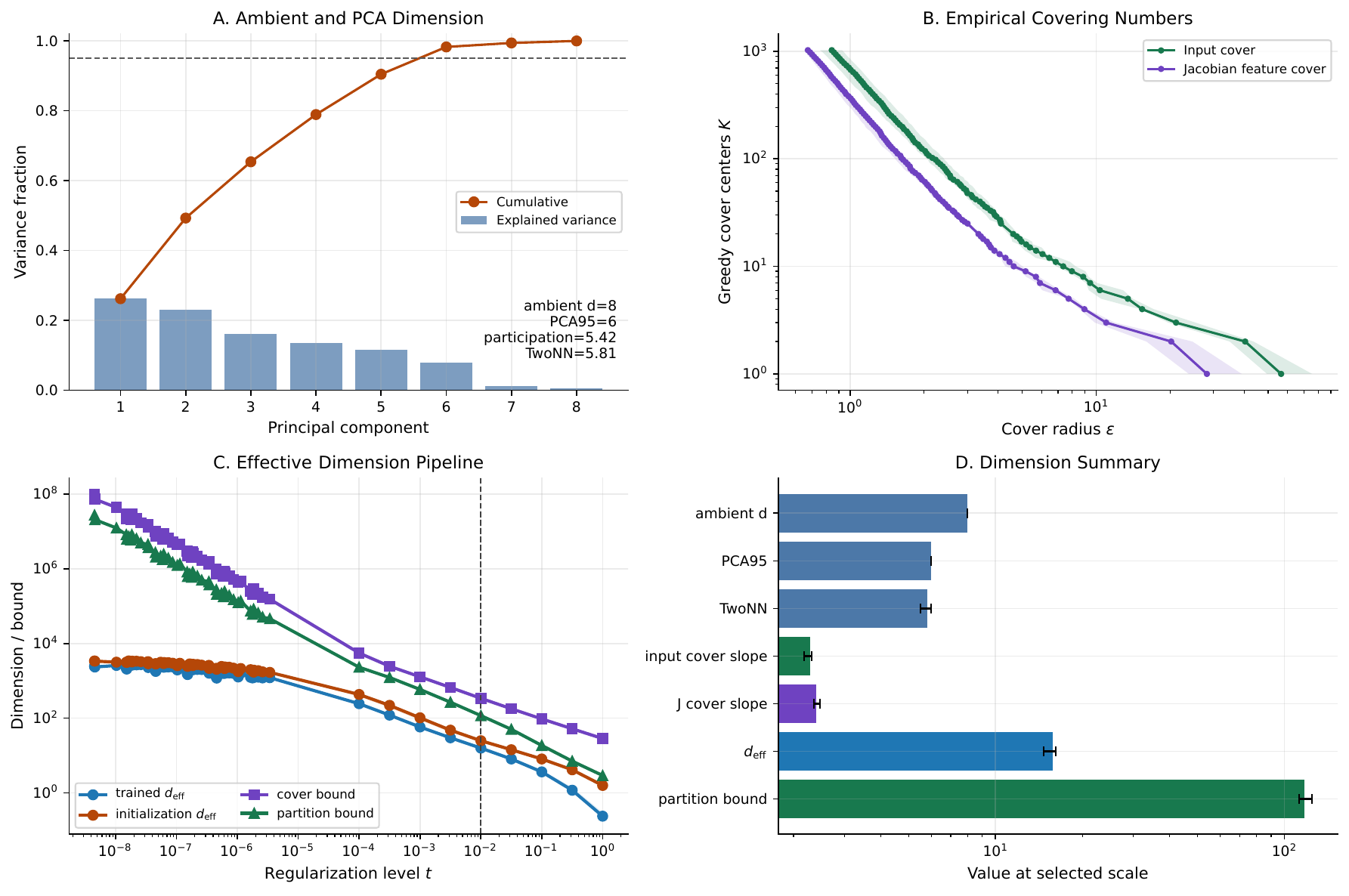}
  \caption{\textbf{California Housing geometry diagnostics.}  The panels show input PCA, empirical covering numbers, effective dimensions, and dimension summaries for the eight dimensional California Housing benchmark.}
  \label{fig:app_california}
\end{figure}

\begin{figure}[H]
  \centering
  \includegraphics[width=0.82\textwidth]{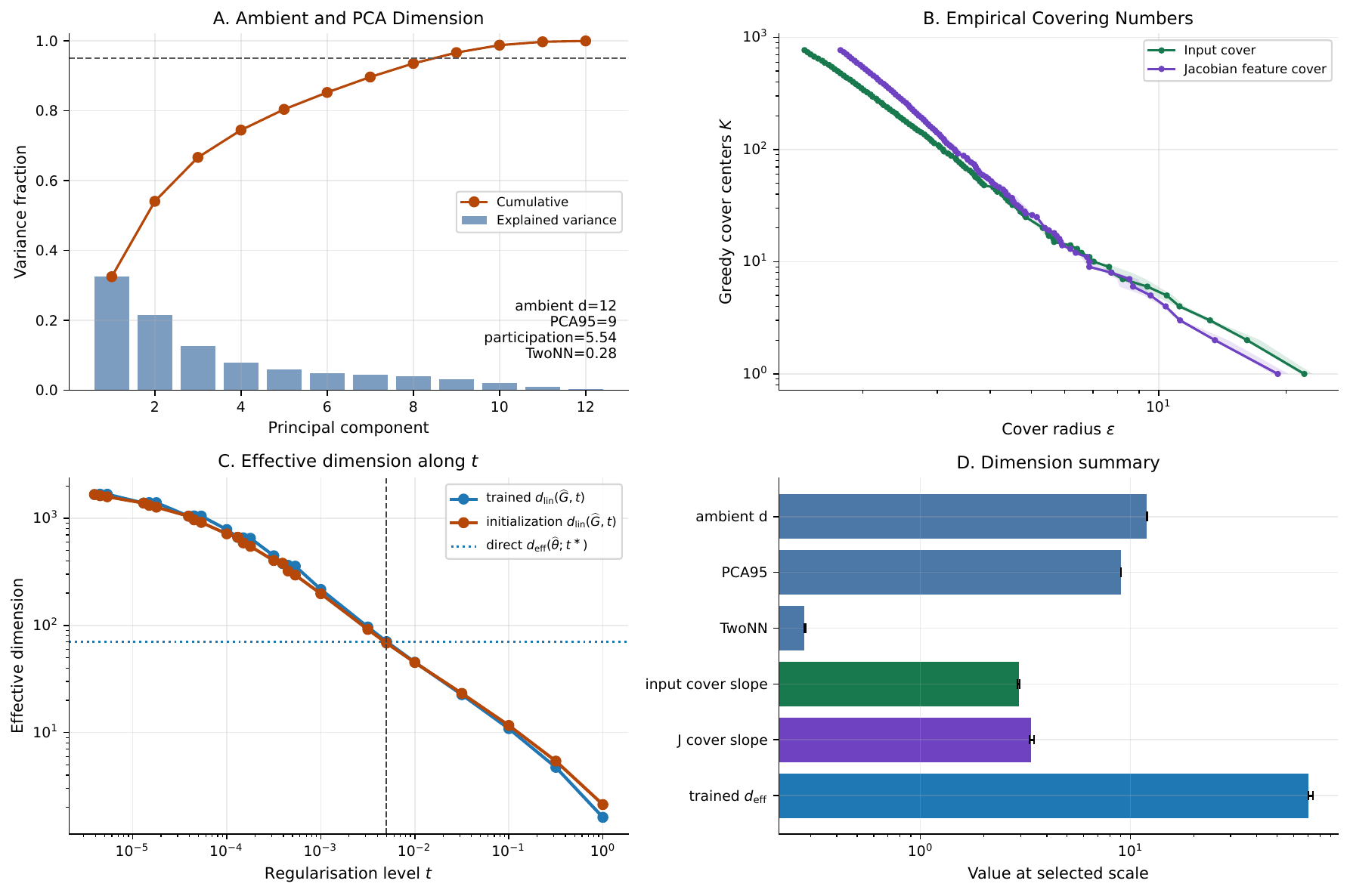}
  \caption{\textbf{UCI Wine Quality geometry diagnostics.}  The Wine Quality benchmark combines red and white wines, and the panels show input PCA, empirical covering numbers, effective dimensions, and dimension summaries.}
  \label{fig:app_wine}
\end{figure}

\begin{figure}[H]
  \centering
  \includegraphics[width=0.82\textwidth]{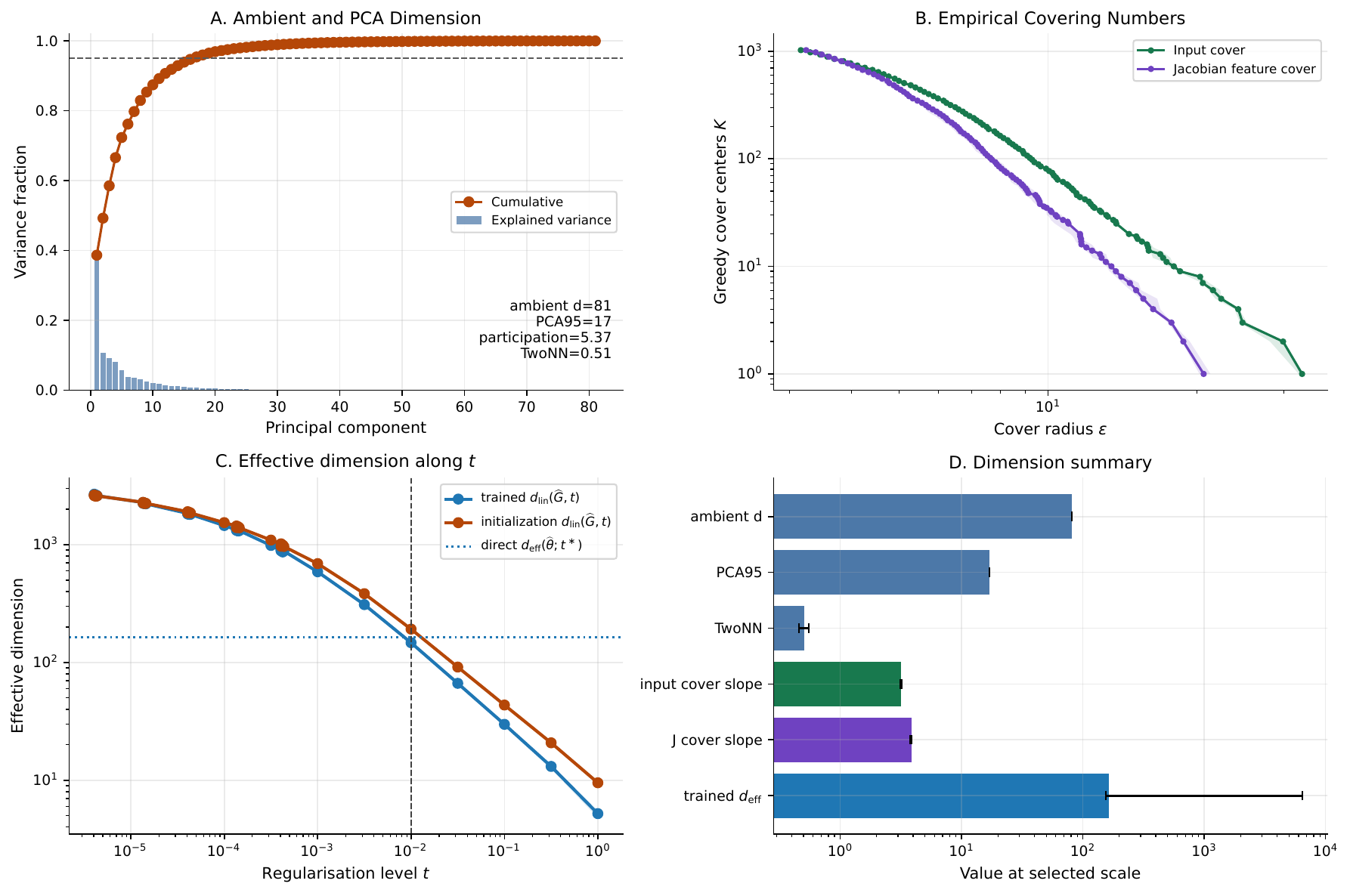}
  \caption{\textbf{UCI Superconductivity geometry diagnostics.}  The panels show input PCA, empirical covering numbers, effective dimensions, and dimension summaries for the 81 dimensional superconductivity benchmark.}
  \label{fig:app_superconductivity}
\end{figure}

\begin{figure}[H]
  \centering
  \includegraphics[width=\textwidth]{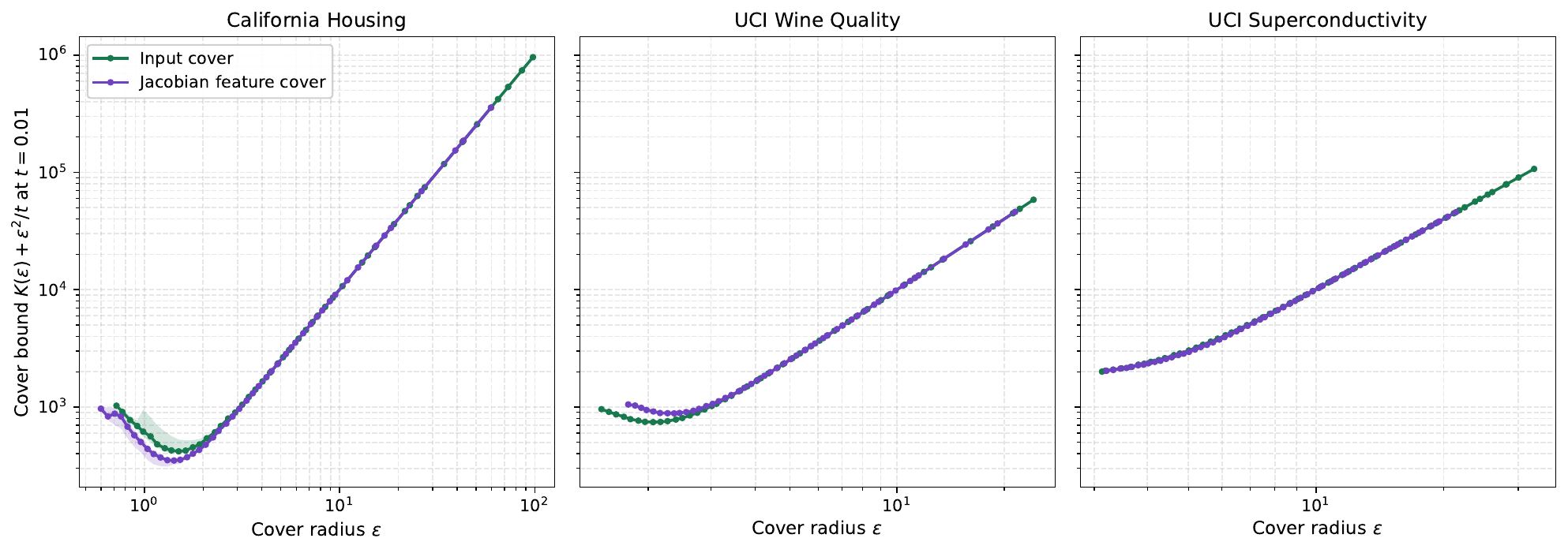}
  \caption{\textbf{Real data cover bound curves.}  For each real data benchmark, the plot compares input space covers and trained Jacobian feature covers through the cover bound \(K(\varepsilon)+\varepsilon^2/t\).  The U shaped curves make explicit the tradeoff in the theory. Small radii require many centers, while large radii pay a larger within ball approximation penalty.}
  \label{fig:app_real_covering_numbers}
\end{figure}

\begin{table}[t]
\centering
\scriptsize
\caption{Real data analysis scale complexity estimates $C_t$ and plug in scores, compared with held out gaps. Values are medians across seeds for each benchmark.}
\begin{tabular}{lrrrrrr}

\toprule
Dataset & \(d\) & PCA95 & TwoNN & $C_t$ & Gap & Plug in score \\
\midrule
California Housing & 8 & 6 & 5.61 & 15.5 & 0.0219 & 0.0653 \\
Wine Quality & 12 & 9 & 0.281 & 70.7 & 0.152 & 0.287 \\
UCI Superconductivity & 81 & 17 & 0.507 & 164 & 0.0323 & 0.204 \\
\bottomrule
\end{tabular}

\label{tab:app_real_numbers}
\end{table}

\begin{figure}[!tbp]
  \centering
  \includegraphics[width=\textwidth]{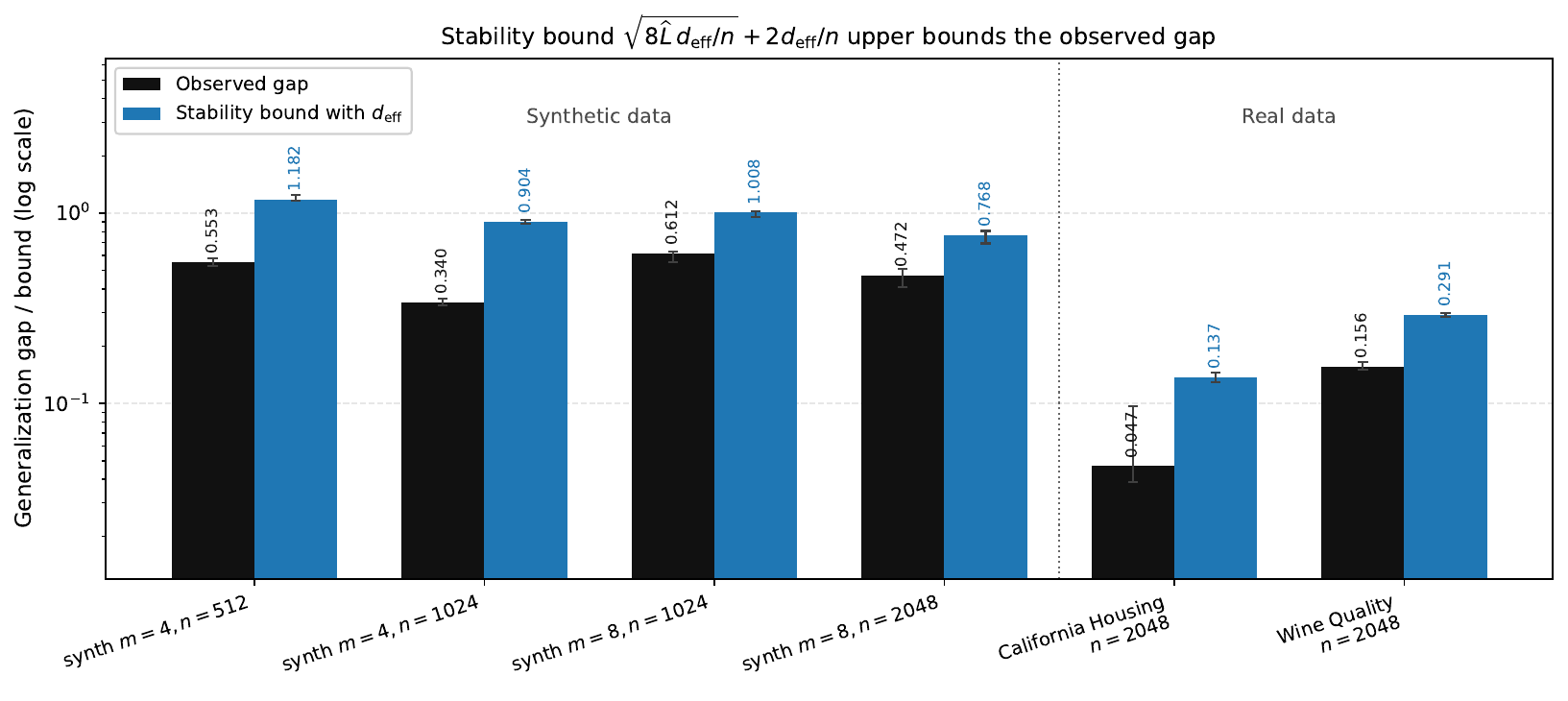}
  \caption{\textbf{
  Plug in geometry scores and held out gaps.}  
  The displayed score has the form of the stability bound, evaluated with the post training analysis scale and unit variance factor. In every configuration, the median plug in geometry score remains above but close to the observed gap.}
  \label{fig:main_supportive_bounds}
\end{figure}

\subsection{Plug in score numbers}
For Figure~\ref{fig:main_supportive_bounds}, observed generalization is the held out train test half MSE gap.  The displayed plug in score, motivated by the form of the bound in Section~\ref{sec:stability_setup}, is
\[
\sqrt{\frac{8\widehat L C}{n}}+\frac{2C}{n},
\]
where \(\widehat L\) is the training half MSE and $C=C_t$ is the analysis scale complexity estimate. We use a unit variance factor and compare the score with performance at held out inputs. This empirical comparison complements the fixed design theorem, which controls fresh responses at the same inputs under the specified noise law.  Table~\ref{tab:app_supportive_numbers} gives the medians underlying the bar chart.

\begin{table}[t]
\centering
\scriptsize
\caption{Median plug in score values used in Figure~\ref{fig:main_supportive_bounds}.  Each row aggregates five seeds. The plotted error bars show the 10--90\% seed interval.}
\label{tab:app_supportive_numbers}
\begin{tabular}{lrrrr}
\toprule
Configuration & Observed gap & Plug in score & Median $C_t$ & Score/gap \\
\midrule
Synth \(m=4,n=512\) & 0.553 & 1.182 & 176 & 2.21 \\
Synth \(m=4,n=1024\) & 0.340 & 0.904 & 191 & 2.67 \\
Synth \(m=8,n=1024\) & 0.612 & 1.008 & 321 & 1.65 \\
Synth \(m=8,n=2048\) & 0.472 & 0.768 & 345 & 1.63 \\
California Housing & 0.0470 & 0.137 & 30.1 & 2.92 \\
Wine Quality & 0.156 & 0.291 & 76.8 & 1.87 \\
\bottomrule
\end{tabular}
\end{table}
\subsection{Response perturbation diagnostics for the selection assumption}
\label{app:exp_selection}
The selection assumption concerns the learning rule as responses change. We examine it on a small torus instance with $m=1$, $d=10$, target $f_\star(\vartheta)=\sin\vartheta$, $n=64$, width $q=20$ ($p=220$), ridge $\lambda=2\times10^{-3}$, and Gaussian noise of standard deviation $0.05$. Full batch Adam (learning rate $0.01$, $6000$ steps) is followed by deterministic Newton refinement with minimum norm solves and objective backtracking. Eight datasets and three response replacements per dataset give $24$ paths
\[
Y^{(i)}(t)=Y+t(Y_i'-Y_i)e_i,\qquad t\in[0,1],
\]
each sampled at $25$ equally spaced points. We compare central finite differences and integrated prediction changes with $(Y_i'-Y_i)J\widehat H_\lambda^\dagger g_i/n$, and compare continuation with retraining from the same initialization.

Thirteen paths meet the stationarity ($10^{-9}$), relative solve ($10^{-10}$), and relative influence error ($10^{-4}$) tolerances, with median integrated error $4.0\times10^{-8}$. The Hessian spectra distinguish seven paths consistent with positive semidefiniteness from six stationary paths with negative curvature. On the seven numerical PSD paths, median integrated error is $2.1\times10^{-9}$ (maximum $4.0\times10^{-8}$), and null direction prediction energy is at most $9.5\times10^{-15}$. Table~\ref{tab:selection_diagnostics} also includes the eleven paths that do not meet the influence tolerances. Across all $24$ paths, the median relative prediction distance between cold start and continuation endpoints is $4.3\times10^{-6}$ and the maximum is $4.6\times10^{-3}$. These measurements support local sensitivity and endpoint agreement on the sampled response perturbations. We also compare conditional replacement averages with a truncated Hessian plug in stability score, obtaining a maximum ratio of $0.48$ across eight datasets.

\begin{table}[t]
\centering
\small
\caption{All $24$ response paths. Influence tests use the stated stationarity, solve, and finite difference tolerances. Numerical PSD is assessed from the minimum Hessian eigenvalues. For each path, $m_{\mathrm{act}}$ is the minimum normalized activation margin $|\widehat w_j(t)^\top x_i|/(\|x_i\|_2\|\widehat w_j(t)\|_2)$ over its $25$ sampled fits, all units, and all nonzero inputs, with each norm floored at $10^{-30}$ in computation.}
\label{tab:selection_diagnostics}
\begin{tabular}{@{}lrrr@{}}
\toprule
Sampled path group & Paths & Median influence error & Median $m_{\mathrm{act}}$ \\
\midrule
Influence tests (numerical PSD) & 7 & $2.1\times10^{-9}$ & $1.2\times10^{-2}$ \\
Influence tests (negative curvature) & 6 & $7.6\times10^{-8}$ & $2.9\times10^{-4}$ \\
Other paths & 11 & $3.6\times10^{-1}$ & $6.5\times10^{-13}$ \\
\bottomrule
\end{tabular}
\end{table}

\begin{figure}[!htbp]
\centering
\includegraphics[width=\textwidth]{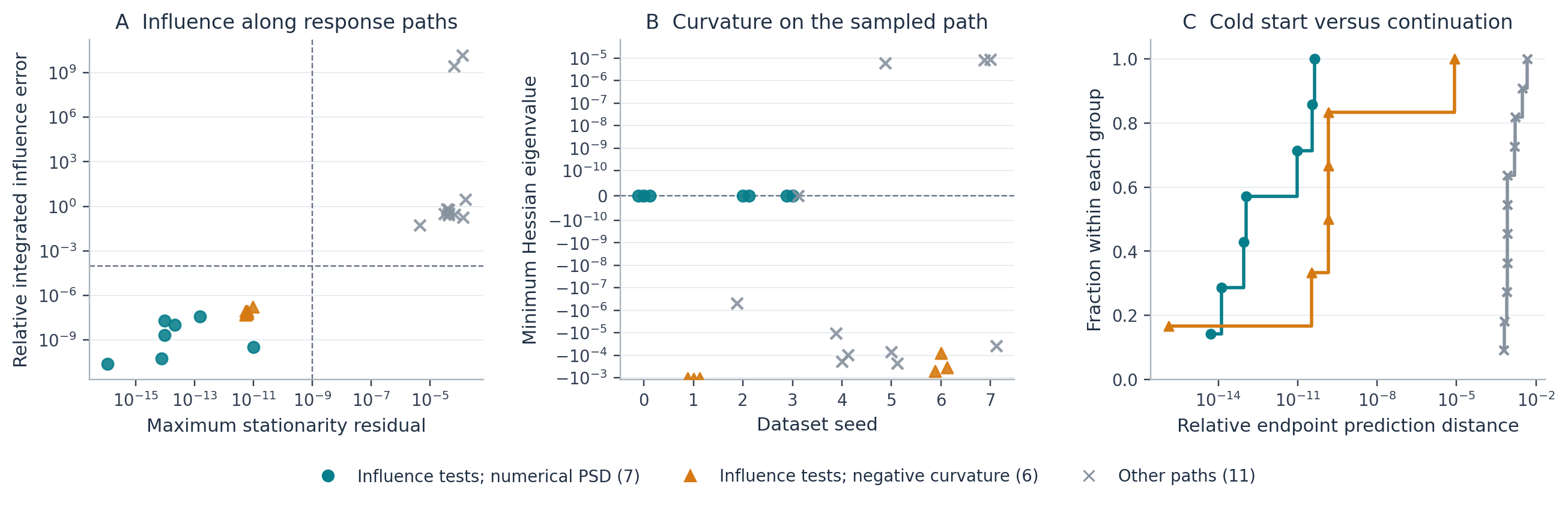}
\caption{Response perturbation diagnostics for all $24$ paths ($m=1$, $d=10$, $q=20$, $n=64$). (Left) Integrated influence error versus stationarity residual. Dashed lines mark the stated tolerances. (Center) Minimum sampled Hessian eigenvalue. (Right) Endpoint discrepancy between cold start retraining and continuation. The groups distinguish seven influence matching paths with numerical PSD, six with negative curvature, and eleven other paths.}
\label{fig:selection_diagnostics}
\end{figure}

\clearpage
\section{Extended Related Work}
\label{app:related_work}

This appendix expands the discussion in the introduction and positions the paper relative to various other directions including stability based generalization, effective dimension in kernel methods, neural tangent and fixed kernel analyses of neural networks, and geometric accounts based on manifolds, coverings, and activation regions.  The common theme is that many existing results control generalization through either a global function class complexity, a fixed kernel, or an algorithmic trajectory.  Our results however isolate a local, trained model quantity, the Jacobian covariance at the fitted point, measured in the inverse Hessian metric of the nonlinear least squares objective.

\subsection{Stability and generalization}

Classical learning theoretic bounds control the generalization gap uniformly over a hypothesis class through VC dimension, pseudodimension, metric entropy, or Rademacher complexity \cite{anthony1999neural,bartlett2019nearly,koltchinskii2002rademacher}.  PAC Bayesian and information theoretic approaches give nonuniform alternatives by replacing a worst case class complexity with a posterior dependent or information dependent term \cite{mcallester1998pac,shalev2014understanding,xu2017information}.  In the neural network setting, these approaches have produced influential norm, margin, PAC Bayes, and compression based bounds \cite{bartlett2017spectrally,neyshabur2018pac,dziugaite2017computing,lotfi2022pac}.  They are important baselines for any theory of generalization in overparameterized models.  However, several works have also emphasized that conventional capacity measures and uniform convergence arguments can miss the behavior of trained deep networks, especially in regimes where networks interpolate or where the relevant inductive bias is strongly algorithm dependent \cite{zhang2021understanding,nagarajan2019uniform}.  Our bounds are closer in spirit to such nonuniform views, but the nonuniform object here is the local Jacobian geometry of the actual fitted least squares solution.

Algorithmic stability offers a direct route from the sensitivity of a learning rule to expected generalization \cite{bousquet2002stability,shalev2014understanding}.  The stability analysis of stochastic gradient methods \cite{pmlr-v48-hardt16} showed that the number of optimization steps, Lipschitz/smoothness constants, and convexity assumptions can be used to control generalization.  Subsequent work extended stability analyses to broader convex, nonsmooth, nonconvex, and weakly convex settings, often by imposing optimization landscape conditions such as the Polyak--Lojasiewicz or quadratic growth condition, or by using algorithmic mechanisms such as step size decay, clipping, noise, or other forms of regularization \cite{CharlesP18,bassily2020stability,lei2020finegrained,lei2023stability}.  Our perspective is different in that we do not prove that a particular optimizer is stable along its whole trajectory.  Instead, we analyze the replace one sensitivity of a nondegenerate local minimizer of the original ridge regularized nonlinear least squares objective, selected by a learning rule satisfying the selection assumption.  Stability arises when the fitted predictor has low effective dimension in the trained Jacobian metric.

The proof is also related to sensitivity and influence function analyses for regularized estimators.  In the linear least squares case, the response derivative is governed by the ridge inverse covariance.  For nonlinear least squares, the analogous derivative involves the inverse Hessian at the fitted local minimizer.  The pointwise confidence bounds of Kuzborskij and Abbasi-Yadkori \cite{kuzborskij2025pointwiseconfidenceestimationnonlinear} use closely related inverse Hessian sensitivity quantities for fixed design nonlinear least squares.  Our contribution is to convert this local response sensitivity into an on average stability and generalization bound, and then to identify when the resulting trace quantity is small.  The use of Gaussian Poincar\'e and Brascamp--Lieb inequalities \cite{brascamp1976, liebmodern} places the argument in the broader tradition of variance control methods, but the final complexity measure is specific to the geometry learned by the fitted model.

\subsection{Effective dimension and kernel ridge regression}

Effective dimension, statistical dimension, and degrees of freedom are central quantities in the analysis of regularized least squares and kernel ridge regression.  For a fixed kernel covariance or empirical Gram matrix, quantities of the form
\[
    \operatorname{tr}\bigl((K+\lambda I)^{-1}K\bigr)
    \quad\text{or}\quad
    \operatorname{tr}\bigl(((K+\lambda I)^{-1}K)^2\bigr)
\]
measure how many spectral directions survive ridge regularization.  Such quantities appear in learning rate analyses for kernel regression \cite{zhang2005learning,CaponnettoV07,neu2018iterate}, in low rank kernel approximation and Nystr\"om analyses \cite{bach2013sharpanalysislowrankkernel}, and in statistical computational tradeoffs for random features and Fourier feature approximations \cite{rudi2021generalizationpropertieslearningrandom,avron2017random}.  Recent work on ridgeless regression and benign overfitting also shows that spectral structure can permit good generalization even when models interpolate, especially when covariance or kernel eigenvalues have favorable decay \cite{liang2020just,bartlett2020benign}.

The present paper should be viewed as a nonlinear extension of this spectral line of work, but with two important changes.  First, the feature map is not fixed in advance.  In a neural network or other nonlinear predictor, the relevant Jacobian features can change substantially during training.  Our effective dimension therefore uses the trained Jacobian covariance \(\widehat G\), instead of the covariance at initialization.  Second, the local inverse metric is different to \((\widehat G+\lambda I)^{-1}\).  Nonlinearity contributes a residual curvature term, so the metric is the inverse Hessian \(\widehat H_{\lambda}^{-1}\) of the fitted objective.
A notion of effective dimension involving $\wh H_{\lambda}^{-1}$ appeared
in generalization analysis of Gibbs predictors with nonconvex potentials \citep{kuzborskij2019distribution}, however, here we focus on deterministic learning procedure rather than a randomized one.
When the predictor is affine in parameters, or when the residual curvature term vanishes, our quantity reduces to the classical ridge effective dimension.  Away from that case, it captures how the learned feature geometry and local objective curvature jointly determine sensitivity.

Our partition and covering bounds are also connected to low rank approximation results for kernel methods.  In kernel ridge regression, fast spectral decay or accurate low rank approximation lowers the effective degrees of freedom and computational cost \cite{bach2013sharpanalysislowrankkernel,avron2017random}.  Here the low rank object is not a preexisting kernel matrix but the empirical covariance of trained Jacobian features.  A partition of the sample into cells with small within cell Jacobian scatter gives a low rank plus residual decomposition of \(\widehat G\).  This is how occupied activation regions, empirical covers, and manifold covers enter the bound, as geometric mechanisms for making the trained Jacobian covariance effectively low rank.

\subsection{NTK and fixed kernel neural network theory}

Neural tangent kernel theory analyzes wide neural networks through a linearization around random initialization \cite{jacot2018neural}.  In this regime, training dynamics can be approximated by kernel gradient descent with an initialization kernel, and several works use this framework to establish optimization and generalization guarantees for overparameterized networks \cite{du2019gradientdescentfindsglobal,arora2019exactcomputationinfinitelywide,cao2019generalizationboundsstochasticgradient}.  The linearized perspective was further clarified by results showing that sufficiently wide networks evolve as linear models under gradient descent \cite{lee2019wide}.  The broader ``lazy training'' viewpoint identifies scaling regimes in which parameters move little and the predictor remains close to its first order Taylor expansion \cite{chizat2019lazy}.

This fixed kernel perspective is powerful, but it intentionally suppresses feature learning.  Work on lazy versus rich regimes shows that the scale of initialization and parametrization can control whether training behaves like kernel regression or instead learns representations that are not captured by a fixed RKHS norm \cite{woodworth2020kernel}.  Related infinite width parametrizations such as maximal update parametrization were developed precisely to preserve nontrivial feature learning as width changes \cite{yang2021tensor}.  Our analysis is aimed at this learned feature side of the picture.  We do not assume infinite width, small parameter displacement, or a fixed tangent kernel.  The relevant kernel like object is the Jacobian Gram matrix at the trained parameter, and the corresponding effective dimension is evaluated in the local Hessian metric of the nonlinear objective.

The distinction from NTK theory is especially visible in the experiments and in the covering bounds.  NTK analyses typically ask whether the initialization Jacobian Gram is well conditioned and remains stable during training.  Our bounds ask whether the trained Jacobian features are compressible after training.  Thus the theory can explain a decrease in effective dimension from initialization to the fitted model, a phenomenon that fixed kernel analyses are not designed to capture.

\subsection{Geometry, manifolds, coverings, and activation regions}

Covering numbers and metric entropy are classical tools for converting geometric size into statistical complexity.  The basic idea goes back to metric entropy and capacity in functional spaces \cite{kolmogorov1961entropy}, and it became central in empirical process theory through chaining, entropy integrals, and uniform laws of large numbers \cite{dudley1967sizes,pollard1984convergence,vanderVaartWellner1996,dudley1999uniform,carl1990entropy}.  In learning theory, Rademacher and Gaussian complexity bounds can often be proved by first covering a function class and then applying symmetrization and chaining \cite{bartlettMendelson2002rademacher,koltchinskii2002rademacher}.
In overparameterized learning problems such bounds can be way too pessimistic.
To this, a recent literature has also explored notions of algorithm dependent covers of the function class, for instance, the subset of the parameter space that is realized by the optimization algorithm \citep{camuto2021fractal,dupuis2024uniform}.

Our use of covers is deliberately more local.  Rather than covering the entire nonlinear neural network function class, we cover the finite set of trained Jacobian features \(\{\nabla_\theta f(x_i;\widehat\theta)\}_{i=1}^n\), so the complexity term reflects the geometry actually realized by the fitted model on the observed data.

The closest classical precedent for the Lipschitz covering part of our argument is the metric space learning framework of von Luxburg and Bousquet \cite{vonluxburg2004distance}.  They study Lipschitz classifiers on bounded metric spaces, relate inverse Lipschitz constant to a margin, and use a covering number approach (via Dudley type entropy bounds and Kolmogorov--Tikhomirov covering estimates for Lipschitz balls) to control Rademacher complexity.  Follow up work on classification and regression in metric spaces makes the dependence on doubling or intrinsic dimension more algorithmic, often through Lipschitz extension and nearest neighbor primitives \cite{gottlieb2014efficientclassification,gottlieb2017efficientregression}.  Our result is related in that Lipschitz continuity transfers input space covers into a statistical complexity bound.  The main difference however is that the functions being covered are not scalar Lipschitz classifiers or regressors.  In our work, they are parameter gradient feature vectors after training, and the cover enters through an effective dimension trace rather than through a uniform Rademacher bound.

A separate line of work studies how the geometry of the data distribution affects learning in high ambient dimension.  The manifold hypothesis asserts that many high dimensional data sets are concentrated near lower dimensional geometric sets, an intuition that motivated classical nonlinear dimension reduction methods such as locally linear embedding, Laplacian eigenmaps, Hessian eigenmaps, diffusion maps, and manifold regularization \cite{roweis2000nonlinear,tenenbaum2000global,belkin2003laplacian,donoho2003hessian,coifman2006diffusion,belkin2006manifold}.  Theoretical analyses of graph Laplacians and related methods clarified when empirical neighborhood graphs approximate intrinsic differential operators on the manifold \cite{belkin2008towards}.  These works motivate the assumption that the data occupied part of input space has low intrinsic dimension.  We use the manifold cover only as a route to controlling the trained Jacobian feature cover.

Nonparametric regression on manifolds is especially relevant because our setting is fixed design square loss regression.  Classical local polynomial and local linear methods can adapt to an unknown lower dimensional manifold and achieve rates governed by intrinsic dimension rather than ambient dimension \cite{bickel2007local,cheng2013local}.  Recent neural network theory establishes analogous intrinsic dimension behavior for ReLU networks under exact or approximate manifold support, low Minkowski dimension, or related geometric assumptions \cite{schmidt2019deep,nakada2020adaptive,chen2022nonparametric,jiao2023deep}.  Empirical work on natural images also supports the view that intrinsic dimension is much smaller than pixel dimension and that lower intrinsic dimension can make tasks easier to learn \cite{levina2005maximum,pope2021intrinsic}.  These results analyze approximation or estimation rates for function classes.  Our bound focuses on controlling the stability of a particular trained solution through the empirical geometry of its Jacobian features.

For ReLU and other piecewise linear networks, the input space is partitioned into activation or linear regions.  Early response region analyses and worst case expressivity bounds show that depth can create many more regions than comparable shallow architectures, in some cases exponentially many in depth \cite{pascanu2013response,montufar2014number,raghu2017expressive,serra2018bounding}.  Average case counts at initialization and empirical counts during training can be much smaller than these worst case bounds. Hanin and Rolnick \cite{pmlr-v97-hanin19a,hanin2019fewpatterns} study activation patterns/regions and argue that deep ReLU networks have surprisingly few such patterns at initialization and during training.  The spline viewpoint of Balestriero and Baraniuk \cite{balestriero2018spline} similarly interprets piecewise linear networks as adaptive affine spline operators whose local affine pieces are selected by activation patterns.  These works explain why activation partitions are natural geometric objects, but they do not by themselves yield a stability bound for nonlinear least squares training.

More recent work connects activation region geometry to learned representations and robustness.  Patel and Montufar \cite{patel2025localcomplexitylinearregions} define a local complexity measure for the density of linear regions near a data distribution and relate lower local complexity to low dimensional learned representations.  Empirically, O'Brien et al \cite{o'brien2025using} study local complexity as a predictor of out of distribution behavior.  The present paper turns this geometric intuition into a direct generalization mechanism for square loss regression. Only data occupied regions enter the partition bound, and the relevant quantity is not the number of all possible regions (which is very large) but the much smaller number of occupied pieces, together with the within piece variation of trained Jacobian features.  Results on simplicity bias in two layer ReLU networks \cite{boursier2025simplicitybiasoptimizationthreshold} are consistent with this picture, since training can favor simpler solutions than arbitrary interpolation even in overparameterized models.

Finally, several empirical and theoretical works associate good generalization with low local sensitivity or collapsed representations.  Jacobian based margin and robustness analyses argue that controlling the input output Jacobian near the data can improve generalization or robustness \cite{sokolic2017robust,novak2018sensitivity}.  Neural collapse studies show that late stage classification training can produce highly structured within class feature geometry \cite{papyan2020prevalence,zhu2021geometric,2025neuralcollapsegloballyoptimal}.  These works are adjacent but not identical to our setting.  We study regression with square loss, and our central object is the parameter gradient feature map \(x\mapsto \nabla_\theta f(x;\widehat\theta)\), rather than the input output Jacobian or the penultimate layer classifier geometry.  Nevertheless, all of these lines support the broader idea that learned local geometry is the relevant complexity controlling generalization.

\subsection{Comparison with existing generalization bounds}
\label{app:comparison}
Table~\ref{tab:literature_comparison} compares representative results and their principal assumptions. Our argument avoids specifying an optimization trajectory or imposing a global PL condition.

The nonconvex SGD result of \citet[Theorem~3.8]{pmlr-v48-hardt16} assumes a bounded, Lipschitz, smooth loss and a decaying step size. \citet[Corollary~1]{pensia2018generalization} obtain a bound proportional to $\sqrt{\sum_t\eta_t^2/\sigma_t^2\,/n}$ under bounded updates and subgaussian losses. Constant step to noise ratios give $O(\sqrt{T/n})$. The PL result of \citet{CharlesP18} yields an $O(1/n)$ term together with an optimization term. \citet[Theorem~1]{RichardsR21} retain an exponential factor in accumulated weak curvature. The simplified rate in the table applies when this factor stays bounded and their sample size and step size conditions hold. \citet[Corollary~1]{RichardsK21} use an oracle quantity combining empirical loss with distance from initialization, under sufficient width and smoothness.

Minima stability is a different notion from replace one algorithmic stability. \citet{QiaoZSSW24} bound estimation error on a strict interior interval for univariate ReLU fits under target regularity and optimization conditions. \citet[Theorem~3.5]{LiangQWP25} give a high probability uniform gap bound with rate $\widetilde O_d(n^{-1/(2d+4)})$ for flat solutions under uniform ball inputs and bounded outputs and labels. Their Theorem~3.6 gives the same exponent for an MSE upper bound under additional conditions. \citet{WangM23} bound deep homogeneous network generalization through norms controlled along GD and SGD trajectories. Their prefactor depends on depth, initialization and cumulative loss, and cannot in general be treated as independent of training time.

\paragraph{Degrees of freedom.} Let $S(Y)=D_Yh(Y)=J\widehat H_\lambda^{-1}J^\top/n$ wherever the influence identity holds. Then $d_{\rm eff}=\operatorname{tr}(S^2)=\|S\|_F^2$. This is related to, but distinct from, the divergence $\operatorname{tr}(S)$ used in degrees of freedom calculations \cite{Ye98,Efron04}. Under Gaussian responses and suitable integrability, Stein's identity relates the expected divergence to a covariance penalty. In linear ridge regression, the eigenvalues of $S$ are $\mu_j/(\mu_j+\lambda)$. The squared trace functional counts their squared contributions. Outside the linear case, the fitted Hessian also enters.

\end{document}